\documentclass[11pt]{article}
\usepackage{fullpage}
\usepackage[numbers, compress]{natbib}

\usepackage[utf8]{inputenc} 
\usepackage[T1]{fontenc}    
\usepackage{url}            
\usepackage{booktabs}       
\usepackage{amsfonts}       
\usepackage{nicefrac}       
\usepackage{microtype}      
\usepackage[table]{xcolor}  
\definecolor{Gray}{gray}{0.9}
\definecolor{midgreen}{rgb}{0.1,0.5,0.1}
\definecolor{darkgray}{gray}{0.25}
\definecolor{lightblue}{rgb}{0.25,0.25,0.8}
\definecolor{mydarkblue}{rgb}{0,0.08,0.45}
\usepackage[colorlinks,
            linkcolor=lightblue,
            anchorcolor=blue,
            urlcolor=mydarkblue,
            citecolor=lightblue]{hyperref}
\usepackage{amsmath}
\usepackage{amsthm}
\usepackage{amssymb}
\usepackage{makecell}
\usepackage{wrapfig}

\usepackage{caption}
\usepackage{subcaption}
\usepackage{enumerate}
\usepackage{fullpage}
\usepackage{thmtools}
\usepackage{thm-restate}
\usepackage[capitalize,nameinlink,noabbrev]{cleverref}
\usepackage[noend]{algorithmic}
\usepackage{algorithm}
\usepackage{apptools}
\usepackage{bm}
\usepackage{bbm}
\usepackage{mathtools}
\usepackage{booktabs}       
\usepackage{multicol}
\usepackage{multirow}
\usepackage{authblk}

\usepackage{enumitem}
\setlist[enumerate]{topsep=0pt,partopsep=0pt}

\usepackage{accents}
\newcommand*{\dt}[1]{\accentset{\mbox{\bfseries .}}{#1}}

\newcommand{\norm}[1]{\ensuremath{\left\| #1 \right\|}}

\newcommand{\inner}[1]{\left \langle {#1} \right \rangle}
\newcommand{\bigo}{\mathcal{O}}
\newcommand{\abs}[1]{\left |#1\right|}

\def\erf{\mathrm{erf}}

\def\tsigma{\widetilde{\sigma}}

\def\tkappa{\widetilde{\kappa}}

\newcommand{\poly}[1]{\mathrm{poly}\left(#1\right)}
\newcommand{\ijij}[1]{#1_{i,j,i',j'}(y,z)}
\newcommand{\sumijij}[1]{\sum_{a=-\frac{q-1}{2}}^{\frac{q-1}{2}} \sum_{b=-\frac{q-1}{2}}^{\frac{q-1}{2}} #1_{i+a,j+b,i'+a,j'+b}(y,z)}

\def\0{{\bm 0}}

\def\c{{\bm c}}

\def\A{{\bm A}}

\def\D{{\bm D}}

\def\I{{\bm I}}

\def\K{{\bm K}}

\def\W{{\bm W}}
\def\X{{\bm X}}
\def\Y{{\bm Y}}

\def\Ebb{\mathop{\mathbb{E}}}
\def\Rbb{\mathbb{R}}

\def\SS{\mathbb{S}}

\newenvironment{proofof}[1]{\noindent{\bf Proof of #1:}}{\hfill$\qed$\par}

\def\normx{\norm{x}_2}
\def\normy{\norm{y}_2}
\renewcommand{\cite}[1]{\citep{#1}}

\newtheorem{theorem}{Theorem}
\newtheorem{lemma}{Lemma}

\theoremstyle{definition}

\newtheorem{defn}{Definition}

\newtheorem*{lemma*}{Lemma}

\everypar{\looseness=-1}

\newcommand{\Amir}[1]{\bf \color{orange} Amir: }

\DeclareMathOperator*{\plim}{plim}

\title{Fast Neural Kernel Embeddings for General Activations}

\author[1]{Insu Han}
\author[2]{Amir Zandieh}
\author[3]{Jaehoon Lee}
\author[3]{Roman Novak}
\author[3]{Lechao Xiao}
\author[1,3]{Amin Karbasi}
\affil[1]{Yale University}
\affil[2]{Max-Planck-Institut für Informatik}
\affil[3]{Google Research}

\date{}

\begin{document}

\maketitle

\begin{abstract}
Infinite width limit has shed light on generalization and optimization aspects of deep learning by establishing connections between 
neural networks and kernel methods. Despite their importance,
the utility of these kernel methods was limited in large-scale learning settings due to their (super-)quadratic runtime and memory complexities.
Moreover, 
most
prior works on neural kernels have focused on the ReLU activation
, mainly due to its popularity but also due to 
the difficulty of computing such kernels for general activations.
In this work, we overcome such difficulties by providing methods to work with general activations. 
First, we compile and expand the list of activation functions admitting exact dual activation expressions to compute neural kernels.
When the exact computation is unknown, we present methods to effectively approximate them.
We propose a fast sketching method that approximates any multi-layered Neural Network Gaussian Process (NNGP) kernel and Neural Tangent Kernel (NTK) matrices for a wide range of activation functions, going beyond the commonly analyzed ReLU activation. This is done by showing how to approximate the neural kernels using the truncated Hermite expansion of any desired activation functions. While most prior works require data points on the unit sphere, our methods do not suffer from such limitations and are applicable to any dataset of points in $\mathbb{R}^d$. Furthermore, we provide a subspace embedding for NNGP and NTK matrices with near input-sparsity runtime and near-optimal target dimension which applies to any \emph{homogeneous} dual activation functions with rapidly convergent Taylor expansion. 
Empirically, with 
respect to exact convolutional NTK (CNTK) computation, our method achieves $106\times$ speedup for approximate CNTK of a 5-layer Myrtle network on CIFAR-10 dataset. 
\end{abstract}

\section{Introduction}

Infinite width limit has enabled fundamental understandings of deep neural networks by establishing a correspondence to kernel methods.
In this limit, the network's function prior is a Gaussian process~\cite{neal,lee2018deep,matthews2018} and under gradient descent training with squared loss, the network behaves as a linearized function~\cite{jacot2018neural,lee2019wide}. Underlying these limit, a core object is a neural kernel which encapsulates architectural inductive prior in its functional form~\cite{xiao2022eigenspace}. The kernel describing gradient descent dynamics, the Neural Tangent Kernel (NTK)~\cite{jacot2018neural}, and Neural Network Gaussian Process (NNGP)~\cite{lee2018deep} kernel have been extensively studied~\cite{novak2018bayesian,garriga2018deep,arora2019exact,hron2020,du2019graph,yang2019scaling} since they were initially identified. In particular, the infinite width theory has shed light on powerful abilities of deep neural networks including optimization~\cite{arora2018optimization,allen2019convergence,pilanci2020neural,xiao2020disentangling}, generalization~\cite{neyshabur2018the,arora2019fine,cao2019generalization}, regularization~\cite{wei2019regularization,hu2019simple,jacot2020implicit} and robustness~\cite{dohmatob2021fundamental,hassani2022curse}. 
Beyond
theoretical findings, 
it has been 
extensively reported that neural kernels can enhance practical applications including small data classification/regression tasks~\cite{arora2019harnessing}, neural architect search~\cite{park2020towards, chen2020tenas}, dataset distillation~\cite{nguyen2021dataset_iclr, nguyen2021dataset}, federated learning~\cite{huang2021fl}, meta learning~\cite{zhou2021meta}, generalization attack~\cite{yuan2021neural}, just to name a few.

Despite those powerful advantages, there is still a gap between practice and theory in the utility of these kernel methods. First, the NNGP and NTK can be exactly computed recursively~\cite{lee2018deep,jacot2018neural} however, the explicit forms are only known when the corresponding neural networks contain a few set of activation functions such as ReLU or Error functions.
While ReLU activation is the default choice for many deep learning applications, recently different activation functions have shown to work well in various domains of machine learning. For example, GeLU~\cite{hendrycks2016gaussian} has been widely used in Transformer based natural language processing settings~\cite{devlin2018bert,radford2019language,brown2020language} and sinusoidal activation functions work well for implicit neural representation (e.g. NeRF)~\cite{sitzmann2020implicit,tancik2020fourier}. Moreover, \citet{xie2020smooth} showed that smooth activation functions could improve robustness compared to ReLU-based models. To enable better theoretical understanding on the role of these activation functions in these domain, expanding the infinite width limit tool set to general activation function is an important step forward.

Secondly, even if the exact neural kernel computation is explicitly known, it requires significantly huge amount of computing resources. For example, it will take order of few 100 to 1,000 GPU hours to compute the exact NTK of depth $10$ convolutional neural networks with pooling on $60{,}000$ CIFAR-10 dataset. High compute requirement is often too expensive to perform extensive studies or use in a practical setting. While \citet{novak2022fast} have sped up Monte Carlo estimation of the NTK, random sampling remains impractical due to still high kernel computation cost, and cubic (in the training set size) inference cost. Recently, \citet{zandieh2021scaling} proposed an efficient method to approximate the NTK computation via sketching algorithms. Their algorithm can approximate the neural kernels with ReLU activation orders of magnitude faster than the exact one. But it remains unclear how sketching algorithms are extended to other activations.

In this work, we fill this gap by showing that neural kernel for \emph{arbitrary} smooth activation can be expressed in a form of series expansion.  
We first focus on how to express a kernel function of neural network with a single hidden layer. Under the infinite width limit, this kernel converges to a static function, so-called a \emph{dual kernel}, and is determined by activation in the network. This is a key block to compute the NNGP and NTK of deeper architectures.  
We establish an explicit expression of dual kernel by expanding activation with the Hermite polynomial basis, and combining it with the fact that Hermite polynomials can play a role of random features of monomial kernels. As a result, our dual kernel formulation relies on coefficients of series expansion of the activation.
In addition, we also derive dual kernel expression of the first-order derivative of activation. 
The NTK can be computed by combining these kernel computations. To the best of our knowledge, our work is the first to study the computation of the NTK for general activations.
Furthermore, we provide a subspace embedding for NNGP and NTK matrices with near input-sparsity runtime and near-optimal target dimension.
As activation functions play an important role in modern neural network architectures, we hope our work could empower researchers to explore properties of activations in a more principled way. 
Our main contributions are summarized as follows:
\begin{itemize}[wide, labelwidth=!, labelindent=5pt]
    \item {\bf Building blocks for infinite-width neural kernel computations}: We derive an explicit expression of the dual kernel for a polynomial activation, which can be a building block for infinite-width neural kernel computations. For non-polynomial activation, we suggest to use its truncated Hermite expansion and analyze an error bound of the dual kernel.
    \item \textbf{Compiling and expanding dual activation~\cref{tab-dual-kernel}}: We compile various known dual kernel for point-wise 
    activations
    providing pointers to the original work and expand the set further.
    We hope our work also serve as an easy reference for various analytic expressions.
    We emphasize that while many prior references lack required computation for NTK, this work is comprehensive in covering both NNGP/NTK transformations for various 
    activations
    where analytic computation is possible.
    \item {\bf NTK computation}: Dual kernels of both activation and its derivative are essential for the NTK computation.  Since our formulation requires coefficients of Taylor series of the activation, it is applicable to the dual kernel of derivative of the activation. In addition, we propose how to 
    automatically compute the dual kernel of the derivative
    without knowing the activation. This approach is useful to characterize the NTK for kernel functions whose activation function is unavailable, e.g., normalized Gaussian, or whose dual kernel of the derivative is unavailable, e.g., GeLU and ELU.
    
    \item {\bf Kernel approximation}:  We analyze a pointwise error bound of approximated dual kernel via truncated Hermite expansion of the activation with a finite degree.
    The estimation error can decay polynomially faster in the degree. Furthermore, due to specific decomposition of our kernel formulation, we accelerate the NTK approximation by sketching techniques, similar to \cite{zandieh2021scaling}. 
    We also propose a new sketching method for the Convolutional NTK with homogeneous activations and analyze both a pointwise error bound and its runtime in~\cref{sec-cntk-sketch}. Notably, our sketching method's runtime scales only linearly in the number of pixels of the input images, while the exact CNTK computation scales quadratically in the number of pixels.
    
    \item {\bf Implementation}: We open-source NNGP and NTK 
    for new activations within the Neural Tangents library \citep{neuraltangents2020} and sketching algorithm at \href{https://github.com/insuhan/ntk\_activations}{\texttt{https://github.com/insuhan/ntk\_activations}}.
\end{itemize}

\subsection{Related Work}
\vspace{-0.05in}
Neural kernels (NTK, NNGP) can be computed using the recursive formula~\cite{jacot2018neural,lee2019wide,lee2018deep, matthews2018}.
A prerequisite for these kernels is computing a static kernel function which is defined as the expectation of some function of (non-linear) activation in neural network over the standard normal distribution. 
\citet{williams1996computing} studied this a dual kernel of $\mathrm{erf}(t)$ and Gaussian. 
\citet{cho2009kernel} derived dual kernels for the rectified monomials, i.e., $t^q \mathbbm{1}_{\{t\geq 0\}}$, this function is equal to arc-cosine kernels where ReLU activation is a special case when $q=1$. \citet{rahimi2007random} showed that sinusoidal activations, e.g., $\sin$ or $\cos$, can result in the Gaussian RBF kernel function using the Fourier transform. 
\citet{daniely2016toward} proposed a method to obtain a dual kernel if activation can be expanded by Hermite polynomials. However, inputs of the resulting kernels are restricted to be on the unit sphere.
\citet{louart2018random} analyzed asymptotic properties of dual kernel with random matrix theory and show closed-form formula of such as $\erf$, $\abs{t}$, sinusoidal.
\citet{tsuchida2020avoiding} studied the dual kernels of both Gaussian Error Linear Unit (GeLU)~\cite{hendrycks2016gaussian} and Exponential Linear Unit (ELU)~\cite{hendrycks2016gaussian}. For activation that does not admit a closed-form expression, \citet{lee2018deep} numerically computed dual activation by doing interpolation on predetermined grid of variances and covariances.  
\cref{tab-dual-kernel} summarizes activations whose dual kernels were priorly known, as well as expanding (in this work) the set to previously unknown expressions.
Recently, \citet{simon2021reverse} discovered that NTK of fully-connected neural network with any depth can be converted into that of a 1 hidden-layer 
neural network by modifying activation function. However, their method is limited to the normalized input data and fully-connected networks.

\begin{table}[t]
\caption{Activation functions and references for their dual kernels. More detailed expressions are provided in \cref{sec-appendix-dual-activations}.} \label{tab-dual-kernel}
\vspace{0.15in}
\scalebox{0.9}{
\centering
\setlength{\tabcolsep}{18pt}
\renewcommand{\arraystretch}{1.15}
\begin{tabular}{@{}llll@{}}
	\toprule
	Activation  & $\sigma(t)$ & \makecell[l]{Reference \\ for the NNGP} & \makecell[l]{Reference \\ for the NTK} \\
	\midrule
	Rectified monomials & $t^q \cdot \mathbbm{1}_{\{ t \geq 0\}}$ & \cite{cho2009kernel} & \cite{cho2009kernel} \\ 
	Error function & $\erf(t)$ & \cite{williams1996computing} & \cite{lee2019wide} \\
	ABReLU (Leaky ReLU) & $-A\min(t,0) + B\max(t,0)$ &\cite{tsuchida2018invariance,tsuchida2019richer,neuraltangents2020} & \cite{tsuchida2018invariance,tsuchida2019richer,neuraltangents2020}\\
	Exponential & $\exp(At)$& \cite{mairal2014convolutional,daniely2016toward} & \cite{mairal2014convolutional,daniely2016toward}\\	
    Hermite polynomials & $h_q(t)$& \cite{daniely2016toward} & This work \\	
	Sinusoidal & $\sin(At+B)$& \cite{rahimi2007random, louart2018random,pearce2020expressive} & This work \\

	Gaussian & $\exp\left(-A t^2\right)$ & \cite{williams1996computing} & This work \\
	GeLU & $\frac{t}{2} \left( 1 + \mathrm{erf}\left(\frac{t}{\sqrt{2}}\right)\right)$ & \cite{tsuchida2020avoiding} & This work \\
	ELU & $\mathrm{step}(t) t + \mathrm{step}(-t)\left(e^t -1 \right)$ & \cite{tsuchida2020avoiding} & This work \\	
	Normalized Gaussian & Unknown & \cite{shankar2020neural} & This work\\
	RBF & $\sqrt{2} \sin(\sqrt{2 A} t + \frac{\pi}{4})$ & \cite{rahimi2007random} & This work\\
	Gabor & $\exp(-t^2) \sin(t)$ & This work & This work \\	
	Monomial & $t^q$&  This work & This work \\
	Polynomial  & $\sum_{j=0}^q a_j t^j$ & This work & This work\\
	\bottomrule
\end{tabular}
}
\end{table}

\vspace{-4pt}
\section{Preliminaries} \label{sec-prelim}
\vspace{-4pt}
\paragraph{Notations.}  We denote the identity matrix of dimension $d$ by $\I_d$. For a scalar function $f$, we write $f^{(k)}$ to denote its $k$-th derivative. We use $\mathbbm{1}_{\mathcal{E}}$ to denote the indicator of event $\mathcal{E}$. For a smooth 
function $\sigma:\Rbb \rightarrow \Rbb$, we use $\sigma^{(k)}$ to denote its $k$-th derivative and define $\norm{\sigma}_{\mathcal{N}(0,\nu^2)}^2 := \Ebb_{t \sim \mathcal{N}(0,\nu^2)}[ \abs{\sigma(t)}^2 ]$ for some $\nu\in\Rbb$ and simply write $\norm{\sigma}_{\mathcal{N}(0,1)}:=\norm{\sigma}_{\mathcal{N}}$.
For scalar functions $f,g$ we use $f \circ g$ to denote the composition of these functions and 
$f^{\circ q}$ to denote the $q$ times self-composition of $f$, e.g., $f^{\circ 3}(x)=f(f(f(x)))$.
Given a positive semidefinite matrix $\K$ and $\lambda >0$, the statistical dimension of $\K$ with regularizer $\lambda$ is defined as $s_{\lambda}(\K) := \mathtt{tr}(\K (\K + \lambda \I)^{-1})$.
We use $\mathrm{nnz}(x)$ to denote the number of nonzero entries in $x$.
Given $x \in \Rbb^m$ and $y \in \Rbb^n$,
we define
$x\otimes y:=\begin{bmatrix} x_1y_1,x_2y_1, \ldots x_my_1, x_1y_2, \ldots x_my_2, \ldots x_my_n\end{bmatrix}$ and 
$x^{\otimes p}$ as 
the $p$-fold self-tensoring of $x$. We also define $\oplus$ as the direct sum between vectors.

\paragraph{Hermite polynomials.}
The \emph{Probabilist's Hermite polynomials} of degree $\ell \geq 0$ is defined as
\begin{align} \label{eq-defn-hermite-polynomial}
    h_\ell(t) = (-1)^\ell e^{\frac{t^2}{2}} \left[ \frac{d^\ell}{dt^\ell} e^{-\frac{t^2}{2}}\right] = \ell! \sum_{i=0}^{\lfloor \ell /2 \rfloor} \frac{(-1)^i}{i! (\ell - 2i)!} \frac{t^{\ell - 2i}}{2^i}.
\end{align}
The polynomials $\{h_\ell\}_{\ell\ge0}$ form a set of orthogonal basis for the space of square-integrable functions in $\Rbb$ with respect to the normal measure $\mathcal{N}(0,1)$, i.e., the $L^2$ space of functions $L^2(\Rbb,\mathcal{N}) := \{ f:\Rbb \rightarrow \Rbb\mid 
\norm{\sigma}_\mathcal{N}^2<\infty\}$. Particularly,
it holds that $\Ebb_{t \sim \mathcal{N}(0,1)}\left[h_\ell(t)~h_m(t)\right] = \ell! \cdot \mathbbm{1}_{\{\ell=m\}}$.
Thus, any function $f \in L^2(\Rbb,\mathcal{N})$ has a unique Hermite expansion in the sense of $\norm{f - \sum_{t=0}^\infty c_j h_j}_{\mathcal{N}}=0$ and coefficient $c_j$ can be computed as $c_j = \Ebb_{t \sim \mathcal{N}(0,1)}\left[f(t)~h_j(t)\right] / j!$.

\paragraph{Infinite width neural kernels.}
Given an activation $\sigma:\Rbb\rightarrow \Rbb$ satisfying that $\norm{\sigma}_{\mathcal{N}} =1$, consider a fully-connected $L$-layered neural network $f:\Rbb^d \rightarrow \Rbb$ for $L \geq 2$ defined as\footnote{
Throughout 
the
paper, we consider scalar-valued networks without 
biases
for simplicity, but this can be extended to vector-valued networks 
with 
biases
. We also assume 
$\norm{\sigma}_{\mathcal{N}}=1$ which does not change our results.}
\begin{align}
    f_\sigma(x;\mathcal{W}) = \inner{w^{(L)}, z_{L-1}} / \sqrt{d_{L-1}}, ~~ z_{\ell} = {\sigma\left(\W^{(\ell)} z_{\ell-1}/\sqrt{d_{l-1}}\right)}, ~~ z_0 = x
\end{align}
where $\mathcal{W}:=\mathrm{vec}\left( w^{(L)}, \cup_{\ell=1}^{L-1} \W^{(\ell)}\right)$ for 
$w^{(L)} \in \Rbb^{d_{L-1}}, \W^{(\ell)}\in\Rbb^{d_{\ell} \times d_{\ell-1}}, d_{0} := d, d_{l} := m$ for $l > 0$
is a collection of learnable parameters, $m$ is the width of the network,
and $\sigma(\cdot)$ is applied point-wisely.
In the infinite width limit, i.e., $m\rightarrow\infty$, when all elements of $\mathcal{W}$ are initialized by i.i.d. random samples from 
$\mathcal{N}(0, 1)$ 
and optimized via 
gradient descent on the least-square loss with an infinitesimal learning rate, the prediction of trained 
network becomes identical to that of its first order Taylor approximation at 
$\mathcal{W}$. Hence, inference with such ultra-wide 
network is equivalent to 
kernel regression 
with a static kernel, the so-called Neural Tangent Kernel (NTK), defined as $\Theta^{(L)}_\sigma(x,y):=\plim_{m \rightarrow \infty} \inner{\nabla_{\mathcal{W}} f_\sigma(x;\mathcal{W}), \nabla_{\mathcal{W}} f_\sigma(y;\mathcal{W})}$ (convergence in probability to a constant).
In addition, at initialization the output of an infinitely wide network is equivalent to a sample from a Gaussian process with mean zero and covariance 
$\Sigma_\sigma^{(L)}(x,y):=\plim_{m \rightarrow \infty} \inner{f_\sigma(x;\mathcal{W}), f_\sigma(y;\mathcal{W})}$, known as the Neural Network Gaussian Process (NNGP) kernel.

\paragraph{Recursive expression for NNGP and NTK.}
Several previous works~\citep{lee2018deep, matthews2018, jacot2018neural, lee2019wide} 
have shown that the NNGP and NTK can be expressed using the following recursive procedure:
\begin{enumerate}[wide, labelwidth=!, labelindent=5pt]
    \item For every $x,y\in \Rbb^d$, let $K_\sigma^{(0)}(x,y) := \inner{x,y}$ and for every layer $h = 1, \ldots, L $, recursively define kernel functions $K_\sigma^{(h 
    )}, \dt{K}_{\sigma}^{(h)}: \Rbb^d \times \Rbb^d \to \Rbb$ as:
    \begin{align}\label{eq-dp-covar}
        &K_\sigma^{(h)}(x,y) := 
    	{\Ebb_{(u,v) \sim \mathcal{N}(0,{\bm\Lambda}_\sigma^{(h)})} \left[ \sigma(u) \sigma(v) \right]},~ 
    	&\dt{K}_{\sigma}^{(h)}(x,y) := {\Ebb_{(u,v)\sim \mathcal{N}(0,{\bm\Lambda}_\sigma^{(h)})} \left[ \sigma'(u)\sigma'(v) \right]}, 
	\end{align}
	where the covariance matrix is
	${\bm\Lambda}_\sigma^{(h)} := 
	    \begin{bmatrix}
		    K_\sigma^{(h-1)}(x,x) & K_\sigma^{(h-1)}(x,y)\\
			K_\sigma^{(h-1)}(y,x) & K_\sigma^{(h-1)}(y,y)
		\end{bmatrix}\in \Rbb^{2 \times 2}.$
	\item The depth-$L$ NNGP kernel is $K^{(L)}_\sigma(x,y)$ and  the depth-$L$ NTK $\Theta_\sigma^{(L)}$ can be recursively computed as $\Theta_\sigma^{(0)}(x,y) := \inner{x,y}$ and 
	\begin{align}\label{eq-dp-ntk}
	    \Theta_\sigma^{(h)}(x,y) := \Theta_\sigma^{(h-1)}(x,y) \cdot \dot{K}_{\sigma}^{(h)}(x,y) + K_\sigma^{(h)}(x,y).
	\end{align}
\end{enumerate}

At the core of the expression for $\Theta_\sigma^{(L)}$, there is the expectation term
over 2-dimensional Gaussian distribution in \cref{eq-dp-covar}.
This expectation term for the case where both diagonal entries of the covariance matrix ${\bm\Lambda}_\sigma^{(\ell)}$ are equal to one, was previously studied in \cite{daniely2016toward}. We extend this to encompass general symmetric covariance matrices in the following definition.

\begin{defn}[Dual Kernel and Dual Activation]\label{def-dual-activation-kernel}
For a smooth $\sigma:\Rbb \rightarrow \Rbb$, we define the \emph{Dual Kernel} of $\sigma$ as $K_\sigma:\Rbb^d \times \Rbb^d \rightarrow \Rbb$ defined as
\begin{align}\label{eq-random-features-sigma}
K_\sigma(x , y) := \Ebb_{w \sim \mathcal{N}(0,\I_d)}\left[\sigma(\inner{w, x}) \sigma(\inner{w, y})\right]~~ \text{ for every } x , y \in \Rbb^d.
\end{align}
\cref{eq-random-features-sigma} only depends on bivariate Gaussian random variables $\inner{w,x}, \inner{w,y}$ where $\Ebb[\inner{w,x}^2]=\norm{x}_2^2, \Ebb[\inner{w,y}^2]=\norm{y}_2^2$ and $\Ebb[\inner{w,x} \cdot \inner{w,y}] = \inner{x,y}$.
Hence one can look at the dual kernel from a different perspective by choosing a proper covariance matrix. To this end, let ${\bm\Lambda}_{a,b,c}:=\begin{bsmallmatrix} a^2 & ab c \\ ab c & b^2\end{bsmallmatrix}$  for every $a,b\in\Rbb_+$ and $c \in [-1,1]$ and the \emph{Dual Activation} of $\sigma$ with respect to ${\bm\Lambda}_{a,b,c}$ is the function $k_\sigma:\Rbb_+\times \Rbb_+ \times[-1,1] \rightarrow \Rbb$ defined as
$
k_\sigma(a,b,c) := \Ebb_{(u,v)\sim \mathcal{N}(0,{\bm\Lambda}_{a,b,c})}\left[\sigma(u) \sigma(v)\right].
$

With these definitions in place, the following relationship between dual kernel and activation holds
\begin{align} \label{eq-dual-kernel-dual-activation}
K_\sigma(x, y) =  k_\sigma\left(\norm{x}_2,\norm{y}_2,\frac{\inner{x,y}}{\norm{x}_2 \norm{y}_2}\right).
\end{align}

\end{defn}

Observe that $K_\sigma(x,y)$ corresponds to the NNGP kernel of a $1$-hidden layer neural network with activation $\sigma$. For some specific activations, e.g., ReLU, Error function, closed form expressions for their dual activations are known (see \cref{tab-dual-kernel}). Hence, one can compute the NTK analytically when dual kernels of the activation and its derivative have a closed form expression. 
The above also holds for kernels corresponding to convolutional neural networks called CNN-GP \citep{garriga2018deep, novak2018bayesian} and CNTK \citep{arora2019exact}.

\section{NNGP and NTK for Smooth Activations}

In this section, we focus on the NNGP and NTK for a wide range of smooth activation functions. We first show that a series expansion for the dual kernel can be obtained from that of the activation function, which is a key to NNGP kernel computation. By applying this result to the derivative of the activation function, we can also compute the NTK for the same activation.

\subsection{Dual Kernel Computation}

\citet{daniely2016toward} proved that for absolutely continuous $\sigma:\Rbb \rightarrow \Rbb$ and any ${x}, {y} \in \mathbb{S}^{d-1}$, the dual kernel is equal to 
$
    K_\sigma({x}, {y}) = \sum_{j=0}^\infty c_j^2~j! \cdot \inner{{x}, {y}}^j.$
where $\{ c_j\}_{j \geq 0}$ are coefficients of Hermite expansion of $\sigma$.
We now proceed to generalize this result from $\mathbb{S}^{d-1}$ to entire $\Rbb^d \setminus \{0\}$. First we remark that it can be naturally extended to the dual kernel of \emph{$q$-homogeneous} activation functions, i.e., $\sigma(at) = \abs{a}^q \sigma(t)$ for every $a, t \in \Rbb$, on the entire $\Rbb^d \setminus \{0\}$. For every $x,y\in\Rbb^d \setminus \{0\}$, the corresponding dual kernel is
\begin{align}
    K_\sigma(x,y) &= \normx^q \normy^q \cdot \sum_{j=0}^\infty c_j^2~ j! \cdot \left(\frac{\inner{x,y}}{\normx\normy}\right)^j.
\end{align}
As examples, (leaky) ReLU and rectified polynomials fall into this activation class. 



Now suppose that $\sigma$ is not homogeneous. In particular, we first consider a polynomial activation $\sigma(t) = \sum_{j=0}^q a_j t^j$ with coefficients $\{a_j\}_{j=0}^q$. Recall that $K_\sigma(x,y)$ can be obtained by taking the expectation of $\sigma(\inner{w,x}) \sigma(\inner{w,y})$ over $w \sim \mathcal{N}(0,\I_d)$ for every $x,y\in\Rbb^d \setminus \{0\}$.
To make use of \citet{daniely2016toward}'s result, we factorize the input into its radial and angular part and rewrite the activation by expressing monomials in the Hermite polynomial basis.  
Formally, let us write monomials in the Hermite basis as $t^i = \sum_{\ell=0}^i \mu_{i,\ell}  h_\ell(t)$ for some coefficients $\{\mu_{j,i}\}_{i=0}^j$. 
Then
\begin{align}
    \sigma(\inner{w,x}) = \sum_{j=0}^q a_j \normx^j \inner{w,\frac{x}{\normx}}^j = \sum_{i=0}^q \left( \sum_{j=i}^q \mu_{j,i} \normx^j a_j\right) h_i\left(\inner{w,\frac{x}{\normx}}\right).
\end{align}
Then, we can derive the dual kernel of polynomial activation. We further relax a condition on the activation and propose the result below.

\begin{restatable}{theorem}{nonpolyact} \label{thm-dual-kernel-expansion}
For a polynomial $\widetilde{\sigma}(t) = \sum_{j=0}^q a_j t^j$, the dual kernel of $\tsigma(\cdot)$, as per \cref{def-dual-activation-kernel}, is 
\begin{align} \label{eq-dual-kernel-expansion}
    &K_{\widetilde{\sigma}}(x,y) := \sum_{\ell=0}^q {r_{\tsigma,\ell}(\normx) ~ r_{\tsigma,\ell}(\normy)} \left(\frac{\inner{x,y}}{\normx \normy} \right)^\ell
\end{align}
where $r_{\widetilde{\sigma},\ell}(t) := \sum_{i=0}^{\lfloor \frac{q-\ell}{2} \rfloor} \frac{a_{\ell+2i} (\ell+2i)!}{2^i \cdot i! \cdot \sqrt{\ell!}} t^{2i+\ell}$. Moreover, if an activation function $\sigma:\Rbb \to \Rbb$ satisfies 
$\norm{\sigma}^2_{\mathcal{N}(0,\nu^2)} < \infty$ and $\norm{\sigma - \widetilde{\sigma}}^2_{\mathcal{N}(0,\nu^2)} \leq \varepsilon$ for some $\varepsilon>0$ and $\nu \ge 1$, then for every $x,y \in \Rbb^d$ such that $\normx, \normy \in (0, \nu]$ the following holds
\begin{align} 
    &\abs{K_\sigma(x,y) - K_{\widetilde{\sigma}}(x,y)} \le \sqrt{ \frac{\nu^2 \cdot \varepsilon\left( 6\norm{\sigma}^2_{\mathcal{N}(0,\nu^2)} + 4 \varepsilon \right) }{\normx \normy}}.
\end{align}
\end{restatable}
The proof of \cref{thm-dual-kernel-expansion} is provided in \cref{sec-proof-thm-dual-kernel-expansion}. 
For non-polynomial activations, one can consider approximating $\sigma$ with its Hermite or Taylor expansion and then apply \cref{thm-dual-kernel-expansion}. Examples can be found in \cref{sec-proof-thm-dual-kernel-expansion}.
For activation functions that do not have a Taylor expansion but are $k$-th order differentiable, we show that, using their Hermite expansion, one can obtain a good approximation to the corresponding dual kernel.

\begin{restatable}{theorem}{dualkernelerrorhermiteexpansion} \label{thm-dual-kernel-error-hermite-expansion}
Given $\sigma:\Rbb\rightarrow \Rbb$, suppose that there exists an integer $k \geq 2$ and some $\nu \ge 1$ such that for every $i=0,\dots,k$, $\sigma^{(i)}$ is absolutely continuous and $\lim_{t\rightarrow \pm\infty} e^{-\frac{t^2}{4}} \sigma^{(i)}(\nu t) = 0$ and moreover $\norm{\sigma}_{\mathcal{N}(0,\nu^2)}^2 < \infty$ and $\norm{\sigma^{(k)}}_{\mathcal{N}(0,\nu^2)}^2 < \infty$. Consider the Hermite expansion coefficients $\{c_j\}_{j\geq0}$ of function $\sigma(\nu t)$ and denote $\tsigma(t) := \sum_{j=0}^q c_j h_j(t/\nu)$. 
Given $x,y \in \Rbb^d $ with $\norm{x}_2, \norm{y}_2 \in (0, \nu]$, 
\begin{align}
    \abs{K_{\sigma}(x,y) - K_{\tsigma}(x,y)} \leq
    \frac{  5 \nu^{k+1} \norm{\sigma^{(k)}  }_{\mathcal{N}(0,\nu^2)}  \max\left( \norm{\sigma}_{\mathcal{N}(0,\nu^2)} ,  \nu^{k}  \norm{\sigma^{(k)} }_{\mathcal{N}(0,\nu^2)}\right)  }{\sqrt{ \norm{x}_2 \norm{y}_2 \cdot k \cdot q^{k-1}}}.
\end{align}
where $K_{\sigma}(\cdot,\cdot)$ and $K_{\tsigma}(\cdot,\cdot)$ are dual kernels corresponding to $\sigma(\cdot)$ and $\tsigma(\cdot)$ in~\cref{def-dual-activation-kernel}, respectively. 
Moreover, for the ReLU activation $\sigma(t) = \max(t,0)$, it holds that \begin{align}
\abs{K_{\sigma}(x,y) - K_{\tsigma}(x,y)} \leq \sqrt{ \frac{ 2 \nu^6}{q \norm{x}_2 \norm{y}_2} }.
\end{align}
\end{restatable}

The proof of \cref{thm-dual-kernel-error-hermite-expansion} is provided in \cref{sec-proof-thm-dual-kernel-error-hermite-expansion}.
Observe that when the activation is $k$-th order differentiable and the norms of its derivative and inputs are bounded then the approximation error decreases with $\bigo(\frac{1}{\sqrt{k q^{k-1}}})$ rate. 
In \cref{sec-exp}, we empirically evaluate the dual kernel of various activations using Hermite expansion and verify that smooth activations (e.g., Gaussian or sinusoidal) provides much lower approximation errors than non-smooth ones (e.g., ReLU).

\subsection{NNGP and NTK Computations}

Once dual kernels of $\sigma$ and $\sigma'$ or their polynomial approximations are calculated, one can compute (approximate) NNGP and NTK using \cref{thm-dual-kernel-expansion} or \cref{thm-dual-kernel-error-hermite-expansion} and the recursion in \cref{eq-dp-ntk}. However, there are scenarios where we are only given the dual kernel and the corresponding activation or derivative of the activation is unknown to us. For example, \citet{shankar2020neural} devised a normalized Gaussian kernel defined as
\begin{align} \label{eq-normalzied-gaussian-kernel}
    K_G(x,y) = \normx \normy \exp\left(\frac{\inner{x,y}}{\normx \normy} - 1\right),
\end{align}
and reported that NNGP with this dual kernel performs better than the ReLU NTK by showing promising results on various tasks. Note that, recovering the activation from $K_G$ is non-trivial. From the dual kernel perspective, the activation should be $1$-homogeneous and its Hermite series expansion is of form $\sum_{j=0}^\infty \frac{\pm 1}{j!} h_j(t)$ and it is generally unknown how to choose the sign pattern on coefficients of this series that would satisfy homogeneity constraint. Instead of trying to recover the activation from dual kernel, we show how to directly derive the dual kernel of derivative of activation without knowing the activation. 

\begin{restatable}{theorem}{lmmderivativedualkernel} \label{lmm-derivative-dual-kernel}
Given a differentiable activation function $\sigma:\Rbb \rightarrow \Rbb$ which satisfies $\abs{\sigma(t)} \le C_1 \exp\left( \frac{t^2}{4.1\nu^2} \right)$, $\abs{\sigma^{\prime}(t)} \le C_2 \exp\left( \frac{t^2}{4.1\nu^2} \right)$,
$\norm{\sigma}_{\mathcal{N}(0,\nu^2)}^2 < \infty$ and $\norm{\sigma^{\prime\prime}}_{\mathcal{N}(0,\nu^2)}^2 < \infty$ for some $\nu \ge 1$ and constants $C_1,C_2$, the following holds for any $x,y \in \Rbb^d $ with $\norm{x}_2, \norm{y}_2 \in (0, \nu]$ and $\abs{ \inner{x,y} } < \norm{x}_2 \norm{y}_2$: 
\begin{align} \label{eq-derivative-dual-kernel}
    K_{\sigma'}(x,y) = \frac{1}{\norm{x}_2\norm{y}_2}~ \frac{\partial }{\partial c} k_{\sigma}\left(\normx, \normy, c\right) \Bigg\lvert_{ c={\LARGE\frac{\inner{x,y}}{\normx\normy}}}.
\end{align}
Additionally, if $\frac{\partial}{\partial c} k_{\sigma}(\cdot,\cdot,c)$ is continuous at $c=\pm 1$ then \cref{eq-derivative-dual-kernel} holds for $x,y$ such that $\abs{\inner{x,y}}=\normx\normy$.
\end{restatable}

The proof of \cref{lmm-derivative-dual-kernel} is provided in \cref{sec-proof-lmm-derivative-dual-kernel}.  
Our result is more general compared to \cite{simon2021reverse} where the previous work assumes that the Hermite expansion of given activation should converge and $\normx = \normy$. Applying \cref{lmm-derivative-dual-kernel} to \cref{eq-normalzied-gaussian-kernel} provides that $\dt{K}_G(x,y)=\exp\left( \frac{\inner{x,y}}{\normx \normy}-1\right)$ hence one can compute the NTK function even if the corresponding activation is unknown. In the previous work~\cite{shankar2020neural}, only ``NNGP'' performances of the normalized Gaussian kernel were reported.

Moreover, with \cref{lmm-derivative-dual-kernel}, only the knowledge of dual activation suffices to compute both NNGP and NTK. For example, while dual activation (thus NNGP) of GeLU was known in \citet{tsuchida2020avoiding}, $k_{\sigma'}$ was not derived explicitly. \cref{lmm-derivative-dual-kernel} provides a simple way to compute $k_{\sigma'}$ (given in \cref{eq-dual-gelu-prime}) via automatic differentiation, without requiring to take the expectation under multivariate Gaussian distribution or computing derivatives by hand. We release it as  \href{https://neural-tangents.readthedocs.io/en/latest/_autosummary/neural_tangents.stax.Elementwise.html}{\texttt{stax.Elementwise}} within the Neural Tangents library \citep{neuraltangents2020}.
Our method allows to omit the entire effort, lines of code, and potential mistakes in deriving and implementing the NTK.  




\subsection{Gauss-Hermite Quadrature}
\label{sec-guass-hermite-quadrature}
One simple approach to obtain dual activation function for general activation functions without closed form expressions is to evaluate the expectation of under the $2d$ Gaussian distribution as numerical integration. This can be efficiently done by Gauss-Hermite quadrature
\begin{align}
k_\sigma(a, b, c) 
\approx \frac{1}{\pi} \sum^{q}_{i=1}\sum^{q}_{j=1} w_i w_j \left[\sigma(\sqrt{2} a x_i) \cdot \sigma(\sqrt{2}bc x_i + \sqrt{2} b\sqrt{1-c^2} x_j)\right]
\end{align}
where $(x_i, w_i)$,  correspond to $i$-th root of degree $q$ Hermite polynomial $h_i(x)$ and associated weights~\cite{abramowitz1988handbook}
$w_i=\frac{q! \sqrt{\pi}}{q^2 (h_{q-1}(\sqrt{2} x_i))^2}$.
See \cref{sec-appendix-guass-hermite-quadrature} for the derivation of the quadrature formula.

For smooth activation functions errors will quickly go down as $q$ increases by \cref{thm-dual-kernel-error-hermite-expansion}. We use this method to compute approximate (non-sketched) kernels for general activation functions in \cref{fig-error} and \cref{fig-exact-activation-comparison}. We implement it as \href{https://neural-tangents.readthedocs.io/en/latest/_autosummary/neural_tangents.stax.ElementwiseNumerical.html#neural_tangents.stax.ElementwiseNumerical}{\texttt{stax.ElementwiseNumerical}} withing the Neural Tangents library \citep{neuraltangents2020}.

\section{Approximating Neural Kernels via Sketching}
Although using our \cref{thm-dual-kernel-expansion}, \cref{thm-dual-kernel-error-hermite-expansion}, and \cref{lmm-derivative-dual-kernel}, one can analytically compute NTK for general activation functions, computing all entries in the NTK kernel matrix requires massive amount of resources, i.e., $\Omega(n^2 (d +Lq^2))$ 
runtime and $\Omega(n^2)$ memory for datasets with $n$ points in $\Rbb^d$. This becomes even more expensive for CNTK, where its runtime can be $\Omega((n d_1 d_2)^2 (c + L q^2))$\footnote{This is assuming Hermite expansion degree $q$, when exact expression is known $q^2$ is constant.
} for $n$ of images with size $d_1\times d_2 \times c$.
To avoid quadratic complexities, we adopt a fast and efficient feature map construction via randomized sketching~\cite{zandieh2021scaling} for both NTK and NNGP, i.e., 
\begin{align}
    \Theta_\sigma^{(L)}(x,y) \approx \inner{\psi^{(L)}(x), \psi^{(L)}(y)}, ~~ K_\sigma^{(L)}(x,y) \approx \inner{\phi^{(L)}(x), \phi^{(L)}(y)}.
\end{align}
The previous approach was only applicable for the ReLU activation but we establish more general scheme based on our new results for dual kernel approximation.

\paragraph{Subspace embedding for homogeneous dual kernels.}
We provide a subspace embedding for NNGP and NTK matrices with near input-sparsity runtime and near-optimal target dimension which applies to any \emph{homogeneous} dual activation functions with rapidly convergent Taylor expansion.
More specifically, we call a dual kernel $K_\sigma$ homogeneous if there exists a positive definite dot-product kernel function $\kappa :[-1,1] \to [-1,1]$ such that,
\begin{equation}\label{def-homogeneous-dual-kernel}
    K_\sigma(x,y) = \normx \normy \cdot \kappa\left( \frac{\inner{x,y}}{\normx \normy} \right).
\end{equation}
For such homogeneous dual kernels,
the NTK and NNGP take a similar homogeneous form. In fact, one can show by induction that when the dual kernel is in form of \cref{def-homogeneous-dual-kernel}, the depth-$L$ NNGP function defined in \cref{eq-dp-covar} is equal to the following for any positive integer $L$,
\begin{align}
K_\sigma^{(L)}(x,y) = \normx \normy \cdot \kappa^{\circ L} \left( \frac{\inner{x,y}}{\normx \normy} \right),
\end{align}
where $\kappa^{\circ L}$ denoted the $L$-fold composition of function $\kappa$. Furthermore, if $\kappa$ has a derivative $\kappa' : [-1,1] \to [-1,1]$,
using \cref{lmm-derivative-dual-kernel}, there exists a depth-$L$ NTK for this dual kernel, 
equal to
\begin{equation}\label{ntk-kernl-simplified-homogen}
    \Theta_\sigma^{(L)}(x,y) = \normx \normy \cdot \left. \sum_{h=0}^L \kappa^{ \circ h}(t) \cdot \prod_{i=h}^{L-1} \kappa' \circ \kappa^{ \circ i}(t) \right|_{t =\frac{\inner{x,y}}{\normx \normy} },
\end{equation}
where 
we use the convention that $\kappa^{ \circ 0 }(t) = t$.
Therefore, if $\kappa(\cdot)$ can be tightly approximated by a low-degree polynomial, then the NNGP and NTK functions can also be tightly approximated by low-degree polynomials. Thus, by applying \textsc{PolySketch}, which is a norm-preserving dimensionality reduction that can be applied to the tensor product of multiple vectors very quickly~\cite{ahle2020oblivious}, to the polynomial approximations to these kernels, we can spectrally approximate the NNGP and NTK kernel matrices. For details on \textsc{PolySketch} see \cref{sec-appendix-polysketch}. We provide the details of this procedure in \cref{alg-subspace-embding} and prove the correctness and runtime of our procedure in \cref{thm-subspace-embding}.

\begin{algorithm}[t]
	\caption{Subspace Embedding of Homogeneous NNGP and NTK} \label{alg-subspace-embding}
	\begin{algorithmic}[1]
		\STATE {\bf input}:  $x\in \Rbb^{d}$, depth $L$, sketching dimension $m$, polynomial $\tkappa(t)= \sum_{j=0}^q a_j t^j$ with $a_j \in \Rbb_+$
		
		\STATE calculate the polynomial $P^{(L)}(t) = \tkappa^{\circ L}(t) = \sum_{j=0}^{q^L} b_j t^j$ with coefficients $b_j \in \Rbb_+$
		
		\STATE calculate the polynomial $R^{(L)}(t) = \sum_{h=0}^L \tkappa^{ \circ h }(t) \cdot \prod_{i=h}^{L-1} \tkappa' \circ \tkappa^{ \circ i }(t) = \sum_{j=0}^{p} c_j t^j$ with coefficients $c_j \in \Rbb_+$ and degree $p = q^{\bigo(L)}$ \label{line-polynomial-R}
		\STATE for $\ell=0,\dots, p$, let $Q^{\ell} \in \Rbb^{m \times d^\ell}$ be a degree-$\ell$ \textsc{PolySketch} (See \cref{sec-appendix-polysketch})\label{line-Q-instantiation}
		\STATE for every $\ell=0,\dots, p$, $u^{\ell} \gets Q^{\ell} \left( \frac{x}{\normx} \right)^{\otimes \ell}$ 
	    \STATE construct ${\phi}^{(L)}(x) \gets \normx \cdot \bigoplus_{j=0}^{q^L} \sqrt{b_j} u^{j}$ and  ${\psi}^{(L)}(x) \gets \normx \cdot  \bigoplus_{j=0}^{p} \sqrt{c_j} u^{j} $

		\STATE {\bf return} ${\phi}^{(L)}(x)$ (NNGP embedding),  ${\psi}^{(L)}(x)$ (NTK embedding)
	\end{algorithmic}
\end{algorithm}

\begin{restatable}[Homogeneous NTK Embedding]{theorem}{homogeneousntkembedding}\label{thm-subspace-embding}
Suppose that the dual kernel $K_\sigma$ is homogeneous as per \cref{def-homogeneous-dual-kernel}. Also suppose $\tkappa(t)$ is a degree-$q$ polynomial with non-negative coefficients that satisfies ${\bf (1)}$ $\max_{t \in [-1,1]} \abs{\tkappa(t) - \kappa(t)} \le \frac{1}{\poly{n}}$ and $\max_{t \in [-1,1]} \abs{\tkappa'(t) - \kappa'(t)} \le \frac{1}{\poly{n}}$, ${\bf (2)}$ $\max_{\abs{t} \le 1+\frac{1}{\poly{n}}} \abs{\tkappa(t+\gamma) - \tkappa(t)} \le \frac{1}{\poly{n}}$ and $\max_{\abs{t} \le 1+\frac{1}{\poly{n}}} \abs{\tkappa'(t+\gamma) - \tkappa'(t)} \le \frac{1}{\poly{n}}$ for any $|\gamma| \le \frac{1}{\poly{n}}$. Then for any integer $L\ge 1$, any $\varepsilon, \lambda \ge \frac{1}{\poly{n}}$, and any dataset $\X \in \Rbb^{d \times n}$ with $\norm{\X}_F \le \poly{n}$, if $\K_{\mathrm{ntk}} \in \Rbb^{n \times n}$ is the depth-$L$ NTK kernel matrix on this dataset,
there exists $m = \bigo\left( \frac{s_\lambda(\K_{\mathrm{ntk}})}{\varepsilon^2} \cdot \poly{q^L, \log n} \right)$ such that the output ${\psi}^{(L)}(\X) \in \Rbb^{m \times n}$ of \cref{alg-subspace-embding} satisfies with probability at least $1 - \frac{1}{\poly{n}}$
\begin{align}
(1-\varepsilon) \left(\K_{\mathrm{ntk}} + \lambda \I_n \right) \preceq {\psi}^{(L)}(\X)^\top {\psi}^{(L)}(\X)  + \lambda \I_n \preceq (1+\varepsilon)\left( \K_{\mathrm{ntk}} + \lambda \I_n \right)
\end{align}
Moreover, the runtime of \cref{alg-subspace-embding} is $\bigo\left( \poly{q^L, \log n} \cdot \varepsilon^{-2} \cdot \left(s_\lambda(\K_{\mathrm{ntk}}) \cdot n +  \mathrm{nnz}(\X) \right) \right)$.
\end{restatable}

We prove this theorem in \cref{appndx-proof-subspace-embed}. As an example, let us apply \cref{thm-subspace-embding} on the normalized Gaussian kernel $K_G$ defined in \cref{eq-normalzied-gaussian-kernel}, which is homogeneous. The dot-product factor corresponding to this dual kernel is $\kappa(t) = \exp(t-1)$. The truncated Taylor series of this function is $\tkappa(t) = \sum_{j=0}^q \frac{t^j}{e \cdot j!}$. If $q = \Omega(\log n)$ then it can be verified that the polynomial $\tkappa(t)$ satisfies the preconditions of \cref{thm-subspace-embding}. Therefore, one can invoke \cref{alg-subspace-embding} to get a subspace embedding for the NTK kernel matrix corresponding to the normalized Gaussian dual kernel $K_G$ in $\bigo\left( \varepsilon^{-2} \cdot \left(s_\lambda(\K_{\mathrm{ntk}}) \cdot n +  \mathrm{nnz}(\X) \right) \cdot \poly{\log^L n} \right)$ time and with a target dimension of $m = \bigo\left({\varepsilon^{-2}}\cdot { s_\lambda(\K_{\mathrm{ntk}})} \cdot \poly{\log^L n} \right)$. For any constant number of layers, $L$, this runtime and target dimension is is optimal up to $\poly{\log n}$ factors.
The implementation of our sketching algorithm is available at \href{https://github.com/insuhan/ntk_activations}{\texttt{https://github.com/insuhan/ntk\_activations}}.

\section{Experiments} \label{sec-exp}

In this section, we perform experiments with the proposed neural kernels based on our dual kernel approximation. All experiments run using a single A100 GPU machine.

\paragraph{Kernel approximation.}
We first benchmark our algorithm to approximate the dual kernel matrix. We use $\mathrm{ReLU}$, Abs (i.e., $\sigma(t)=\abs{t}$), $\sin$, Gaussian, $\mathrm{erf}$ and GeLU activations and approximate them by their Hermite expansion where degree changes from $q=1$ to $20$.  We randomly generate $n=1{,}000$ of $256$-dimensional inputs where each entry is i.i.d. drawn from $\mathcal{N}(0,{1}/{\sqrt{256}})$. We also compare our approach to the Monte Carlo estimation of dual kernel, i.e., $K_\sigma(x,y) \approx \frac1m \sum_{i=1}^m \sigma(\inner{w_i, x}) \sigma(\inner{w_i, y})$ where $\{ w_i\}_{i=1}^m$ are i.i.d. standard Gaussian vectors. In \cref{fig-kernel-approx-synthetic}, we plot relative errors of the Frobenius norm of kernel approximations in terms of wall-clock times ({\bf top}) and polynomial degree ({\bf bottom}). 
We run 10 independent trials and evaluate the average approxmation errors.  We observe that our approximation with Hermite expansion outperforms the Monte Carlo method for all activations we used. In particular, $\sin$ and Gaussian are well approximated because they are smooth and norms of their derivatives are bounded with respect to the normal measure. 

\begin{figure}[t]
	\centering
	\includegraphics[width=\textwidth]{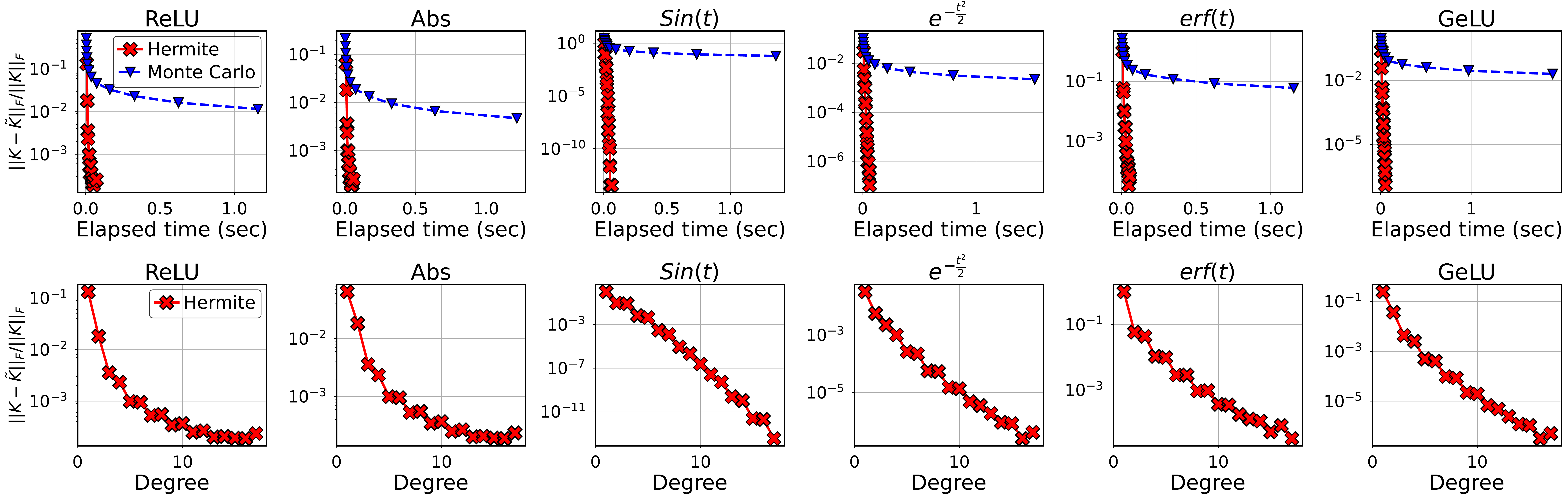}
	\vspace{-0.2in}
	\caption{Relative errors of dual kernel approximations via the truncated Hermite expansion and Monte Carlo estimation under synthetic dataset with $n=1{,}000, d=256$. 
	} \label{fig-kernel-approx-synthetic}
	\vspace{-0.15in}
\end{figure}

\begin{wrapfigure}{R}{0.5\textwidth}
	\includegraphics[width=0.5\textwidth]{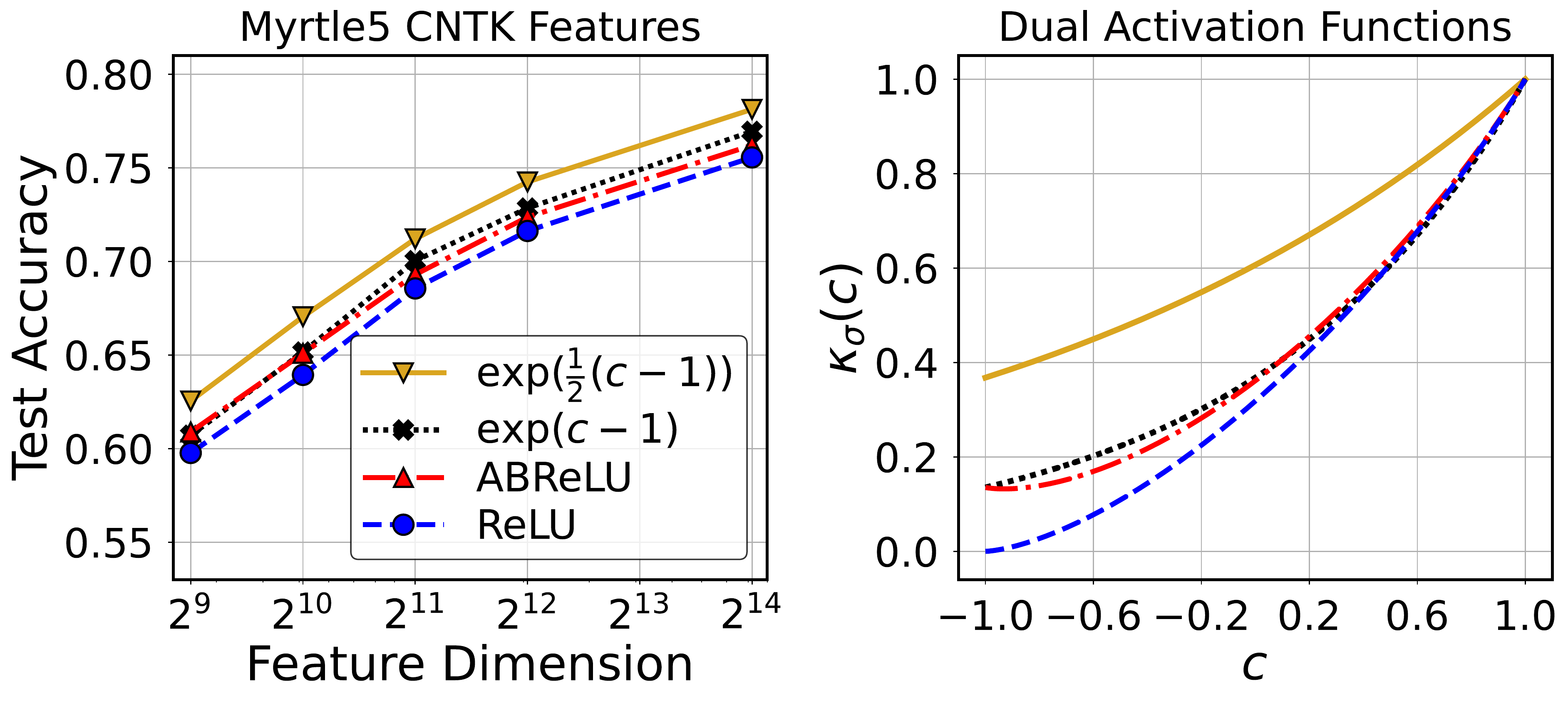}
	\vspace{-0.05in}
	\caption{Test accuracy of CIFAR-10} \label{fig-cifar10}
	\vspace{-0.15in}
\end{wrapfigure}
\paragraph{Performance on CIFAR-10 classification.} We also benchmark the proposed CNTK approximating via sketching algorithm. We perform CIFAR-10 classification~\cite{krizhevsky2009learning} by solving the ridge regression problem. The image classes are converted into $10$-dimensional one-hot vectors and inputs are pre-processed with regularized ZCA~\cite{shankar2020neural,lee2020finite}.  We report the best test accuracy among $20$ choices of ridge parameters in $\{ 10^{-10+\frac{12}{19}i} \mid i=0,1,\dots,19\}$.  We extract CNTK features of a $5$-layer convolutional neural network (known as Myrtle5~\cite{shankar2020neural}) without pooling by setting degree $q=8$ and  explore feature dimension $m=\{2^{9}, \dots, 2^{14}\}$ and homogeneous dual kernels including ReLU, ABReLU activations as well as deep normalized Gaussian kernels with $2$ scaling factors. See \cref{sec-appendix-abrelu} for more details. In \cref{fig-cifar10}, the test accuracy of neural kernels ({\bf left}) and the corresponding their dual activations ({\bf right}) are plotted. The dual activation of ABReLU is very similar to the normalized Gaussian without scaling and their test performances are also comparable.
We observe that the scaled normalized Gaussian shows the best performance which achieves 78.13\% while the ReLU CNTK features~\cite{zandieh2021scaling} shows 75.56\% with the same runtime. This is because the coefficients decay of the normalized Gaussian is faster than that of the ReLU, which leads to a lower approximation error of sketching algorithm. We also perform comparison among different activation functions in neural kernels in \cref{sec-appendix-guass-hermite-quadrature}. 

\paragraph{Speedup.} We observe that the exact CNTK of Myrtle-5 constructs
a kernel matrix of size $60{,}000\times60{,}000$ and achieves 86-87\% test accuracy. However, this requires approximately 151 GPU hours.
Under the same setting, our CNTK features for the normalized Gaussian kernel take about 1.4 GPU hours, i.e. a  $106 \times$ speedup. 
If we use less training data to construct $20{,}000 \times 20{,}000$ kernel matrix, the accuracy is about 77\% accuracy and the runtime is 16.8 GPU hours in which our approximation is still $12 \times$ faster without loss of accuracy.  We believe such acceleration through our methods open the door to using neural kernels in a wide range of research domains.

\section{Discussion}
In this work, we introduced methods to efficiently compute neural kernels for general activations. As activation functions play an important role in modern neural network architectures, we hope our work could empower researchers to explore properties of activations in a more principled way. We are excited with sketching method's compute efficiency by orders of magnitude on highly performant neural kernels to open up applications in dataset distillation~\cite{nguyen2021dataset} or uncertainty critical problems~\cite{adlam2021exploring} such as autonomous driving, healthcare and science.

\section*{Acknowledgements}

Amir Zandieh was supported by the Swiss NSF grant No. P2ELP2\textunderscore 195140. Amin Karbasi acknowledges funding in direct support of this work from NSF (IIS-1845032), ONR (N00014- 19-1-2406), and the AI Institute for Learning-Enabled Optimization at Scale (TILOS). We thank Timothy Nguyen and Jeffrey Pennington for discussions and feedback on the project. 

\bibliographystyle{unsrtnat}
\bibliography{references.bib}




\newpage
\appendix

\section{Sketching Preliminaries} \label{sec-appendix-polysketch}

The {\sc PolySketch} algorithm is a norm-preserving dimensionality reduction that can be applied to the tensor product of multiple vectors very quickly~\cite{ahle2020oblivious}, i.e., for any $v_1, \dots, v_q \in \Rbb^d$, there exists a randomized mapping $Q^q:\Rbb^{d^q}\rightarrow\Rbb^m$ which satisfies that
\begin{align*}
    \norm{ Q^q\left(v_1 \otimes \cdots \otimes v_p \right) }_2 \approx \norm{v_1 \otimes \cdots \otimes v_p }_2
\end{align*}
with high probability and $Q^q\left(v_1 \otimes \cdots \otimes v_p\right)$ can be computed very fast. Here, sketching dimension $m$ is a trade-off parameter between runtime and accuracy.
\cref{alg-polysketch} describes the pseudo-code of \textsc{PolySketch} and \cref{lem-polysketch} summarizes Theorems 1.2 and 1.3 of \cite{ahle2020oblivious} which guarantees spectral approximation of output of \textsc{PolySketch}.

\begin{algorithm}[h]
	\caption{\textsc{PolySketch}~\cite{ahle2020oblivious}} \label{alg-polysketch}
	\begin{algorithmic}[1]
		\STATE {\bf input}:  $x \in \Rbb^{d}$, degree $q$, sketch dimension $m$, SRHT instances  $\{ S^0_j :\Rbb^d \rightarrow \Rbb^m\}_{j=1}^q$ and $\{ S_j^i : \Rbb^{m^2} \rightarrow \Rbb^m \mid j=1, \dots, 2^{\lceil \log_2 q\rceil - i}, i=1, \dots, \lceil \log_2 q\rceil\}$
		\STATE let $\overline{q} \leftarrow 2^{\lceil \log_2 q \rceil}$
		\STATE for every $j = 1, \dots, q$, let $y_j^{0} \leftarrow S^0_j \cdot x$
		\STATE for every $j = q+1, \dots,  \overline{q}$, let $y_j^{0} \leftarrow S^{0}_j \cdot e_1$ where $e_1\in\Rbb^d$ is the first column vector of $\I_d$ 
		\STATE {\bf for}  { $i = 1 , \dots, \log_2 \overline{q}$} 
		\STATE ~~~~ {\bf for}  { $j = 1, \dots, \overline{q} / 2^i$ }
		\STATE ~~~~~~~~ compute $y_j^{i} \leftarrow S_{j}^{i} \cdot \left(y_{2j-1}^{i-1} \otimes y_{2j}^{i-1}\right)$
		\STATE {\bf return} $z = y_1^{\log_2 \overline{q}}$
	\end{algorithmic}
\end{algorithm}
\begin{algorithm}[h]
	\caption{Subsampled Randomized Hadamard Transform (\textsc{SRHT})}
	\begin{algorithmic}[1]
		\STATE {\bf input}: $ x \in \mathbb{R}^d$, dimension $m$, random signs $s \in \{ +1, -1\}^d$, random  indices $b \in \{1,\dots, d\}^m$
		\STATE let $y \leftarrow \left[ x_1 s_1, x_2 s_2, \dots, x_d s_d\right]$
		\STATE compute $ z \leftarrow \mathrm{FFT}(y)$
		\STATE {\bf return} $\frac1{\sqrt{m}} \left[ z_{b_1}, \dots, z_{b_m} \right] $
	\end{algorithmic}
\end{algorithm}

\begin{restatable}[{\sc PolySketch}]{theorem}{sodaresults}\label{lem-polysketch}
    For every integers $p,d\ge 1$ and every $\varepsilon, \delta>0$, there exists a distribution on random matrices $Q^p \in \Rbb^{m \times d^p}$, called degree $p$ {\sc PolySketch} such that {\bf (1)} for some $m=\bigo\left(\frac{p}{\varepsilon^{2}} \log^3 \frac{1}{\varepsilon\delta} \right)$ and any $y \in \Rbb^{d^p}$, $\Pr\left[ \|Q^p y\|_2^2 \in (1\pm \varepsilon)\|y\|_2^2 \right] \ge 1 - \delta$; {\bf (2)} for any $x \in \Rbb^d$, the total time to compute~$Q^p x^{\otimes p}$ is $\bigo\left( p m \log m + \frac{p^{3/2}}{\varepsilon}\log\frac{1}{\delta}~{\rm nnz}(x) \right)$; {\bf (3)} for any collection of vectors $v_1,\dots,v_p \in \Rbb^d$, the time to compute $Q^p \left(v_1 \otimes \dots \otimes v_p\right)$ is bounded by $\bigo\left( p m \log m + \frac{p^{3/2}}{\varepsilon} d \log \frac{1}{\delta} \right)$; {\bf (4)} for any $\lambda>0$ and any matrix $\A \in \Rbb^{d^p \times n}$, where the statistical dimension of $\A^\top \A$ is $s_\lambda$, there exists some $m = \bigo\left( \frac{p^4 s_\lambda}{\varepsilon^2} \log^3 \frac{n}{\varepsilon\delta} \right)$ such that,
    \begin{align}
        \Pr \left[ (1-\varepsilon) \left(\A^\top \A + \lambda \I_n \right) \preceq (Q^p \A)^\top (Q^p \A)  + \lambda \I_n \preceq (1+\varepsilon)\left( \A^\top \A + \lambda \I_n \right) \right] \ge 1 - \delta.
    \end{align}
\end{restatable}

\section{Proofs}


\subsection{Properties of Hermite Polynomials}


We first introduce that Hermite polynomials can be used as the random feature of monomial kernels for inputs on the unit sphere, which will be used in our analysis.
\begin{restatable}{prop}{hermiteorthogonality} \label{prop-hermite}
For $x, y \in \SS^{d-1}$, it holds that
\begin{align}
    \Ebb_{w \sim \mathcal{N}(0, \I_d)} \left[ h_\ell(\inner{w,x})~h_m(\inner{w,y})\right] = \ell!\inner{x,y}^\ell\cdot\mathbbm{1}_{\{\ell=m\}}.
\end{align}
\end{restatable}

\begin{proofof}{\cref{prop-hermite}}
Let $a := \inner{w,x}, b:= \inner{w,y}$ then $\Ebb_w[a] = \Ebb_w[b] = 0$ and $\mathrm{Cov}(a,b) = \Ebb[ab] = \inner{x,y}$. Hence, we have that 
\begin{align}
    \Ebb_{w \sim \mathcal{N}(0, \I_d)} \left[ h_\ell(\inner{w,x})~h_m(\inner{w,y})\right]
    = \Ebb_{(a,b) \sim \mathcal{N}(0, \Sigma)}\left[ h_\ell(a)~h_m(b)\right]
\end{align}
where 
\begin{align}
\Sigma = \begin{bmatrix} \norm{x}_2^2 & \inner{x,y} \\ \inner{x,y} & \norm{y}_2^2 \end{bmatrix}=\begin{bmatrix} 1 & \inner{x,y} \\ \inner{x,y} & 1 \end{bmatrix}.
\end{align}
We introduce Proposition 11.31 in \citet{o2014analysis}: 
\begin{align}
    \Ebb_{(a,b) \sim \mathcal{N}(0, \Sigma)}\left[ h_\ell(a)~h_m(b)\right] = \ell!\cdot\inner{x,y}^\ell\cdot\mathbbm{1}(\ell=m).
\end{align}
This completes the proof of \cref{prop-hermite}.
\end{proofof}



\subsection{Proof of \cref{thm-dual-kernel-expansion}} \label{sec-proof-thm-dual-kernel-expansion}

\nonpolyact*

\begin{proofof}{\cref{thm-dual-kernel-expansion}}
Due to homogeneity of the inner-product, we can write
\begin{align} \label{eq-sigma-taylor}
    \widetilde{\sigma}(\inner{w,x}) = \sum_{i=0}^q a_i \inner{w,x}^i  = \sum_{i=0}^q a_i \norm{x}_2^i  \inner{w,\frac{x}{\norm{x}_2}}^i.
\end{align}
Note that monomial $t^i$ of degree $i\geq 0$ can be explicitly written in the Hermite basis as $t^i = \sum_{\ell=0}^i \mu_{i,\ell}  h_\ell(t)$ where 
\begin{align} \label{eq-coeff-mono-to-hermite}
\mu_{i,\ell} = 
    \begin{dcases}
    \frac{i!}{2^{\frac{i-\ell}{2}} \cdot (\frac{i-\ell}{2})! \cdot \ell!}  &\text{if}~i-\ell \text{ is even},\\ 
    0 &\text{if}~i-\ell \text{ is odd}.
    \end{dcases}
\end{align}
Plugging this into \cref{eq-sigma-taylor} and re-arranging terms, we obtain that
\begin{align} \label{eq-sigma-hermite}
    \widetilde{\sigma}(\inner{w,x}) = \sum_{\ell=0}^q \left( \sum_{i=\ell}^q a_i~\mu_{i,\ell} \norm{x}_2^i  \right) h_\ell\left(\inner{w,\frac{x}{\norm{x}_2}}\right).
\end{align}
Applying \cref{eq-sigma-hermite} to the definition of the dual kernel $K_{\widetilde{\sigma}}(x,y)$ given in \cref{eq-random-features-sigma} and taking the expectation over $w$ gives
\begin{align}
    K_{\widetilde{\sigma}}(x,y) &= \Ebb_{w\sim\mathcal{N}(0,\I_d)}\left[ \widetilde{\sigma}(\inner{w,x})\cdot \widetilde{\sigma}(\inner{w,y}) \right] \nonumber \\
    &=\Ebb_w\left[ \sum_{\ell=0}^q \sum_{m=0}^q \left( \sum_{i=\ell}^q a_i~\mu_{i,\ell} \normx^i \right) \left( \sum_{j=m}^q a_j~\mu_{j,m} \norm{y}_2^j \right) h_\ell\left(\frac{\inner{w,x}}{\normx}\right)h_m\left(\frac{\inner{w,y}}{\normy}\right) \right] \nonumber \\
    &=\sum_{\ell=0}^q \sum_{m=0}^q \left( \sum_{i=\ell}^q a_i~\mu_{i,\ell} \normx^i \right) \left( \sum_{j=m}^q a_j~\mu_{j,m} \norm{y}_2^j \right) \Ebb_w\left[ h_\ell\left(\frac{\inner{w,x}}{\normx}\right)h_m\left(\frac{\inner{w,y}}{\normy}\right) \right] \nonumber \\
    &= \sum_{\ell=0}^q \left( \sum_{i=\ell}^q a_i~\mu_{i,\ell} \normx^i \right) \left( \sum_{j=\ell}^q a_j~\mu_{j,\ell} \norm{y}_2^j \right) \ell! \left(\frac{\inner{x,y}}{\normx \normy}\right)^\ell, \label{eq-hermite-ortho}
\end{align}
where \cref{eq-hermite-ortho} follows from \cref{prop-hermite}. Now using \cref{eq-coeff-mono-to-hermite}, we have
\begin{align*}
    \sum_{i=\ell}^q a_i\cdot \mu_{i,\ell} \norm{x}_2^i = \sum_{k=0}^{\lfloor \frac{q-\ell}{2} \rfloor}  \frac{a_{\ell+2k} \cdot (\ell+2k)! }{2^k \cdot k! \cdot \ell!} \norm{x}_2^{\ell+2k}.
\end{align*}
Therefore, we obtain that
\begin{align*}
    K_{\widetilde{\sigma}}(x,y) = \sum_{\ell=0}^q \left( \sum_{i=0}^{\lfloor \frac{q-\ell}{2} \rfloor} \frac{a_{\ell+2i} (\ell+2i)!}{2^i \cdot i! \sqrt{\ell!}} \norm{x}_2^{2i+\ell} \right) \left( \sum_{j=0}^{\lfloor \frac{q-\ell}{2} \rfloor} \frac{a_{\ell+2j} (\ell+2j)!}{2^j \cdot j! \sqrt{\ell!}} \norm{y}_2^{2j+\ell} \right) \left(\frac{ \inner{x,y}}{\norm{x}_2\norm{y}_2}\right)^\ell.
\end{align*}

We finish off the proof of \cref{thm-dual-kernel-expansion} by bounding the error $\abs{K_\sigma(x,y) - K_{\widetilde{\sigma}}(x,y)}$. 
We use \cref{eq-random-features-sigma} along with the assumption that both $x, y \neq 0$ to write,
\begin{align}
    \abs{K_\sigma(x,y) - K_{\widetilde{\sigma}}(x,y)} &= \abs{ \Ebb_w\left[ \sigma(\inner{w,x})\cdot \sigma(\inner{w,y}) - \widetilde{\sigma}(\inner{w,x})\cdot \widetilde{\sigma}(\inner{w,y}) \right]}\nonumber\\
    &\le \left| \Ebb_w\left[ \left( \sigma(\inner{w,x}) - \widetilde{\sigma}(\inner{w,x}) \right) \cdot \sigma(\inner{w,y}) \right] \right|\label{dual-ker-diffbound-first-term}\\ 
    &\qquad + \left| \Ebb_w\left[ \left( \sigma(\inner{w,y}) - \widetilde{\sigma}(\inner{w,y}) \right) \cdot \widetilde{\sigma}(\inner{w,x}) \right] \right|\label{dual-ker-diffbound-second-term}
\end{align}
where the inequality above follows from the triangle inequality. Now we bound each of \cref{dual-ker-diffbound-first-term} and \cref{dual-ker-diffbound-second-term} separately. First let us bound \cref{dual-ker-diffbound-first-term} using Cauchy–Schwarz inequality as follows,
\begin{align}
    \left| \Ebb_w\left[ \left( \sigma(\inner{w,x}) - \widetilde{\sigma}(\inner{w,x}) \right) \cdot \sigma(\inner{w,y}) \right] \right| &\le \sqrt{ \Ebb_w\left[ \left| \sigma(\inner{w,x}) - \widetilde{\sigma}(\inner{w,x}) \right|^2 \right] \cdot \Ebb_w\left[ \sigma(\inner{w,y})^2 \right] }\nonumber\\
    &= \sqrt{ \Ebb_{\alpha \sim \mathcal{N}(0, \norm{x}_2^2)}\left[ \left| \sigma(\alpha) - \widetilde{\sigma}(\alpha) \right|^2 \right] \cdot \Ebb_{\beta \sim \mathcal{N}(0, \norm{y}_2^2)}\left[ \sigma(\beta)^2 \right] }\nonumber\\
    &\le \sqrt{ \frac{ \Ebb_{\alpha \sim \mathcal{N}(0, \nu^2)}\left[ \left| \sigma(\alpha) - \widetilde{\sigma}(\alpha) \right|^2 \right]}{\norm{x}_2/\nu }  \cdot \Ebb_{\beta \sim \mathcal{N}(0, \norm{y}_2^2)}\left[ \sigma(\beta)^2 \right] }\nonumber\\
    &\le \sqrt{ \frac{ \varepsilon \cdot \nu }{\norm{x}_2 }  \cdot \frac{\nu}{\norm{y}_2} \Ebb_{\beta \sim \mathcal{N}(0, \nu^2)}\left[ \sigma(\beta)^2 \right] },\label{dual-ker-diffbound-first-term-bound}
\end{align}
where the first line follows from the Cauchy–Schwarz inequality
and the third line above follows from the assumption that $\normx, \normy \neq 0$ and the last line follows from the precondition of \cref{thm-dual-kernel-expansion}. 

Similarly, we can bound \cref{dual-ker-diffbound-second-term}, as follows,
\begin{align}
    &\left| \Ebb_w\left[ \left( \sigma(\inner{w,y}) - \widetilde{\sigma}(\inner{w,y}) \right) \cdot \widetilde{\sigma}(\inner{w,x}) \right] \right|\nonumber\\ 
    &\qquad\le \sqrt{\Ebb_w\left[ \abs{\left( \sigma(\inner{w,y}) - \widetilde{\sigma}(\inner{w,y}) \right)}^2  \right] \cdot \Ebb_w\left[ \abs{\widetilde{\sigma}(\inner{w,x})}^2  \right] }\nonumber\\
    &\qquad= \sqrt{ \Ebb_{\alpha \sim \mathcal{N}(0, \norm{y}_2^2)}\left[ \left| \sigma(\alpha) - \widetilde{\sigma}(\alpha) \right|^2 \right] \cdot \Ebb_{\beta \sim \mathcal{N}(0, \norm{x}_2^2)}\left[ \abs{\widetilde{\sigma}(\beta)}^2 \right] }\nonumber\\
    &\qquad\le \sqrt{ \frac{ \Ebb_{\alpha \sim \mathcal{N}(0, \nu^2)}\left[ \left| \sigma(\alpha) - \widetilde{\sigma}(\alpha) \right|^2 \right]}{ \norm{y}_2/\nu }  \cdot  \Ebb_{\beta \sim \mathcal{N}(0, \norm{x}_2^2)}\left[\abs{ \widetilde{\sigma}(\beta)}^2 \right]  }\nonumber \\
    &\qquad\le \sqrt{ \frac{ \varepsilon \cdot  \nu }{ \norm{y}_2 } \cdot  \Ebb_{\beta \sim \mathcal{N}(0, \norm{x}_2^2)}\left[2 \abs{\sigma(\beta)}^2 + 2 \abs{\widetilde{\sigma}(\beta) - \sigma(\beta)}^2 \right] }\nonumber\\
    &\qquad \le \sqrt{ \frac{ \varepsilon \cdot  \nu }{ \norm{y}_2 } \cdot  \frac{\nu }{ \norm{x}_2 }  \Ebb_{\beta \sim \mathcal{N}(0, \nu^2)}\left[2 \abs{\sigma(\beta)}^2 + 2 \abs{\widetilde{\sigma}(\beta) - \sigma(\beta)}^2 \right] }\nonumber\\
    &\qquad \le \sqrt{ \frac{ \varepsilon \cdot  \nu }{ \norm{y}_2 } \cdot  \frac{2 \nu (\varepsilon + \Ebb_{\beta \sim \mathcal{N}(0, \nu^2)}[ \abs{\sigma(\beta)}^2])}{ \normx }  }, \label{dual-ker-diffbound-second-term-bound}
\end{align}
where the fourth line above follows from the assumption that $\normx, \normy \neq 0$, the fifth line above follows from the AM-GM inequality along with the the preconditions of \cref{thm-dual-kernel-expansion}, and the
last equality above follows from the preconditions of \cref{thm-dual-kernel-expansion}. 

Now by plugging \cref{dual-ker-diffbound-first-term-bound} and \cref{dual-ker-diffbound-second-term-bound} back into \cref{dual-ker-diffbound-first-term} and \cref{dual-ker-diffbound-second-term} we find that,
\begin{align*}
    \abs{K_\sigma(x,y) - K_{\widetilde{\sigma}}(x,y)}
    &\le \sqrt{ \frac{ \varepsilon \cdot \nu^2 \cdot \Ebb_{\beta \sim \mathcal{N}(0, \nu^2)}\left[ \sigma(\beta)^2 \right]}{\norm{x}_2 \cdot \norm{y}_2} }\left( 1 + \sqrt{ 2 + \frac{2 \varepsilon}{\Ebb_{\beta \sim \mathcal{N}(0, \nu^2)}\left[ \sigma(\beta)^2 \right] } } \right) \\
    &\le \sqrt{ \frac{ \varepsilon \cdot \nu^2 \cdot \Ebb_{\beta \sim \mathcal{N}(0, \nu^2)}\left[ \sigma(\beta)^2 \right]}{\normx \cdot \normy} } \cdot \sqrt{ 6 + \frac{4 \varepsilon}{\Ebb_{\beta \sim \mathcal{N}(0, \nu^2)}\left[ \sigma(\beta)^2 \right] } }\\
    &= \sqrt{ \frac{ \varepsilon \cdot \nu^2}{\norm{x}_2 \cdot \norm{y}_2} \cdot \left( 6\Ebb_{\beta \sim \mathcal{N}(0, \nu^2)}\left[ \sigma(\beta)^2 \right] + 4 \varepsilon \right)}.
\end{align*}
This completes the proof of \cref{thm-dual-kernel-expansion}.
\end{proofof}

\paragraph{Examples for Taylor expansion.}  
Observe that $\sigma(t) = \sin(t)$ is analytic and has a Taylor expansion with coefficients $a_{\ell+2i} = \frac{(-1)^{\frac{\ell+2i-1}{2}}}{(\ell+2i)!} \cdot \mathbbm{1}(\ell \text{ is odd})$. By invoking \cref{thm-dual-kernel-expansion} we have, 
\begin{align}
    r_{\sigma,\ell}(t) = \mathbbm{1}(\ell \text{ is odd})\cdot \sum_{i=0}^\infty \frac{(-1)^{\frac{\ell+2i-1}{2}}}{(\ell+2i)!} \frac{(\ell+2i)!}{2^i \cdot i! \cdot \sqrt{\ell!}} t^{2i+\ell}
    = \mathbbm{1}(\ell \text{ is odd})\cdot \frac{(-1)^{\frac{\ell-1}{2}} t^\ell}{\sqrt{\ell!}} \cdot e^{-\frac{t^2}{2}}.
\end{align}
Therefore,
\begin{align}
    K_{\sin}(x,y) = \sum_{\ell=0}^\infty e^{-\frac{\norm{x}_2^2}{2}} \cdot e^{-\frac{\norm{y}_2^2}{2}} \cdot \frac{\inner{x,y}^{2\ell+1}}{(2\ell+1)!}
    = e^{-\frac{\norm{x}_2^2 + \norm{y}_2^2}{2}} \sinh(\inner{x,y}).
\end{align}
Similarly, we can derive that $K_{\cos}(x,y) = e^{-\frac{\norm{x}_2^2 + \norm{y}_2^2}{2}} \cosh(\inner{x,y})$ that corresponds to \cref{tab-dual-kernel}.

\subsection{Proof of \cref{thm-dual-kernel-error-hermite-expansion}} \label{sec-proof-thm-dual-kernel-error-hermite-expansion}

\dualkernelerrorhermiteexpansion*

In order to prove this theorem we first need to establish the following bound on the decay rate of the Hermite expansion coefficients of smooth functions,
\begin{restatable}{lemma}{lmmhermiteexpansioncoefficients} \label{lmm-hermite-expansion-coefficients}
Suppose that there exists an integer $k \geq 0$ such that for every $i=0,\dots,k$, $\sigma^{(i)}(t)$ are absolutely continuous in $\Rbb$ and $\lim_{t\rightarrow \pm\infty} e^{-\frac{t^2}{4}} \sigma^{(i)}(t) = 0$. Assume that $\norm{\sigma}_{\mathcal{N}}^2  < \infty$ and  $\norm{\sigma^{(k)}}_{\mathcal{N}}^2 < \infty$.
Let $\{c_j\}_{j=0}^\infty$ be the Hermite expansion coefficients of this function such that $\norm{\sigma - \sum_{j=0}^\infty c_j h_j}_{\mathcal{N}} = 0$. Then, for any integer $j \ge k$:
\begin{align}
    \abs{c_j} \leq \norm{\sigma^{(k)}}_{\mathcal{N}} \frac{\sqrt{(j-k)!}}{j!}. \label{eq-hermite-expansion-coefficients}
\end{align}
\end{restatable}

The proof of \cref{lmm-hermite-expansion-coefficients} is provided in \cref{sec-proof-lmm-hermite-expansion-coefficients}. 

\begin{proofof}{\cref{thm-dual-kernel-error-hermite-expansion}}
First, because of the precondition of \cref{thm-dual-kernel-error-hermite-expansion} about $\norm{\sigma}^2_{\mathcal{N}(0,\nu^2)} = \Ebb_{t\sim \mathcal{N}(0,\nu^2)}\left[\abs{\sigma(t)}^2 \right] = \Ebb_{t\sim \mathcal{N}(0,1)}\left[\abs{\sigma(\nu t)}^2 \right] < \infty$, the function $\sigma(\nu t)$ is an $L^2$ function with respect to the normal measure $\mathcal{N}(0,1)$ on the real line. Therefore, because the Hermite polynomials $\{h_j\}_{j=0}^\infty$ provide an orthogonal basis for $L^2$ function with respect to normal measure $\mathcal{N}(0,1)$, $\sigma(\nu t)$ converges to its Hermite expansion, i.e., $\Ebb_{t\sim \mathcal{N}(0,1)}\left[ \abs{\sigma(\nu t) - \sum_{j=0}^\infty c_j h_j(t)}^2 \right] = 0$.
We obtain an error bound on the dual kernel by invoking \cref{thm-dual-kernel-expansion}. To do so, we need to first upper bound $\Ebb_{t\sim \mathcal{N}\left( 0, \nu^2 \right)} \left[ \abs{\sigma(t) - \tsigma(t)}^2 \right]$, as follows
\begin{align}
\Ebb_{t\sim \mathcal{N}\left( 0, \nu^2 \right)} \left[ \abs{\sigma(t) - \tsigma(t)}^2 \right]
    &= \Ebb_{t\sim \mathcal{N}\left( 0, 1 \right)} \left[ \abs{\sigma(\nu t) - \tsigma(\nu t)}^2 \right]\nonumber\\
    &= \sum_{j=q+1}^\infty |c_j|^2 \cdot \Ebb_{t\sim \mathcal{N}\left( 0, 1 \right)} \left[ \abs{h_j(t)}^2 \right]\nonumber\\
    &= \sum_{j=q+1}^\infty |c_j|^2 \cdot j!, \label{bound-polynomial-difference}
\end{align}
where the second line above follows from the fact that $h_j$'s are orthogonal with respect to the normal measure $\mathcal{N}\left( 0, 1 \right)$. The third line follows from the fact that $\norm{h_j}_{\mathcal{N}}^2=j!$.

We now proceed to upper bound the term in \cref{bound-polynomial-difference}, using the bound on the Hermite expansion coefficients we proved in \cref{lmm-hermite-expansion-coefficients}. We apply this lemma to the function $\sigma(\nu t)$ whose Hermite expansion coefficients are $\{c_i\}_{i=0}^\infty$. By precondition of \cref{thm-dual-kernel-error-hermite-expansion} we have $\norm{\sigma^{(k)}}_{\mathcal{N}}^2 = \Ebb_{t \sim \mathcal{N}(0,\nu^2)}\left[ \abs{\sigma^{(k)}(t)}^2 \right] < \infty$. This implies that,
\[
\Ebb_{t \sim \mathcal{N}(0,1)}\left[ \abs{\frac{d^k}{dt^k}\sigma(\nu t)}^2 \right] = \nu^{2k} \cdot \Ebb_{t \sim \mathcal{N}(0,1)}\left[ \abs{\frac{d^k}{d(\nu t)^k}\sigma(\nu t)}^2 \right] = \nu^{2k} \cdot \norm{\sigma^{(k)}}_{\mathcal{N}(0,\nu^2)}^2 < \infty.
\]
Furthermore, the precondition of \cref{thm-dual-kernel-error-hermite-expansion} about $\lim_{t\rightarrow \pm\infty} e^{-\frac{t^2}{4}} \sigma^{(i)}(\nu t) = 0$ implies the following,
\[
\lim_{t\rightarrow \pm\infty} e^{-\frac{t^2}{4}} \frac{d^i}{dt^i}\sigma(\nu t) = \nu^i \cdot \lim_{t\rightarrow \pm\infty} e^{-\frac{t^2}{4}} \frac{d^i}{d(\nu t)^i}\sigma(\nu t) = 0.
\]
Therefore, the preconditions of \cref{lmm-hermite-expansion-coefficients} are satisfied and by invoking this lemma we have the following inequality for any integer $j \ge k$,
\[
\abs{c_j} \leq \sqrt{\Ebb_{t \sim \mathcal{N}(0,1)}\left[ \abs{\frac{d^k}{dt^k}\sigma(\nu t)}^2 \right] } \cdot \frac{\sqrt{(j-k)!}}{j!} = \nu^k \cdot \norm{\sigma^{(k)} }_{\mathcal{N}(0,\nu^2)} \cdot \frac{\sqrt{(j-k)!}}{j!}.
\]
Plugging the above inequality into \cref{eq-hermite-expansion-coefficients} into \cref{bound-polynomial-difference}, gives
\begin{align}
    \Ebb_{t\sim \mathcal{N}\left( 0, \nu^2 \right)} \left[ \abs{\sigma(t) - \tsigma(t)}^2 \right]
    &= \sum_{j=q+1}^\infty |c_j|^2 \cdot j! \nonumber \\
    &\leq \nu^{2k} \cdot \norm{\sigma^{(k)} }_{\mathcal{N}(0,\nu^2)}^2 \cdot \sum_{j=q+1}^\infty \frac{(j-k)!}{j!} \nonumber \\
	&= \frac{\nu^{2k} \cdot \norm{\sigma^{(k)} }_{\mathcal{N}(0,\nu^2)}^2}{k-1} \cdot \frac{1}{q (q-1) \cdots (q-k+2)} \nonumber \\
	&\leq\frac{\nu^{2k} \cdot \norm{\sigma^{(k)} }_{\mathcal{N}(0,\nu^2)}^2}{k \cdot q^{k-1}} \label{eq-truncate-error-norm-bound}
\end{align}
Thus we can now invoke \cref{thm-dual-kernel-expansion} with $\varepsilon = \frac{\nu^{2k} \cdot \norm{\sigma^{(k)} }_{\mathcal{N}(0,\nu^2)}^2 }{k \cdot q^{k-1}}$ to find that
\begin{align*}
    &\left| K_\sigma(x,y) - K_{\tsigma}(x,y) \right| \\
    &\le \sqrt{ \frac{ \varepsilon \cdot \nu^2}{\norm{x}_2 \norm{y}_2}  \left( 6\norm{\sigma}_{\mathcal{N}(0,\nu^2)}^2  + 4 \varepsilon \right) }\\
	&\le \sqrt{ \frac{  \nu^2}{\norm{x}_2 \norm{y}_2}  \left( 6\norm{\sigma}_{\mathcal{N}(0,\nu^2)}^2   \frac{\nu^{2k} \cdot \norm{\sigma^{(k)} }_{\mathcal{N}(0,\nu^2)}^2 }{k \cdot q^{k-1}} + 4 \left( \frac{\nu^{2k} \cdot \norm{\sigma^{(k)} }_{\mathcal{N}(0,\nu^2)}^2 }{k \cdot q^{k-1}} \right)^2 \right) }\\ 
	&\leq  \sqrt{ \frac{  \nu^2}{\norm{x}_2 \norm{y}_2}}  \left( \sqrt{ 6\norm{\sigma}_{\mathcal{N}(0,\nu^2)}^2   \frac{\nu^{2k} \cdot \norm{\sigma^{(k)} }_{\mathcal{N}(0,\nu^2)}^2 }{k \cdot q^{k-1}} } + \sqrt{4 \left( \frac{\nu^{2k} \cdot \norm{\sigma^{(k)} }_{\mathcal{N}(0,\nu^2)}^2 }{k \cdot q^{k-1}} \right)^2 }\right) 
	\\
	&\leq  \frac{  \nu}{\sqrt{  \norm{x}_2 \norm{y}_2}}  \left( \norm{\sigma}_{\mathcal{N}(0,\nu^2)} \nu^{k}  \norm{\sigma^{(k)} }_{\mathcal{N}(0,\nu^2)} \sqrt{   \frac{6  }{k \cdot q^{k-1}} } + 2 \left( \frac{\nu^{2k} \cdot \norm{\sigma^{(k)} }_{\mathcal{N}(0,\nu^2)}^2 }{k \cdot q^{k-1}} \right) \right) \\
	&\leq  \frac{  \nu^{k+1} \norm{\sigma^{(k)} }_{\mathcal{N}(0,\nu^2)}}{\sqrt{ \norm{x}_2 \norm{y}_2}}  \max\left( \norm{\sigma}_{\mathcal{N}(0,\nu^2)} ,  \nu^{k}  \norm{\sigma^{(k)} }_{\mathcal{N}(0,\nu^2)}\right) \left( \sqrt{   \frac{6  }{k \cdot q^{k-1}} } + \frac{2 }{k \cdot q^{k-1}} \right) \\	
	&\leq  \frac{  5 \nu^{k+1} \norm{\sigma^{(k)}  }_{\mathcal{N}(0,\nu^2)}  \max\left( \norm{\sigma}_{\mathcal{N}(0,\nu^2)} ,  \nu^{k}  \norm{\sigma^{(k)} }_{\mathcal{N}(0,\nu^2)}\right)  }{\sqrt{ \norm{x}_2 \norm{y}_2 \cdot k \cdot q^{k-1}}}.
\end{align*}

Now we prove the second statement of the theorem about the ReLU activation $\sigma(t) = \max(t,0)$. It is easy to check that for this function
\begin{align}
    \norm{\sigma}_{\mathcal{N}(0,\nu^2)}^2 = \Ebb_{t\sim \mathcal{N}(0,\nu^2)}[\abs{\sigma(t)}^2] = \frac{\nu^2}{\sqrt{2\pi}} \int_{0}^\infty t^2 \cdot e^{-\frac{t^2}{2}} dt= \frac{\nu^2}{2}. \label{eq-relu-hermite-norm}
\end{align}
Furthermore for any $j \geq 0$, the Hermite coefficients of $\sigma(\nu t)$ are
\begin{align*}
c_j = \frac{1}{\sqrt{2 \pi} j!}\int_{-\infty}^\infty \max(\nu t,0) \cdot h_j(t) \cdot e^{-\frac{t^2}{2}} dt
= \frac{\nu}{\sqrt{2 \pi} j!}\int_{0}^\infty t \cdot h_j(t) \cdot e^{-\frac{t^2}{2}} dt 
\end{align*}
Using integration-by-parts and the fact that $h_j'(t) = j h_{j-1}(t)$ for all $j \geq 1$, we get that
\begin{align*}
    \int_{0}^\infty t \cdot h_j(t) \cdot e^{-\frac{t^2}{2}} dt 
    &= h_j(x) \left(-e^{-\frac{t^2}{2}}\right)\bigg\vert_{0}^{\infty} + \int_{0}^\infty h_j'(t) \cdot e^{-\frac{t^2}{2}} dt \\
    &= h_j(0) + \int_{0}^\infty h_j'(t) \cdot e^{-\frac{t^2}{2}} dt \\
    &= h_j(0) + j \int_{0}^\infty h_{j-1}(t) \cdot e^{-\frac{t^2}{2}} dt \\
    &= (-1)^{\frac{j}{2}} \cdot(j-1)!!\cdot \mathbbm{1}_{\{j \text{ is even}\}} + j \cdot(-1)^{\frac{j}{2}-1}\cdot (j-3)!!\cdot \mathbbm{1}_{\{j \text{ is even}\}}\\
    &= (-1)^{\frac{j}{2} - 1} \cdot (j-3)!! \cdot \mathbbm{1}_{\{j \text{ is even}\}}.
\end{align*}
Therefore,
\begin{align}
    \Ebb_{t\sim \mathcal{N}\left( 0, \nu^2 \right)} \left[ \abs{\sigma(t) - \tsigma(t)}^2 \right]
    &= \sum_{j=q+1}^\infty c_j^2 \cdot \sqrt{2 \pi} j! 
    = \sum_{j=q+1}^\infty \nu^2 \cdot \frac{\left( (j-3)!!\right)^2}{\sqrt{2 \pi} j!} = \frac{\nu^2}{\sqrt{2 \pi}(q+1)}, \label{eq-delta-hermite-norm}
\end{align}
By invoking \cref{thm-dual-kernel-expansion}, using \cref{eq-relu-hermite-norm} and \cref{eq-delta-hermite-norm}, we have
\begin{align*}
\left| K_\sigma(x,y) - K_{\tsigma}(x,y) \right| &\le \sqrt{ \frac{ \varepsilon \cdot \nu^2}{\norm{x}_2 \norm{y}_2}  \left( 6\norm{\sigma}_{\mathcal{N}(0,\nu^2)}^2 + 4 \varepsilon \right) }\\
&= \sqrt{ \frac{ \frac{\nu^2}{\sqrt{2 \pi}(q+1)} \cdot \nu^2}{\norm{x}_2 \norm{y}_2}  \left( 3\nu^2  + 4 \frac{\nu^2}{\sqrt{2 \pi}(q+1)} \right) }\\
&\le \sqrt{ \frac{ 2 \nu^6}{(q+1)\norm{x}_2 \norm{y}_2} }
\end{align*}
This completes the proof of \cref{thm-dual-kernel-error-hermite-expansion}.
\end{proofof}

\subsection{Proof of \cref{lmm-hermite-expansion-coefficients}} \label{sec-proof-lmm-hermite-expansion-coefficients}

\lmmhermiteexpansioncoefficients*

\begin{proofof}{\cref{lmm-hermite-expansion-coefficients}}
The proof can be obtained by slightly modifying Theorem 3.1 in~\cite{xiang2012asymptotics}.
The precondition $\norm{\sigma}_{\mathcal{N}}^2 = \Ebb_{t\sim \mathcal{N}(0,1)}\left[\abs{\sigma(t)}^2 \right] < \infty$ implies that $\sigma$ is an $L^2$-function with respect to measure $e^{-\frac{t^2}{2}}$ on real line. Because Hermite polynomials $\{h_j\}_{j=0}^\infty$ form an orthogonal basis for the Hilbert space of $L^2$-functions with respect to normal measure $\mathcal{N}(0,1)$, $\sigma(t)$ converges to its Hermite expansion, i.e., $\sum_{j=0}^\infty c_j h_j(t)$. The $j$-th coefficient in this expansion is
\begin{align}
c_j = \frac{1}{\sqrt{2 \pi} j!} \int_{-\infty}^\infty \sigma(t) \cdot h_j(t) \cdot e^{-\frac{t^2}{2}} dt. 
\end{align}
Using the Rodrigues' expression of Hermite polynomials in \cref{eq-defn-hermite-polynomial} and integration-by-parts, we have,
\begin{align}
	\int_{-\infty}^\infty \sigma(t) h_j(t) e^{-\frac{t^2}{2}} dt
	&= (-1)^j \int_{-\infty}^\infty \sigma(t)  \left[ \frac{d^j}{dt^j} e^{-\frac{t^2}{2}} \right] dt  \nonumber \\
	&= (-1)^j \sigma(t) \left[ \frac{d^{j-1}}{dt^{j-1}} e^{-\frac{t^2}{2}} \right] \Bigg\vert_{-\infty}^{\infty} +  (-1)^{j-1} \int_{-\infty}^\infty \sigma^{(1)}(t) \left[ \frac{d^{j-1}}{dt^{j-1}} e^{-\frac{t^2}{2}} \right] dt \nonumber \\
    &= - \sigma(t) \cdot h_{j-1}(t) \cdot e^{-\frac{t^2}{2}} \Bigg\vert_{-\infty}^{\infty} +  \int_{-\infty}^\infty \sigma^{(1)}(t) h_{j-1}(t)  e^{-\frac{t^2}{2}} dt	,\label{eq-hermite-coeff-intergration-by-parts}
\end{align}
where the last line above follows from the Rodrigues' expression of degree $j-1$ Hermite polynomial in \cref{eq-defn-hermite-polynomial}.
Therefore, using 22.14.17 in \citep{abramowitz1988handbook}\footnote{Equation 22.14.17 in \citep{abramowitz1988handbook} was $| H_j(t)| \leq a_0 e^{\frac{t^2}{2}}  2^{\frac{j}{2}} \sqrt{j!}$ where $a_0 \approx 1.086435$ and $H_j(\cdot)$ is physicist's Hermite polynomial. Using $H_j(t)=2^{\frac{j}{2}} h_j(\sqrt{2}t)$ gives that $|e^{-\frac{t^2}{4}} h_j(t)|\leq a_0 \sqrt{j!}$.}, the first term in \cref{eq-hermite-coeff-intergration-by-parts} is $0$ and by applying the above repeatedly we have
\begin{align}
    \abs{\int_{-\infty}^\infty \sigma(t)\cdot h_j(t) \cdot e^{-\frac{t^2}{2}} dt} 
    &=\abs{ \int_{-\infty}^\infty \sigma^{(1)}(t) \cdot h_{j-1}(t) \cdot e^{-\frac{t^2}{2}} dt} \nonumber \\
    &\qquad \vdots \nonumber \\
    &= \abs{ \int_{-\infty}^\infty \sigma^{(k)}(t) \cdot h_{j-k}(t) \cdot e^{-\frac{t^2}{2}} dt } \nonumber \\
    &\le \sqrt{ \int_{-\infty}^\infty \abs{\sigma^{(k)}(t)}^2 e^{-\frac{t^2}{2}}dt  \cdot  \int_{-\infty}^\infty \abs{h_{j-k}(t)}^2 e^{-\frac{t^2}{2}}dt} \nonumber \\
	&= \sqrt{2 \pi } \cdot \norm{\sigma^{(k)}}_{\mathcal{N}} \cdot \sqrt{(j-k)!}
\end{align}
where the second last inequality comes from Cauchy-Schwarz inequality and the last one holds from that $\Ebb_{t \sim \mathcal{N}(0,1)}\left[\abs{h_\ell(t)}^2\right] = \ell!$ and the assumption.  This completes the proof of \cref{lmm-hermite-expansion-coefficients}.
\end{proofof}

\subsection{Proof of \cref{lmm-derivative-dual-kernel}} \label{sec-proof-lmm-derivative-dual-kernel}
\lmmderivativedualkernel*
Note that our assumption on $\sigma$ can be weakened to be: there exists $\epsilon>0$ and constants $C_1$ and $C_2$ such that 
\begin{align}
    |\sigma''(t)| \leq C_1 e^{C_2 |t|^{2-\epsilon}}. 
\end{align}

\begin{proofof}{\cref{lmm-derivative-dual-kernel}}
Recall that the dual activation is defined as
\begin{align}
	k_\sigma(a,b,c) := \mathop{\Ebb}_{(u,v)\sim \mathcal{N}(0,{\bm\Lambda})}\left[\sigma(u) \sigma(v)\right].
\end{align}
where for $a,b \in \Rbb_{\geq 0}$ and $c \in [-1,1]$
\begin{align*}
	{\bm\Lambda}:=\begin{bmatrix} a^2 & {ab} c \\ {ab} c & b^2\end{bmatrix} =  \begin{bmatrix} a & 0  \\ bc & b\sqrt{1-c^2} \end{bmatrix}   \begin{bmatrix} a & 0  \\ bc & b\sqrt{1-c^2} \end{bmatrix}^\top
\end{align*}
Using a whitening transformation, we introduce the standard i.i.d. Gaussian random variables $w_1, w_2 \sim \mathcal{N}(0,1)$ that satisfy 
\begin{align*}
	\begin{bmatrix} u \\ v \end{bmatrix} = \begin{bmatrix} a & 0  \\ bc & b\sqrt{1-c^2} \end{bmatrix} \begin{bmatrix} w_1 \\ w_2 \end{bmatrix}.
\end{align*}
Thus, by denoting $w = \begin{bmatrix} w_1 \\ w_2 \end{bmatrix}$, the dual activation can be written as
\begin{align}
    k_{\sigma}(a,b,c) =  \mathop{\Ebb}_{w \sim \mathcal{N}(0,\I_2)}\left[\sigma(a w_1) \cdot \sigma(bc w_1 +  b\sqrt{1-c^2} w_2)\right]. \label{eq-dual-activation-2}
\end{align}

Using \cref{eq-dual-activation-2}, we can calculate $\frac{\partial}{\partial c} k_\sigma(\cdot,\cdot,c)$ if the derivative can be interchangeable with the expectation. To this end, we use the ``measure theory'' statement of \emph{Leibniz integral rule}.

\begin{lemma}[Measure theory statement of Leibniz integral rule, Theorem 6.28 of \cite{klenke2013probability}]\label{leibniz-rule}
    Let $\mu$ be a probability distribution with support $\Omega$, let $I\subset \Rbb$ be a nontrivial open interval, also let $f : \Omega \times I \to \Rbb$ be a map with the following properties:
    \begin{enumerate}
        \item For any $x \in I$, $\Ebb_{w \sim \mu}[|f(w,x)|] < \infty$.
        \item For almost all $w \in \Omega$, the map $x \to f(w,x)$ is differentiable with derivative $\frac{\partial}{\partial x} f(w,x)$.
        \item  There is a map $h: \Omega \to \Rbb$ with the property that $\Ebb_{w \sim \mu}[|h(w)|] < \infty$, such that $\abs{\frac{\partial}{\partial x} f(\cdot,x)} \le h$.
    \end{enumerate}
    Then, for any $x \in I$, $\Ebb_{w \sim \mu}\left[ \abs{ \frac{\partial}{\partial x} f(w,x) } \right] < \infty$ and the function $F: x \to \Ebb_{w \sim \mu}[ f(w,x) ]$ is differentiable with derivative
    \[ 
    F^{\prime}(x) = \Ebb_{w \sim \mu}\left[  \frac{\partial}{\partial x} f(w,x)  \right].
    \]
\end{lemma}

To invoke \cref{leibniz-rule} on the above expression of the dual kernel, we let $I$ be the open interval $(-1+\varepsilon, 1-\varepsilon)$ for an arbitrarily small $\varepsilon>0$ and $f: \Rbb^2 \times I \to \Rbb$ be the function defined as $f\left(w,c\right) := \sigma(a w_1) \cdot \sigma(bc w_1 +  b\sqrt{1-c^2} w_2)$ for some fixed values of $a,b$. 
With these notations in place, we proceed to check if the preconditions of \cref{leibniz-rule} are satisfied. 
To verify the first precondition, we need to show that for any $c \in I$, $\mathop{\Ebb}_{w \sim \mathcal{N}(0,\I_2)}\left[\abs{\sigma(a w_1) \cdot \sigma(bc w_1 +  b\sqrt{1-c^2} w_2) }\right] < \infty$. We find that,
\begin{align}
    \mathop{\Ebb}_{w \sim \mathcal{N}(0,\I_2)}\left[\abs{\sigma(a w_1) \cdot \sigma(bc w_1 +  b\sqrt{1-c^2} w_2) }\right] &\le \sqrt{\mathop{\Ebb}_{w_1}  \left[\abs{\sigma(a w_1)}^2  \right]  \mathop{\Ebb}_{w} \left[\abs{ \sigma \left(bc w_1 + b\sqrt{1-c^2} w_2 \right)}^2  \right]} \nonumber\\
    &= \sqrt{\mathop{\Ebb}_{w_1}  \left[\abs{\sigma(a w_1)}^2  \right] \cdot \mathop{\Ebb}_{\gamma\sim \mathcal{N}(0,1)} \left[\abs{ \sigma \left(b\gamma \right)}^2  \right]}, \label{eq:cauchy-schwarz-dual-kernel-def}
\end{align}
where the first line above follows from Cauchy–Schwarz inequality. The second line above follows from the fact that $w_1$ and $w_2$ are independent copies of the normal random variable $\mathcal{N}(0,1)$, thus the random variable $bc w_1 + b\sqrt{1-c^2} w_2$ is indeed $b \gamma$ for a normal $\gamma \sim \mathcal{N}(0,1)$. 
Now using the preconditions of \cref{lmm-derivative-dual-kernel}, for any $a,b \in (0, \nu]$ we have 
\begin{equation}\label{conseq-precondition-norm-bound}
\mathop{\Ebb}_{w_1\sim \mathcal{N}(0,1)} \left[\abs{\sigma(a w_1 )}^2  \right] \le
\frac{\nu}{a}  \norm{\sigma}_{\mathcal{N}(0,\nu^2)}^2 < \infty , \text{ and } \mathop{\Ebb}_{\gamma\sim \mathcal{N}(0,1)} \left[\abs{ \sigma \left(b\gamma \right)}^2  \right] \le \frac{\nu}{b}  \norm{\sigma}_{\mathcal{N}(0,\nu^2)}^2 < \infty.
\end{equation}
Also in case $a=0$ or $b=0$ we have $\Ebb_{w_1\sim \mathcal{N}(0,1)} \left[\abs{\sigma(a w_1 )}^2  \right] = \abs{\sigma(0)}^2 < \infty$,
therefore, above inequalities along with \cref{eq:cauchy-schwarz-dual-kernel-def} proves the first precondition of \cref{leibniz-rule} for any $a,b \in [0, \nu]$.

To verify that the second precondition of \cref{leibniz-rule} holds,
we show that for almost all $w_1,w_2 \in \Rbb$ the map $c \to \sigma(a w_1) \cdot \sigma(bc w_1 +  b\sqrt{1-c^2} w_2)$ is differentiable.
This holds true because of the assumption of \cref{lmm-derivative-dual-kernel} on about the activation $\sigma(\cdot)$ being differentiable. The derivative of this map is in fact 
\[
\sigma(a w_1) \cdot \sigma^{\prime}(bc w_1 +  b\sqrt{1-c^2} w_2) \cdot \left( b w_1 - \frac{bc}{\sqrt{1-c^2}} w_2 \right).
\]

Finally, we check the third precondition of \cref{leibniz-rule}. Since $|c|<1$, there is an $\varepsilon>0$ such that $c \in (-1+\varepsilon, 1-\varepsilon)$. We have,
\begin{align}
    &\abs{\sigma(a w_1) \cdot \sigma^{\prime}(bc w_1 +  b\sqrt{1-c^2} w_2) \cdot \left( b w_1 - \frac{bc}{\sqrt{1-c^2}} w_2 \right)} \nonumber \\ 
    &\qquad \le \abs{\sigma(a w_1) \cdot \sigma^{\prime}(bc w_1 +  b\sqrt{1-c^2} w_2)} \cdot \left( b |w_1| + \frac{b}{\varepsilon} |w_2| \right) \nonumber\\
    &\qquad \le C_1C_2 \exp \left( \frac{a^2 w_1^2 + ( b c w_1 + b \sqrt{1-c^2}w_2)^2}{4.1\nu^2} \right) \cdot \left( b |w_1| + \frac{b}{\varepsilon} |w_2| \right)\nonumber \\
    &\qquad \le C_1C_2 \exp \left( \frac{(a^2+b^2)w_1^2 + b^2 w_2^2}{4.1\nu^2} \right) \cdot \left( b |w_1| + \frac{b}{\varepsilon} |w_2| \right)\nonumber \\
    &\qquad \le C_1C_2 \exp \left( \frac{w_1^2 + w_2^2}{2.05} \right) \cdot \left( b |w_1| + \frac{b}{\varepsilon} |w_2| \right) =: h(w),\label{eq-dominant-bound}
\end{align}
where the second inequality follows from the preconditions of \cref{lmm-derivative-dual-kernel} about the upper bounds on $\sigma(\cdot)$ and $\sigma^{\prime}(\cdot)$,
the third one follows from $(cw_1 + \sqrt{1-c^2} w_2)^2 \leq w_1^2 + w_2^2$, and the fourth one follows from $a,b \leq \nu$. Now it is easy to check that this upper bound function satisfies $\Ebb_{w \sim \mathcal{N}(0,{\bm I}_2)}[|h(w)|] < \infty$.

Therefore, we can invoke \cref{leibniz-rule} to calculate the derivative of the dual kernel $k_{\sigma}(a,b,c)$ with respect to $c$ as follows,

\begin{align}
	\frac{\partial}{\partial c} k_{\sigma}(a,b,c) 
	&= \frac{\partial}{\partial c} ~ \mathop{\Ebb}_{w \sim \mathcal{N}(0,\I_2)}\left[\sigma(a w_1) \cdot \sigma(bc w_1 +  b\sqrt{1-c^2} w_2)\right] \nonumber \\
	&=  \mathop{\Ebb}_{w}\left[ \frac{\partial}{\partial c} \left( \sigma(a w_1) \cdot \sigma(bc w_1 +  b\sqrt{1-c^2} w_2)\right)\right] \nonumber \\
	&=\mathop{\Ebb}_{w}  \left[ \sigma(a w_1)  \cdot \sigma' \left(bc w_1 + b\sqrt{1-c^2} w_2 \right) \cdot b w_1 \right] \label{derivative-c-first-term}\\
	&\qquad \qquad - \mathop{\Ebb}_{w}  \left[ \sigma(a w_1) \cdot \sigma' \left(bc w_1 + b\sqrt{1-c^2} w_2 \right) \frac{bc}{\sqrt{1-c^2}} w_2 \right].\label{derivative-c-second-term}
\end{align}

Next we compute \cref{derivative-c-first-term,derivative-c-second-term} by using Stein's lemma, 
\begin{lemma}[Stein's Lemma] For a differentiable function $\phi:\Rbb\rightarrow \Rbb$ with $\Ebb_{x\sim\mathcal N(0,1)} \left[|\phi'(x)|\right] <\infty$, 
	\begin{align*}
		\Ebb_{x\sim\mathcal N(0,1)} \left[\phi(x) x\right] = \Ebb_{x\sim\mathcal N(0,1)} \left[\phi'(x)\right] \,.   
	\end{align*}
\end{lemma}
Applying Stein's Lemma to \cref{derivative-c-first-term} gives, 
\begin{align}
& \mathop{\Ebb}_{w}  \left[ \sigma(a w_1) \cdot \sigma' \left(bc w_1 + b\sqrt{1-c^2} w_2 \right) \cdot b w_1 \right] \nonumber \\
 &= ab \mathop{\Ebb}_{w}  \left[ \sigma'(a w_1) \cdot \sigma' \left(bc w_1 + b\sqrt{1-c^2} w_2 \right) \right] 
 + b^2 c \mathop{\Ebb}_{w}  \left[ \sigma(a w_1)\cdot \sigma^{\prime\prime} \left(bc w_1 + b\sqrt{1-c^2} w_2 \right) \right].\label{eq:stein-first-term}
\end{align}
Applying Stein's Lemma to \cref{derivative-c-second-term} gives,
\begin{align}
&-\mathop{\Ebb}_{w}  \left[ \sigma(a w_1) \cdot \sigma' \left(bc w_1 + b\sqrt{1-c^2} w_2 \right) \frac{bc}{\sqrt{1-c^2}} w_2 \right] \nonumber \\
&= -b^2c\mathop{\Ebb}_{w}  \left[\sigma(a w_1) \cdot \sigma^{\prime\prime} \left(bc w_1 + b\sqrt{1-c^2} w_2 \right)  \right].\label{eq:stein-second-term}
\end{align}
Here we show that the term $b^2c\mathop{\Ebb}_{w}  \left[\sigma(a w_1) \cdot \sigma^{\prime\prime} \left(bc w_1 + b\sqrt{1-c^2} w_2 \right)  \right]$ in \cref{eq:stein-first-term,eq:stein-second-term} has a bounded value as follows,
\begin{align*}
    \abs{b^2c\mathop{\Ebb}_{w}  \left[\sigma(a w_1) \cdot \sigma^{\prime\prime} \left(bc w_1 + b\sqrt{1-c^2} w_2 \right)  \right]} &\le |b^2c| \cdot \sqrt{\mathop{\Ebb}_{w_1}  \left[\abs{\sigma(aw_1)}^2  \right]  \mathop{\Ebb}_{w}  \left[\abs{ \sigma^{\prime\prime} \left(bcw_1 + b\sqrt{1-c^2}w_2 \right)}^2  \right]} \\
    &= |b^2c| \cdot \sqrt{\mathop{\Ebb}_{w_1}  \left[\abs{\sigma(aw_1)}^2  \right] \cdot \mathop{\Ebb}_{\gamma\sim \mathcal{N}(0,1)}  \left[\abs{ \sigma^{\prime\prime} \left(b\gamma \right)}^2  \right]},
\end{align*}
where the first line above follows from Cauchy–Schwarz inequality. The second line above follows from the fact that $w_1$ and $w_2$ are independent copies of the normal random variable $\mathcal{N}(0,1)$, thus the random variable $bcw_1 + b\sqrt{1-c^2}w_2$ is indeed $b \gamma$ for a normal $\gamma \sim \mathcal{N}(0,1)$. 
Therefore, in order for the expectation $b^2c\mathop{\Ebb}_{w}  \left[\sigma(a w_1) \cdot \sigma^{\prime\prime} \left(bc w_1 + b\sqrt{1-c^2} w_2 \right)  \right]$ to make sense, it is enough to have $b^2 \mathop{\Ebb}_{w_1}  \left[\abs{\sigma(aw_1)}^2  \right] < \infty$ and $b^2 \mathop{\Ebb}_{\gamma} \left[\abs{ \sigma^{\prime\prime} \left(b\gamma \right)}^2  \right] < \infty$. Note that the dual activation is symmetric with respect to swapping $a$ and $b$, in the sense that $k_\sigma(a,b,c)=k_\sigma(b,a,c)$. Thus, we can without loss of generality assume that $b \le a$. With this assumption $b^2 \mathop{\Ebb}_{w_1}  \left[\abs{\sigma(aw_1)}^2  \right] \le a^2 \mathop{\Ebb}_{w_1}  \left[\abs{\sigma(aw_1)}^2  \right]$.
Now, by recalling \cref{conseq-precondition-norm-bound}, we have $a^2 \mathop{\Ebb}_{w_1}  \left[\abs{\sigma(aw_1)}^2  \right] \le a \nu \cdot \norm{\sigma}_{\mathcal{N}(0,\nu^2)}^2 < \infty$ and $b^2 \mathop{\Ebb}_{\gamma} \left[\abs{ \sigma^{\prime\prime} \left(b\gamma \right)}^2  \right] \le b \nu \cdot \norm{\sigma^{\prime\prime}}_{\mathcal{N}(0,\nu^2)}^2 < \infty$. 

Summing \cref{eq:stein-first-term,eq:stein-second-term} and dividing the sum by $ab$ give that
\begin{align} \label{eq-derivative-dual-kernel-2}
    \frac{1}{ab} ~ \frac{\partial}{\partial c} k_{\sigma}(a,b,c) = \mathop{\Ebb}_{w\sim\mathcal{N}(0,\I_2)}\left[ \sigma^\prime(a w_1) \cdot \sigma^\prime \left(bc w_1 + b\sqrt{1-c^2} w_2 \right)\right].
\end{align}
Finally, plugging in the values $a = \normx,b=\normy$ and $c=\frac{\inner{x,y}}{\normx\normy}$ such that $a,b \le \nu$ and $\abs{c} < 1$ and using \cref{eq-dual-kernel-dual-activation} result in \cref{eq-derivative-dual-kernel}.

Now suppose that the map $c \rightarrow \frac{\partial }{\partial c} k_\sigma(\cdot, \cdot, c)$ is continuous at $c=\pm1$.
Since we consider that $\sigma^{\prime\prime}(\cdot)$ exists, $\sigma^\prime(\cdot)$ is continuous almost everywhere.
Using these properties, we claim that the right-hand side in \cref{eq-derivative-dual-kernel-2} is continuous in $c$ because for every $c' \in [-1,1]$ it holds that
\begin{align}
    &\lim_{c\rightarrow c^\prime} ~ \mathop{\Ebb}_{w\sim\mathcal{N}(0,\I_2)}\left[ \sigma^\prime(a w_1) \cdot \sigma^\prime \left(bc w_1 + b\sqrt{1-c^2} w_2 \right)\right] \nonumber \\
    &= \mathop{\Ebb}_{w\sim\mathcal{N}(0,\I_2)}\left[ \sigma^\prime(a w_1) \cdot \lim_{c\rightarrow c^\prime} \sigma^\prime \left(bc w_1 + b\sqrt{1-c^2} w_2 \right)\right].
\end{align}
The above equality holds from the dominated convergence theorem (see Corollary 6.26 in \cite{klenke2013probability}) with the dominated function obtained as
\begin{align}
   \abs{\sigma^\prime(a w_1) \cdot \sigma^{\prime}\left(bc w_1 +  b\sqrt{1-c^2} w_2\right)} 
    &\le C_1C_2 \exp \left( \frac{a^2 w_1^2 + \left( b c w_1 + b \sqrt{1-c^2} w_2\right)^2}{4.1\nu^2} \right) \nonumber \\
    &\le C_1 C_2 \exp \left( \frac{(a^2 + b^2) w_1^2 + b^2 w_2^2}{4.1\nu^2 }\right) \nonumber \\
    &\le C_1 C_2 \exp\left( \frac{w_1^2 + w_2^2}{2.05}\right) := h'(w)
\end{align}
where the first inequality follows from the preconditions of \cref{lmm-derivative-dual-kernel} and the second one follows from $(cw_1 + \sqrt{1-c^2} w_2)^2 \leq w_1^2 + w_2^2$, and the third one follows from $a,b \leq \nu$. And it is easy to check that $\Ebb_{w \sim \mathcal{N}(0,{\bm I}_2)}[|h'(w)|] < \infty$.
Hence both sides of \cref{eq-derivative-dual-kernel-2} are continuous at $c=\pm1$ and 
taking $\lim_{c\rightarrow \pm 1}$ in both sides of \cref{eq-derivative-dual-kernel-2} gives that \cref{eq-derivative-dual-kernel} holds for $x,y$ such that $\abs{\inner{x,y}}=\normx\normy$. This concludes the proof of \cref{lmm-derivative-dual-kernel}. 
\end{proofof}

\paragraph{Examples.} 
For $\sigma(t) = \sin(t)$, the corresponding dual kernel is known to be
\begin{align}
    k_{\sin}(a,b,c) = e^{-\frac{a^2 + b^2}{2}}~\frac{e^{abc} - e^{-abc}}{2}.
\end{align}
Applying \cref{lmm-derivative-dual-kernel} to $k_{\sin}$
\begin{align}
    \frac{1}{a \cdot b}\frac{\partial  k_{\sin}}{\partial c} = e^{-\frac{a^2 + b^2}{2}}~\frac{e^{abc} + e^{-abc}}{2}
\end{align}
which is equivalent to $k_{\cos}(a,b,c)$ (see \cref{table-dual-kernel-details} for detailed derivations).

For $\sigma(t) = \erf(t)$, the corresponding dual kernel is known as
\begin{align}
    k_{\erf}(a, b, c) = \frac{2}{\pi} \sin^{-1}\left(\frac{2abc}{\sqrt{(1+2 a^2)(1+2 b^2)}}\right).
\end{align}
Again, applying \cref{lmm-derivative-dual-kernel} to $k_{\erf}$ provides that
\begin{align}
    \frac{1}{a \cdot b}\frac{\partial  k_{\erf}}{\partial c}
    &= \frac{2}{\pi} \frac{1}{a \cdot b} \frac{1}{\sqrt{1 - \left(  \frac{2abc}{\sqrt{(1+2a^2)(1+2b^2)}} \right)^2}} \cdot  \frac{2ab}{\sqrt{(1+2a^2)(1+2b^2)}} \nonumber \\
    &= \frac{4}{\pi} \frac{1}{\sqrt{(1+2a^2)(1+2b^2) - 4 a^2 b^2 c^2}}.
\end{align}
One can check that this matches the dual kernel of $(\erf(t))' = \frac{2}{\sqrt{\pi}} e^{-t^2}$ from \cref{table-dual-kernel-details}.

In addition, \cref{lmm-derivative-dual-kernel} holds for the ReLU activation because 
\begin{align}
    \frac{1}{a \cdot b}\frac{\partial  k_{\mathrm{ReLU}}}{\partial \theta}&=\frac{1}{a\cdot b} \frac{\partial}{\partial c} \left( ab \frac{\sqrt{1-c^2} + (\pi - \cos^{-1}(c)) c}{2 \pi} \right) = \frac{\pi - \cos^{-1}(c)}{2 \pi}
\end{align}
which is equivalent to the dual kernel of $\mathrm{ReLU}'(t)=\mathrm{Step}(t)$.

This theorem is used in \texttt{Elementwise} in our codebase to automatically derive the NTK given only the NNGP function.

\section{Proof of \cref{thm-subspace-embding}}\label{appndx-proof-subspace-embed}
\homogeneousntkembedding*
\begin{proofof}{\cref{thm-subspace-embding}}
We start the proof by showing that the polynomial $R^{(L)}(t)$ defined as
\[
R^{(L)}(t) := \sum_{h=0}^L \tkappa^{ \circ h }(t) \cdot \prod_{i=h}^{L-1} \tkappa' \circ \tkappa^{ \circ i }(t)
\]
tightly approximates the following function at every point $t \in [-1,1]$
\[
T^{(L)}(t) := \sum_{h=0}^L \kappa^{ \circ h }(t) \cdot \prod_{i=h}^{L-1} \kappa' \circ \kappa^{ \circ i}(t)
\]
Specifically, we prove that
\begin{equation}\label{eq-poly-approx-quality-homogen-ntk}
    \max_{t \in [-1,1]} \abs{T^{(L)}(t) - R^{(L)}(t)} \le \frac{1}{\poly{n}}. 
\end{equation}
In order to prove \cref{eq-poly-approx-quality-homogen-ntk}, we first show that for every $h = 0, 1, 2, \ldots L$ the following holds
\[ \max_{t \in [-1,1]} \abs{ \kappa^{ \circ h }(t) - \tkappa^{ \circ h }(t)} \le \frac{1}{\poly{n}}. \]
The proof of the above is by induction on $h$. For $h=0$ by convention $\kappa^{ \circ h }(t) = \tkappa^{ \circ h }(t) = t$, which proves the base of induction. 
For the inductive step suppose that $\max_{t \in [-1,1]} \abs{ \kappa^{ \circ h-1 }(t) - \tkappa^{ \circ h-1}(t)} \le \frac{1}{\poly{n}}$ holds for some $h \ge 1$. Using this inductive hypothesis along with preconditions of \cref{thm-subspace-embding}, for any $t \in [-1,1]$ we can write,
\begin{align*}
    \abs{ \tkappa^{ \circ h}(t) - \kappa^{ \circ h }(t)}  &\le \abs{ \tkappa^{ \circ h }(t) - \tkappa \circ \kappa^{ \circ h-1 }(t)} +  \abs{ \tkappa \circ \kappa^{ \circ h-1 }(t) - \kappa^{ \circ h }(t)}\\
    &\le \frac{1}{\poly{n}} +  \abs{ \tkappa \circ \kappa^{ \circ h-1 }(t) - \kappa^{ \circ h }(t)}\\
    &\le \frac{1}{\poly{n}},
\end{align*}
where the first line above follows from triangle inequality. The second line above follows from precondition ${\bf (2)}$ of the theorem. The third line follows from precondition ${\bf (1)}$ of the theorem.
Therefore $\max_{t \in [-1,1]} \abs{ \kappa^{ \circ h }(t) - \tkappa^{ \circ h }(t)} \le \frac{1}{\poly{n}}$ for any $h=0,1, \ldots L$.

Moreover, by preconditions of the theorem, we can show in a similar fashion that $$\max_{t \in [-1,1]} \abs{ \kappa' \circ \kappa^{ \circ h-1 }(t) - \tkappa' \circ \tkappa^{ \circ h-1 }(t)} \le \frac{1}{\poly{n}}.$$ These inequalities are sufficient to prove \cref{eq-poly-approx-quality-homogen-ntk}.

Now, let us define the kernel $\widetilde{\Theta}_\sigma^{(L)}$ as 
\[
\widetilde{\Theta}_\sigma^{(L)}(x,y) := \normx \normy \cdot R^{(L)}\left( \frac{\inner{x,y}}{\normx \normy} \right).
\]
The depth-$L$ NTK kernel, as we showed in \cref{ntk-kernl-simplified-homogen}, is 
\[
\Theta_\sigma^{(L)}(x,y) := \normx \normy \cdot T^{(L)}\left( \frac{\inner{x,y}}{\normx \normy} \right).
\]
Using \cref{eq-poly-approx-quality-homogen-ntk}, for any $x,y \in \Rbb^d$, we have,
\[
\abs{ \Theta_\sigma^{(L)}(x,y) - \widetilde{\Theta}_\sigma^{(L)}(x,y) } \le \frac{\normx \normy}{\poly{n}}.
\]
For any dataset $\X=[x_1, x_2, \dots, x_n] \in\Rbb^{d \times n}$, we let $\widetilde{\K}_{\mathrm{ntk}} \in \Rbb^{n \times n}$ be the kernel matrix corresponding to the kernel function $\widetilde{\Theta}_\sigma^{(L)}$ and $\X$, i.e., $[\widetilde{\K}_{\mathrm{ntk}}]_{i,j} = \widetilde{\Theta}_\sigma^{(L)}(x_i,x_j)$ to have that
\begin{align*}
    \norm{ \K_{\mathrm{ntk}} - \widetilde{\K}_{\mathrm{ntk}} }_{op} &\le \norm{ \K_{\mathrm{ntk}} - \widetilde{\K}_{\mathrm{ntk}} }_F\\
    &\le \frac{\norm{\X}_F^2}{\poly{n}}\\
    &\le \frac{1}{\poly{n}} \le \frac{\varepsilon \lambda}{3},
\end{align*}
where the third line above follows from the assumption of the theorem about $\norm{\X}_F \le \poly{n}$ and $\varepsilon, \lambda \ge \frac{1}{\poly{n}}$. Therefore, in order to prove the desired subspace embedding guarantee of \cref{thm-subspace-embding}, it suffices to prove that with probability at least $1 - \frac{1}{\poly{n}}$, the following holds
\[
(1-\varepsilon/2) \left(\widetilde{\K}_{\mathrm{ntk}} + \lambda \I_n \right) \preceq {\psi}^{(L)}(\X)^\top {\psi}^{(L)}(\X)  + \lambda \I_n \preceq (1+\varepsilon/2)\left( \widetilde{\K}_{\mathrm{ntk}} + \lambda \I_n \right).
\]
From now on we focus on proving the above inequality. If we let $R^{(L)}(t) = \sum_{i=0}^p c_j t^j$ be the polynomial defined in line~\ref{line-polynomial-R} of \cref{alg-subspace-embding} then we have that 
\begin{align*}
    \widetilde{\K}_{\mathrm{ntk}} =\D \left( \sum_{j=0}^p c_j \left(\Y^{\otimes j} \right)^\top \Y^{\otimes j} \right) \D = \sum_{j=0}^p c_j \cdot \left(\Y^{\otimes j}\D\right)^\top \Y^{\otimes j} \D
\end{align*}
where
\begin{align*}
 \D = \mathrm{diag}\left(\left[ \norm{x_1}_2, \dots, \norm{x_n}_2 \right]\right)\in \Rbb^{n \times n} ~&,~~~~~ \Y = \left[ \frac{x_1}{\norm{x_1}_2}, ~ \ldots, \frac{x_n}{\norm{x_n}_2} \right] \in \Rbb^{d \times n}.
\end{align*}
Note that each of the term $\left(\Y^{\otimes j}\D\right)^\top \Y^{\otimes j} \D = \D (\Y^{\otimes j})^{\top} \Y^{\otimes j} \D$ is a positive definite Gram matrix.
Also, from the fact that coefficients $c_j$ are positive and by Courant-Fischer's min-max theorem, the statistical dimension of the Gram matrix $c_j \cdot \left(\Y^{\otimes j}\D\right)^\top \Y^{\otimes j} \D$ for every $j \ge 0$ is upper bounded by the statistical dimension of the kernel matrix $\widetilde{\K}_{\mathrm{ntk}}$. More specifically, for any $\mu > 0$ and every $j = 0,1, \ldots p$, we have
\[
s_\mu\left( c_j \cdot \left(\Y^{\otimes j}\D\right)^\top \Y^{\otimes j} \D\right) \le s_\mu\left( \widetilde{\K}_{\mathrm{ntk}} \right).
\]
Now let $\mu := \frac{\lambda}{p+1}$ and note that from the definition of statistical dimension it follows that $s_\mu\left( \widetilde{\K}_{\mathrm{ntk}} \right) \le (p+1) s_\lambda\left( \widetilde{\K}_{\mathrm{ntk}} \right)$. The sketch matrix $Q^j$ defined in line~\ref{line-Q-instantiation} of the algorithm has $m = \Omega\left( \varepsilon^{-2}{s_\lambda(\K_{\mathrm{ntk}})} \cdot \poly{q^L, \log n} \right) = \Omega\left( \varepsilon^{-2}{s_\mu(\K_{\mathrm{ntk}})} \cdot \poly{q^L, \log n} \right)$ rows. Therefore, by \cref{lem-polysketch}, the following holds for every $j=0,1, \ldots p$ with probability at least $1 - \frac{1}{\poly{n}}$,
\[
\frac{c_j \cdot \D (\Y^{\otimes j})^\top \Y^{\otimes j} \D + \mu \I_n }{1+\varepsilon/3} \preceq c_j \cdot \D \left(Q^j\Y^{\otimes j}\right)^\top Q^j\Y^{\otimes j} \D + \mu \I_n \preceq \frac{c_j \cdot\D (\Y^{\otimes j})^\top \Y^{\otimes j} \D + \mu \I_n }{1 - \varepsilon/3}.
\]
By union bound over $p+1 = \bigo(q^L) = o(\poly{n})$ events, the above inequality holds simultaneously for all $j$ with high probability in $n$. Thus, by summing up the above inequality over all $j$ and using the fact that $\mu = \frac{\lambda}{p+1}$ we find that,
\[
\frac{\widetilde{\K}_{\mathrm{ntk}} + \lambda \I_n }{1+\varepsilon/3} \preceq \sum_{j=0}^p \left( c_j \cdot \D^j \left(Q^j\Y^{\otimes j}\right)^\top Q^j\Y^{\otimes j} \D\right) + \lambda \I_n \preceq \frac{\widetilde{\K}_{\mathrm{ntk}} + \lambda \I_n }{1 - \varepsilon/3}.
\]
This proves the theorem because the output of the algorithm satisfies that $${\psi}^{(L)}(\X)^\top {\psi}^{(L)}(\X) = \sum_{j=0}^p c_j \D^j \left(Q^j\Y^{\otimes j}\right)^\top Q^j\Y^{\otimes j} \D^j.$$
The runtime bound follows immediately from \cref{lem-polysketch}.
\end{proofof}

\section{Convolutional Neural Tangent Kernel}
In this section, we design and analyze an efficient oblivious sketch for the Convolutional Neural Tangent Kernel (CNTK), which is the kernel function corresponding to a CNN with infinite number of channels.
\citet{arora2019exact} gave dynamic programming (DP) based solutions for computing two variants of CNTK; one is the vanilla version which performs no pooling, and the other performs Global Average Pooling (GAP) on its top layer. For conciseness, we focus mainly on the CNTK with GAP, which also exhibits superior empirical performance~\cite{arora2019exact}. However, we remark that the vanilla CNTK has a very similar structure and hence our techniques can be applied to it, as well. 

We start by restating the DP approach proposed in \cite{arora2019exact} for computing the $L$-layered CNTK with an arbitrary activation function $\sigma$, convolutional filters of size $q \times q$ and GAP. Consider two input images $y,z \in \Rbb^{d_1\times d_2 \times c}$ where $c$ is the number of channels ($c=3$ for the standard color image).
\begin{enumerate}[wide, labelwidth=!, labelindent=5pt, leftmargin=.5cm]
    \item For every $i,i' \in [d_1]$ and $j,j' \in [d_2]$, define 
	\begin{align}
		&\ijij{\Gamma^{(0)}}  := \sum_{l=1}^c y_{i,j,l} \cdot z_{i',j',l}, \label{eq:dp-cntk-zero}\\ 
		&\ijij{K^{(0)}} := \sumijij{\Gamma^{(0)}}.\nonumber
	\end{align}

	\item For every $h \in[L]$, every $i,i' \in [d_1]$ and $j,j' \in [d_2]$, define 
	\begin{equation}\label{eq:dp-cntk-covar}
		\begin{split}
			&\ijij{\Gamma^{(h)}} := \frac{1}{q^2} \cdot \Ebb_{(u,v) \sim \mathcal{N}\left( 0, {\bm\Lambda}^{(h)}_{i,j,i',j'}(x,y) \right)} \left[ \sigma(u) \sigma(v) \right],\\
			&\ijij{K^{(h)}} := \sumijij{\Gamma^{(h)}},
		\end{split}
	\end{equation}
	where the covariance matrix is
	\begin{equation}
	    {\bm\Lambda}^{(h)}_{i,j,i',j'}(x,y)
	    := 
	    \begin{bmatrix}
			K_{i,j,i',j'}^{(h-1)}(y,y) & K_{i,j,i',j'}^{(h-1)}(y,z)\\
			&\\
			K_{i,j,i',j'}^{(h-1)}(z,y) & K_{i,j,i',j'}^{(h-1)}(z,z)
		\end{bmatrix} \in \Rbb^{2 \times 2}.
	\end{equation}
	\item For every $h \in [L] $, every $i,i' \in [d_1]$ and $j,j' \in [d_2]$, define 
	\begin{equation}\label{eq:dp-cntk-derivative-covar}
		\ijij{\dt{\Gamma}^{(h)}} := \frac{1}{q^2 } \cdot \Ebb_{(u,v) \sim \mathcal{N}\left( 0, {\bm\Lambda}^{(h)}_{i,j,i',j'}(y,z) \right)} \left[ \sigma'(u) \sigma'(v) \right].
	\end{equation}
	\item Let $\Pi^{(0)}(x,y) := 0$ and for every $h\in[L-1]$, every $i,i' \in [d_1]$ and $j,j' \in [d_2]$, define 
	\begin{equation}\label{eq:dp-cntk-pi}
		\ijij{\Pi^{(h)}} := \sum_{a=-\frac{q-1}{2}}^{\frac{q-1}{2}} \sum_{b=-\frac{q-1}{2}}^{\frac{q-1}{2}} \left[\Pi^{(h-1)}(y,z) \odot \dt{\Gamma}^{(h)}(y,z) + \Gamma^{(h)}(y,z)\right]_{i+a,j+b,i'+a,j'+b},
	\end{equation}
	and also $\Pi^{(L)}(y,z) := \Pi^{(L-1)}(y,z) \odot \dt{\Gamma}^{(L)}(y,z).$
	\item The final CNTK expressions is defined as:
	\begin{equation}\label{eq:dp-cntk-gap-def}
		\Theta_{\mathrm{cntk}}^{(L)}(y,z) := \frac{1}{d_1^2d_2^2} \cdot \sum_{i , i' \in [d_1]} \sum_{j , j' \in [d_2]} \Pi_{i,j,i',j'}^{(L)}(y,z).
	\end{equation}
\end{enumerate}

The above procedure for exact computation of the depth-$L$ CNTK value $\Theta_{\mathrm{cntk}}^{(L)}(y,z)$ takes $\Omega\left( (d_1d_2)^2 (c + L) \right)$ runtime, which is extremely slow particularly due to its quadratic dependence on the number of pixels of input images $d_1d_2$.
Fortunately, we are able to show that the CNTK for homogeneous dual kernels, as per \cref{def-homogeneous-dual-kernel}, is a highly structured object that can be fully characterized in terms of tensoring and composition of the dot-product factor of dual kernels, and exploiting this special structure is key in designing efficient sketching methods for the CNTK.

\subsection{CNTK for Homogeneous Dual Kernels}

In this section we show that the CNTK function corresponding to any homogeneous dual kernel, i.e., $K_\sigma(x,y) = \normx \normy \cdot \kappa\left( \frac{\inner{x,y}}{\normx \normy} \right)$ for some $\kappa:[-1,1]\rightarrow[-1,1]$, takes a simple form which enables us to devise efficient sketching algorithms for the CNTK.
Unlike the fully-connected NTK, the CNTK is not a simple dot-product kernel function. The key reason being that CNTK works by partitioning its input images into patches and locally transforming the patches at each layer, as opposed to the NTK which operates on the entire input vectors.
The depth-$L$ CNTK corresponding to homogeneous dual kernels can be fully characterized in terms of tensoring and composition of the dot-product kernel $\kappa$ and its derivative $\kappa'$. 

\begin{defn}[CNTK for Homogeneous Dual Kernels] \label{relu-cntk-def}
	For every positive integers $q,L$, the $L$-layered CNTK for a homogeneous dual kernel, as per \cref{def-homogeneous-dual-kernel}, and convolutional filter size of $q \times q$ is defined as follows
	\begin{enumerate}[wide, labelwidth=!, labelindent=5pt, leftmargin=.5cm]
		\item For $x\in \Rbb^{d_1\times d_2 \times c}$, every $i \in [d_1]$ and $j \in [d_2]$ let $N_{i,j}^{(0)} (x) :=q^2 \cdot \sum_{l=1}^c \left| x_{i,j,l} \right|^2$, and for every $h \ge 1$, recursively define,
		\begin{equation}\label{eq:dp-cntk-norm-simplified}
			N^{(h)}_{i,j}(x):=\frac{1}{q^2} \cdot \sum_{a=-\frac{q-1}{2}}^{\frac{q-1}{2}} \sum_{b=-\frac{q-1}{2}}^{\frac{q-1}{2}} N^{(h-1)}_{i+a,j+b}(x).
		\end{equation}
		\item For every $h\in[h]$, every $i,i' \in [d_1]$ and $j,j' \in [d_2]$, define 
		\begin{equation}\label{eq:dp-cntk-covar-simplified}
			\Gamma^{(h)}_{i,j,i',j'}(y,z) := \frac{\sqrt{N^{(h)}_{i,j}(y) \cdot N^{(h)}_{i',j'}(z)}}{q^2} \cdot \kappa\left( A \right),~~~\Gamma_{i,j,i',j'}^{(0)}(y,z)  = \sum_{l=1}^c y_{i,j,l}  \cdot z_{i',j',l}
		\end{equation}
		where $A = \frac{1}{\sqrt{N^{(h)}_{i,j}(y) \cdot N^{(h)}_{i',j'}(z)}} \sum_{a=-\frac{q-1}{2}}^{\frac{q-1}{2}} \sum_{b=-\frac{q-1}{2}}^{\frac{q-1}{2}}  \Gamma^{(h-1)}_{i+a,j+b,i'+a,j'+b}(y,z)$.
		
		\item 
		For every $h\in[L]$, every $i,i' \in [d_1]$ and $j,j' \in [d_2]$, define 
		\begin{equation}\label{eq:dp-cntk-derivative-covar-simplified}
			\dt{\Gamma}^{(h)}_{i,j,i',j'}(y,z) := \frac{1}{q^2} \cdot \kappa'\left(A\right).
		\end{equation}
		
		\item Let $\Pi^{(0)}(y,z) := 0$ and for every $h\in[L-1]$, every $i,i' \in [d_1]$ and $j,j' \in [d_2]$, define 
		\begin{equation}\label{eq:dp-cntk}
			\Pi^{(h)}_{i,j,i',j'}(y,z) := \sum_{a=-\frac{q-1}{2}}^{\frac{q-1}{2}} \sum_{b=-\frac{q-1}{2}}^{\frac{q-1}{2}} \left[\Pi^{(h-1)}(y,z) \odot \dt{\Gamma}^{(h)}(y,z) + \Gamma^{(h)}(y,z)\right]_{i+a,j+b,i'+a,j'+b},
		\end{equation}
		Furthermore, define
		\begin{equation}\label{eq:dp-cntk-last-layer}
			\Pi^{(L)}(y,z) := \Pi^{(L-1)}(y,z) \odot \dt{\Gamma}^{(L)}(y,z).
		\end{equation}
		\item The final CNTK expressions for ReLU activation is:
		\begin{equation}\label{eq:dp-cntk-finalkernel}
			\Theta_{\mathrm{cntk}}^{(L)}(y,z) := \frac{1}{d_1^2d_2^2} \cdot \sum_{i , i' \in [d_1]} \sum_{j , j' \in [d_2]} \Pi_{i,j,i',j'}^{(L)}(y,z).
		\end{equation}
	\end{enumerate}
\end{defn}

We now describes some of the basic properties of the functions $\Gamma^{(h)}(y,z), \dt{\Gamma}^{(h)}(y,z)$, and $\Pi^{(h)}(y,z)$ defined in \cref{eq:dp-cntk-covar-simplified}, in the following lemma,

\begin{lemma}[Properties of $\Gamma^{(h)}(y,z), \dot{\Gamma}^{(h)}(y,z)$, and $\Pi^{(h)}(y,z)$]\label{properties-gamma}
	Suppose that the dot-product kernel $\kappa(\cdot)$ in \cref{def-homogeneous-dual-kernel} and its derivative satisfy $\kappa(1) = \kappa'(1)=1$. 
	For every images $y,z \in \Rbb^{d_1 \times d_2 \times c}$, every integer $h \ge 0$ and every  $i,i' \in [d_1]$ and $j,j' \in [d_2]$ the following properties are satisfied by functions $\Gamma^{(h)}, \dt{\Gamma}^{(h)},  \Pi^{(h)}$ and $N^{(h)}$ defined in \cref{eq:dp-cntk-covar-simplified}, \cref{eq:dp-cntk-derivative-covar-simplified}, \cref{eq:dp-cntk} and \cref{eq:dp-cntk-last-layer}, and \cref{eq:dp-cntk-norm-simplified} of \cref{relu-cntk-def}:
	\begin{enumerate}[wide, labelwidth=!, labelindent=5pt, leftmargin=.5cm]
		\item {\bf Cauchy–Schwarz:} $\left| \Gamma_{i,j,i',j'}^{(h)}(y,z) \right| \le \frac{\sqrt{N_{i,j}^{(h)}(y) \cdot N_{i',j'}^{(h)}(z)}}{q^2}$, and $\left| \dot{\Gamma}_{i,j,i',j'}^{(h)}(y,z) \right| \le \frac{1}{q^2}$, and $\Pi_{i,j,i',j'}^{(h)}(y,z) \le \sqrt{\Pi_{i,j,i,j}^{(h)}(y,y) \cdot \Pi_{i',j',i',j'}^{(h)}(z,z)}$.
		
		\item {\bf Norm value:} $\Gamma_{i,j,i,j}^{(h)}(y,y) = \frac{N_{i,j}^{(h)}(y)}{q^2} \ge 0$, and $\dot{\Gamma}_{i,j,i,j}^{(h)}(y,y) = \frac{1}{q^2} \ge 0$, and $\Pi_{i,j,i,j}^{(h)}(y,y) = \begin{cases}
			h \cdot N_{i,j}^{(h+1)}(y) & \text{if } h < L\\
			\frac{L-1}{q^2} \cdot N_{i,j}^{(L)}(y) & \text{if } h = L
		\end{cases}$.
		
	\end{enumerate}
\end{lemma}
The properties stated in the above lemma can be straightforwardly proved using induction.

\subsection{CNTK Sketch for Homogeneous Dual Kernels} \label{sec-cntk-sketch}
Our sketching method relies on approximating the dot-product kernel function $\kappa(\cdot)$ and its derivative $\kappa'(\cdot)$ with low-degree polynomials via Taylor expansion, and then applying \textsc{PolySketch} to the resulting polynomial kernels. 
Our sketch computes the features for each pixel of the input image, by tensor product of the sketches for function $\kappa(\cdot)$ at consecutive layers, which in turn can be sketched efficiently by \textsc{PolySketch}. Additionally, the features of pixels that lie in the same patch get \emph{locally combined} at each layer via direct sum operation. This precisely corresponds to the convolution operation in neural networks. We start by presenting our CNTK Sketch algorithm in \cref{alg-cntk-sketch} and prove the correctness and runtime of our procedure in \cref{thm-cntk-sketch}.

\begin{algorithm}[t]
	\caption{CNTK Sketch for Homogeneous Dual Kernels} \label{alg-cntk-sketch}
	\begin{algorithmic}[1]
		\STATE {\bf input}: image $x \in \Rbb^{d_1 \times d_2 \times c}$, depth $L$, filter size $q$, sketching dimensions $m, m'$, polynomials $\tkappa(t)= \sum_{j=0}^p a_j t^j$ and $\tkappa'(t)= \sum_{j=0}^p b_j t^j$ with $a_j, b_j \in \Rbb_+$
		\STATE for every $i \in [d_1]$, $j \in [d_2]$, and $h = 0, 1, 2, \ldots L$ compute $N_{i,j}^{(h)}(x)$ as per \cref{eq:dp-cntk-norm-simplified}
		
		\STATE for every $i \in [d_1]$, $j \in [d_2]$, initialize $\phi_{i,j}^{(0)}(x) \gets x_{i,j,:}$ and $\psi^{(0)}_{i,j}(x) \gets 0$ \label{cntk-sketch-covar-zero}
		
		\FOR{ $h=1$ to $L$}
		
			\STATE{For $\ell=0,\dots, p$, let $Q^{\ell}$ be a degree-$\ell$ \textsc{PolySketch} with target dimension $m$ and for every $i \in [d_1]$, $j \in [d_2]$ compute} \label{eq:def-Z-vectors-mu}\\
				\begin{align*}
				 Z^{(h)}_{i,j,\ell}(x) \gets Q^{\ell} \cdot\left( \mu^{(h)}_{i,j}(x) \right)^{\otimes \ell}   , ~~~  \mu^{(h)}_{i,j}(x) \gets \frac{1}{\sqrt{N^{(h)}_{i,j}(x)}} \cdot  \bigoplus_{a=-\frac{q-1}{2}}^{\frac{q-1}{2}} \bigoplus_{b=-\frac{q-1}{2}}^{\frac{q-1}{2}}  \phi_{i+a,j+b}^{(h-1)}(x)
				\end{align*}
			\STATE for every $i \in [d_1]$, $j \in [d_2]$ construct $\phi_{i,j}^{(h)}(x) \gets \frac{\sqrt{N^{(h)}_{i,j}(x)}}{q} \cdot \bigoplus_{\ell=0}^{p} \sqrt{a_{\ell}} \cdot Z^{(h)}_{i,j,\ell}(x)$ \label{eq:maping-cntk-covar}

		\STATE for every $i \in [d_1]$, $j \in [d_2]$ construct $\dt{\phi}^{(h)}(x) \gets \frac{1}{q} \cdot \bigoplus_{\ell=0}^{p} \sqrt{b_{\ell}} \cdot Z^{(h)}_{i,j,\ell}(x)$ \label{eq:cntk-map-phidot}
		
		\STATE Let $Q^2$ be a degree-2  \textsc{PolySketch} with target dimension $m'$
		\IF{$h=L$}
		\STATE for every $i \in [d_1]$, $j \in [d_2]$ compute \label{psi-cntk}
		$$
		\psi^{(h)}_{i,j}(x) \gets  \bigoplus_{a=-\frac{q-1}{2}}^{\frac{q-1}{2}} \bigoplus_{b=-\frac{q-1}{2}}^{\frac{q-1}{2}} \left[ Q^2 \left(\psi_{i+a,j+b}^{(h-1)}(x) \otimes \dt{\phi}_{i+a,j+b}^{(h)}(x)\right) \oplus \phi_{i+a,j+b}^{(h)}(x) \right]
		$$
		\ELSE
		\STATE for every $i \in [d_1]$ and $j \in [d_2]$ compute \label{psi-cntk-last}
		$$
		\psi^{(L)}_{i,j}(x) \gets  \bigoplus_{a=-\frac{q-1}{2}}^{\frac{q-1}{2}} \bigoplus_{b=-\frac{q-1}{2}}^{\frac{q-1}{2}} Q^2 \left(\psi_{i+a,j+b}^{(L-1)}(x) \otimes \dt{\phi}_{i+a,j+b}^{(L)}(x)\right)
		$$

		\ENDIF
		\ENDFOR
		
		\STATE {\bf return} $\Psi_{\mathrm{cntk}}^{(L)}(y,z) := \frac{1}{d_1d_2} \cdot \sum_{i \in [d_1]} \sum_{j \in [d_2]} \psi^{(L)}_{i,j}(x) $
	\end{algorithmic}
\end{algorithm}

\begin{theorem}[Correctness and Runtime of \cref{alg-cntk-sketch}]\label{thm-cntk-sketch}
Suppose that the dual kernel $K_\sigma$ is homogeneous as per \cref{def-homogeneous-dual-kernel} also assume that $\kappa(1) = \kappa'(1)=1$. Fix some $\varepsilon>0$ and $L \in \mathbb{Z}_{>0}$ and suppose that $\tkappa(t)$ and $\tkappa'(t)$ are degree-$p$ polynomials with non-negative coefficients that satisfies ${\bf (1)}$ $\max_{t \in [-1,1]} \abs{\tkappa(t) - \kappa(t)} = \bigo\left(\frac{\varepsilon}{L}\right)$ and $\max_{t \in [-1,1]} \abs{\tkappa'(t) - \kappa'(t)} = \bigo\left(\frac{\varepsilon}{L}\right)$, ${\bf (2)}$ $\max_{\abs{t} \le 1+\bigo(\varepsilon)} \abs{\tkappa(t+\gamma) - \tkappa(t)} \le \bigo(\gamma)$ and $\max_{\abs{t} \le 1+\bigo(\varepsilon)} \abs{\tkappa'(t+\gamma) - \tkappa'(t)} \le \bigo(\gamma)$ for any $|\gamma| \le \bigo(\varepsilon)$.
If $m = \Omega\left( \frac{L^4 p}{\varepsilon^2} \cdot \log^3n \right)$ and $m' = \Omega\left( \frac{L^2}{\varepsilon^2} \cdot \log^3n \right)$, then for any $y,z \in \Rbb^{d_1 \times d_2 \times c}$, the output of \cref{alg-cntk-sketch} satisfies
\[  \Pr \left[ \left| \left< \Psi_{\mathrm{cntk}}^{(L)}(y) , \Psi_{\mathrm{cntk}}^{(L)}(z) \right> - \Theta_{\mathrm{cntk}}^{(L)}(y,z) \right| > \varepsilon \cdot \sqrt{ \Theta_{\mathrm{cntk}}^{(L)}(y,y) \cdot \Theta_{\mathrm{cntk}}^{(L)}(z,z) } \right] \le \frac{1}{\poly{n}}. \]
Furthermore, for every image $x \in \Rbb^{d_1 \times d_2 \times c}$, $\Psi_{\mathrm{cntk}}^{(L)}(x) \in \Rbb^{m'}$ can be computed in time $\bigo\left( Lp^2 m \log m \cdot d_1 d_2 \right)$.
\end{theorem}
\begin{proof}
	The correctness proof is by induction on the value of $h=0,1,2, \ldots L$. 
	More formally, consider the following invariants for every iteration $h=0,1,2, \ldots L$ of the algorithm:
	\begin{enumerate}[leftmargin=1.5cm]
		\item[${\bf P_1(h) :}$] Simultaneously for all $i,i' \in [d_1]$ and $j,j' \in [d_2]$:
		\[ \begin{split}
			&\left| \left< \phi_{i,j}^{(h)}(y), \phi_{i',j'}^{(h)}(z) \right> - \Gamma_{i,j,i',j'}^{(h)}\left( y , z \right) \right| \le ({h+1}) \cdot \frac{\varepsilon}{60L^2} \cdot \frac{\sqrt{N_{i,j}^{(h)}(y) \cdot N_{i',j'}^{(h)}(z)}}{q^2},\\
			&\left| \left\| \phi_{i,j}^{(h)}(y) \right\|_2^2 - \Gamma_{i,j,i,j}^{(h)}\left( y , y \right) \right| \le  \frac{({h+1}) \cdot\varepsilon}{60L^2} \cdot \frac{N_{i,j}^{(h)}(y)}{q^2}, \\ 
			&\left| \left\| \phi_{i',j'}^{(h)}(z) \right\|_2^2 - \Gamma_{i',j',i',j'}^{(h)}\left( z , z \right) \right| \le  \frac{({h+1}) \cdot\varepsilon}{60L^2} \cdot \frac{N_{i',j'}^{(h)}(z)}{q^2}.
		\end{split} \]
		\item[${\bf P_2(h) :}$] Simultaneously for all $i,i' \in [d_1]$ and $j,j' \in [d_2]$:
		\[ \begin{split}
			&\left| \left< \psi_{i,j}^{(h)}(y), \psi_{i',j'}^{(h)}(z) \right> - \Pi_{i,j,i',j'}^{(h)}\left( y , z \right) \right| \le \begin{cases}
				\frac{\varepsilon}{10} \cdot \frac{h^2}{L+1} \cdot \sqrt{N^{(h+1)}_{i,j}(y) \cdot N^{(h+1)}_{i',j'}(z)} &\text{if } h<L\\
				\frac{\varepsilon}{10} \cdot \frac{L-1}{q^2} \cdot \sqrt{N^{(L)}_{i,j}(y) \cdot N^{(L)}_{i',j'}(z)} &\text{if } h=L
			\end{cases},\\
			&{(\text{only for } h<L):} ~~~ \left| \left\| \psi_{i,j}^{(h)}(y)\right\|_2^2 - \Pi_{i,j,i,j}^{(h)}\left( y , y \right) \right| \le \frac{\varepsilon}{10} \cdot \frac{h^2}{L+1} \cdot N^{(h+1)}_{i,j}(y),\\
			&{(\text{only for } h<L):} ~~~ \left| \left\| \psi_{i',j'}^{(h)}(z)\right\|_2^2 - \Pi_{i',j',i',j'}^{(h)}\left( z , z \right) \right| \le \frac{\varepsilon}{10} \cdot \frac{h^2}{L+1} \cdot N^{(h+1)}_{i',j'}(z).
		\end{split}
		\]
		
	\end{enumerate}
We prove that probabilities $\Pr[ P_1(0)]$ and $\Pr[ P_2(0)|P_1(0)]$ are both greater than $1 - \frac{1}{\poly{n}}$. Additionally, for every $h = 1,2, \ldots L$, we prove that the conditional probabilities $\Pr[ P_1(h) | P_1(h-1)]$ and $\Pr[ P_2(h) | P_2(h-1), P_1(h), P_1(h-1)]$ are greater than $1 - \frac{1}{\poly{n}}$.
These invariants immediately give the correctness proof.

The {\bf base of induction} corresponds to $h=0$. By line~\ref{cntk-sketch-covar-zero} of the algorithm, $\phi_{i,j}^{(0)}(y) = y_{i,j,:}$ and $\phi_{i',j'}^{(0)}(z) = z_{i',j',:}$, therefore, by using \cref{eq:dp-cntk-covar-simplified}, it trivially holds that $\Pr[ P_1(0) ] = 1 \ge 1 - \frac{1}{\poly{n}}$. Moreover, by line~\ref{cntk-sketch-covar-zero}, we have that $\psi_{i,j}^{(0)}(y) = 0$ and $\psi_{i',j'}^{(0)}(z) = 0$, thus, by \cref{eq:dp-cntk}, it trivially holds that $\Pr[P_2(0)|P_1(0)] = 1 \ge 1 - \frac{1}{\poly{n}}$. This completes the base of induction.

We now proceed to prove the {\bf inductive step}. By assuming the inductive hypothesis for $h-1$, we prove that statements $P_1(h)$ and $P_2(h)$ hold. More precisely, first we condition on the statement $P_1(h-1)$ being true for some $h \ge 1$, and then prove that $P_1(h)$ holds with probability at least $1 - \frac{1}{\poly{n}}$. Next we show that conditioned on statements $P_2(h-1), P_1(h), P_1(h-1)$ being true, $P_2(h)$ holds with probability at least $1 - \frac{1}{\poly{n}}$. This will complete the induction.

First, by conditioning on the inductive hypothesis $P_1(h-1)$ and using the definition of $\mu_{i,j}^{(h)}(\cdot)$ in line~\ref{eq:def-Z-vectors-mu} of the algorithm and applying Cauchy–Schwarz inequality and invoking \cref{properties-gamma} we find that,
	\begin{equation}\label{eq:mu-innerprod-bound}
		\begin{split}
			&\left| \left<\mu_{i,j}^{(h)}(y), \mu_{i',j'}^{(h)}(z)\right> - \frac{\sum_{a=-\frac{q-1}{2}}^{\frac{q-1}{2}} \sum_{b=-\frac{q-1}{2}}^{\frac{q-1}{2}}  \Gamma_{i+a,j+b,i'+a,j'+b}^{(h-1)}\left( y , z \right)}{\sqrt{N^{(h)}_{i,j}(y) \cdot N^{(h)}_{i',j'}(z)}} \right|\\ 
			& \le \frac{\sum_{a=-\frac{q-1}{2}}^{\frac{q-1}{2}} \sum_{b=-\frac{q-1}{2}}^{\frac{q-1}{2}}  \sqrt{N_{i+a,j+b}^{(h-1)}(y) \cdot N_{i'+a,j'+b}^{(h-1)}(z)}}{q^2 \cdot \sqrt{N^{(h)}_{i,j}(y) \cdot N^{(h)}_{i',j'}(z)}} \cdot  \frac{h \cdot \varepsilon}{60L^2}\\
			& \le \frac{\sqrt{\sum_{a=-\frac{q-1}{2}}^{\frac{q-1}{2}} \sum_{b=-\frac{q-1}{2}}^{\frac{q-1}{2}} N_{i+a,j+b}^{(h-1)}(y) / q^2 } \cdot \sqrt{\sum_{a=-\frac{q-1}{2}}^{\frac{q-1}{2}} \sum_{b=-\frac{q-1}{2}}^{\frac{q-1}{2}} N_{i'+a,j'+b}^{(h-1)}(z) / q^2}}{ \sqrt{N^{(h)}_{i,j}(y) \cdot N^{(h)}_{i',j'}(z)}} \cdot  \frac{h \cdot \varepsilon}{60L^2}\\
			&= h \cdot \frac{\varepsilon}{60L^2},
		\end{split}
	\end{equation}
	where the last line follows from \cref{eq:dp-cntk-norm-simplified}.
	
	Furthermore, if we let the collection of vectors $\left\{Z^{(h)}_{i,j,\ell}(y)\right\}_{\ell=0}^{p}$ and $\left\{Z^{(h)}_{i,j,\ell}(z)\right\}_{\ell=0}^{p}$ be defined as per line~\ref{eq:def-Z-vectors-mu} of the algorithm, then by \cref{lem-polysketch} and union bound, the following inequalities hold, with probability at least $1 - \frac{1}{\poly{n}}$, simultaneously for all $\ell = 0,1,2, \ldots p$, all $i,i' \in [d_1]$ and $j,j' \in [d_2]$:
	\begin{align}
		&\left|\left<Z^{(h)}_{i,j, \ell}(y), Z^{(h)}_{i',j', \ell}(z)\right> - \left<\mu_{i,j}^{(h)}(y), \mu_{i',j'}^{(h)}(z)\right>^\ell \right| \le \bigo\left( \frac{\varepsilon}{L^2} \right) \left\| \mu_{i,j}^{(h)}(y) \right\|_2^\ell \left\| \mu_{i',j'}^{(h)}(z)\right\|_2^\ell \nonumber\\
		&\left\| Z^{(h)}_{i,j, \ell}(y) \right\|_2^2 \le \frac{11}{10} \cdot \left\| \mu_{i,j}^{(h)}(y) \right\|_2^{2\ell} \label{eq:Zinner-prod-bound}\\
		& \left\| Z^{(h)}_{i',j', \ell}(z) \right\|_2^2 \le \frac{11}{10} \cdot \left\| \mu_{i',j'}^{(h)}(z) \right\|_2^{2\ell} \nonumber
	\end{align}
	Therefore, by Cauchy–Schwarz inequality, we find that with probability at least $1 - \frac{1}{\poly{n}}$, the following holds simultaneously for all $i,i' \in [d_1]$ and $j,j' \in [d_2]$:
	\begin{equation} \label{eq:phi-inner-prod-bound2}
		\left| \left< \phi_{i,j}^{(h)}(y), \phi_{i',j'}^{(h)}(z)\right> - \frac{\sqrt{N^{(h)}_{i,j}(y) N^{(h)}_{i',j'}(z)}}{q^2} \cdot \tkappa\left( \left<\mu_{i,j}^{(h)}(y), \mu_{i',j'}^{(h)}(z)\right> \right) \right| \le \bigo\left(\frac{\varepsilon}{L^2}\right) \cdot B,
	\end{equation}
	where $B:= \frac{\sqrt{N^{(h)}_{i,j}(y) N^{(h)}_{i',j'}(z)}}{q^2} \cdot \sqrt{\tkappa\left(\|\mu^{(h)}_{i,j}(y)\|_2^2\right) \cdot \tkappa\left(\|\mu^{(h)}_{i',j'}(z)\|_2^2\right)}$.
	
	By conditioning on the inductive hypothesis $P_1(h-1)$ and using \cref{properties-gamma} we have, 
	\[
	\left| \left\| \mu_{i,j}^{(h)}(y) \right\|_2^2 - 1 \right| \le h \cdot \frac{\varepsilon}{60L^2}, \text{ and } \left| \left\| \mu_{i',j'}^{(h)}(z) \right\|_2^2 - 1 \right| \le h \cdot \frac{\varepsilon}{60L^2}.
	\]
	Therefore, the precondition of the theorem implies that $\left| \tkappa\left(\|\mu^{(h)}_{i,j}(y)\|_2^2\right) - \tkappa(1) \right| \le h \cdot \frac{\varepsilon}{60L^2}$ and $\left| \tkappa\left(\|\mu^{(h)}_{i',j'}(z)\|_2^2\right) - \tkappa(1) \right| \le h \cdot \frac{\varepsilon}{60L^2}$. Consequently, because $\tkappa(1) \le 1.01 \kappa(1) = 1.01$, we find that
	\[ B \le \frac{11}{10} \cdot \frac{\sqrt{N^{(h)}_{i,j}(y) N^{(h)}_{i',j'}(z)}}{q^2}.\]
	By plugging this into \cref{eq:phi-inner-prod-bound2} we find that the following holds simultaneously for all $i,i' \in [d_1]$ and all $j,j' \in [d_2]$, with probability at least $1 - \frac{1}{\poly{n}}$,
	\small
	\begin{equation} \label{eq:phi-inner-prod-bound3}
		\left| \left< \phi_{i,j}^{(h)}(y), \phi_{i',j'}^{(h)}(z)\right> - \frac{\sqrt{N^{(h)}_{i,j}(y) N^{(h)}_{i',j'}(z)}}{q^2} \cdot \tkappa\left( \left<\mu_{i,j}^{(h)}(y), \mu_{i',j'}^{(h)}(z)\right> \right) \right| \le \bigo\left(\frac{\varepsilon}{L^2}\right) \cdot \frac{\sqrt{N^{(h)}_{i,j}(y) N^{(h)}_{i',j'}(z)}}{q^2}.
	\end{equation}
	\normalsize
		
	We recall that $A := \frac{1}{\sqrt{N^{(h)}_{i,j}(y) \cdot N^{(h)}_{i',j'}(z)}} \sum_{a=-\frac{q-1}{2}}^{\frac{q-1}{2}} \sum_{b=-\frac{q-1}{2}}^{\frac{q-1}{2}}  \Gamma_{i+a,j+b,i'+a,j'+b}^{(h-1)}\left( y , z \right)$ and
	\[
	\Gamma_{i,j,i',j'}^{(h)}(y,z) = \frac{\sqrt{N^{(h)}_{i,j}(y) N^{(h)}_{i',j'}(z)}}{q^2} \kappa(A).
	\]
	Note that by \cref{properties-gamma} and \cref{eq:dp-cntk-norm-simplified}, $-1 \le A \le 1$. Hence, using the precondition of the theorem and \cref{eq:mu-innerprod-bound} to find that,
	\[ \left|\tkappa\left( \left<\mu_{i,j}^{(h)}(y), \mu_{i',j'}^{(h)}(z)\right> \right) - \tkappa\left(A\right) \right| \le h \cdot \frac{\varepsilon}{60L^2}. \]
	By incorporating the above inequality into \cref{eq:phi-inner-prod-bound3} using triangle inequality we find that, with probability at least $1 - \frac{1}{\poly{n}}$, the following holds simultaneously for all $i,i' \in [d_1]$ and all $j,j'\in [d_2]$:
	\begin{equation} \label{eq:phi-inner-prod-bound4}
		\left| \left< \phi_{i,j}^{(h)}(y), \phi_{i',j'}^{(h)}(z)\right> - \frac{\sqrt{N^{(h)}_{i,j}(y) N^{(h)}_{i',j'}(z)}}{q^2} \cdot \tkappa\left(A\right) \right| \le \left(\bigo\left(\frac{\varepsilon}{L^2}\right) + \frac{h \cdot\varepsilon}{60L^2}\right) \cdot \frac{\sqrt{N^{(h)}_{i,j}(y) N^{(h)}_{i',j'}(z)}}{q^2}.
	\end{equation}

	Additionally, since $-1 \le A \le 1$, using the preconditions of the theorem we can conclude that $\left| \tkappa\left(A\right) - \kappa(A) \right| \le \frac{\varepsilon}{76 L^2}$.
	By combining the above inequality with \cref{eq:phi-inner-prod-bound4} via triangle inequality and using the fact that, by \cref{eq:dp-cntk-covar-simplified}, 
	we get the following inequality, with probability at least $1 - \frac{1}{\poly{n}}$
	\[  \left| \left< \phi_{i,j}^{(h)}(y), \phi_{i',j'}^{(h)}(z)\right> - \Gamma_{i,j,i',j'}^{(h)}(y,z) \right| \le (h+1) \cdot \frac{\varepsilon}{60L^2} \cdot \frac{\sqrt{N^{(h)}_{i,j}(y) N^{(h)}_{i',j'}(z)}}{q^2}. \]
	Similarly, we can prove that with probability at least $1 - \frac{1}{\poly{n}}$ the following hold, simultaneously for all $i,i' \in [d_1]$ and $j,j' \in [d_2]$,
	\begin{align*}
	&\left| \left\| \phi_{i,j}^{(h)}(y)\right\|_2^2 - \Gamma_{i,j,i,j}^{(h)}(y,y) \right| \le \frac{(h+1) \varepsilon}{60L^2} \cdot \frac{N^{(h)}_{i,j}(y)}{q^2},\\
	&\left| \left\| \phi_{i',j'}^{(h)}(z)\right\|_2^2 - \Gamma_{i',j',i',j'}^{(h)}(z,z) \right| \le \frac{(h+1) \varepsilon}{60L^2} \cdot \frac{N^{(h)}_{i',j'}(z)}{q^2}.
	\end{align*}
	This is sufficient to prove the inductive step for statement $P_1(h)$, i.e., $\Pr[P_1(h)|P_1(h-1)] \ge 1 - \frac{1}{\poly{n}}$.
	
	Now we prove the inductive step for statement $P_2(h)$. That is, we prove that conditioned on $P_2(h-1), P_1(h)$, and $P_1(h-1)$, $P_2(h)$ holds with probability at least $1-\frac{1}{\poly{n}}$.
    First note that using the definition of $\dt{\phi}_{i,j}^{(h)}(y), \dt{\phi}_{i',j'}^{(h)}(z)$ in line~\ref{eq:cntk-map-phidot} of the algorithm and \cref{eq:Zinner-prod-bound},we find that with probability at least $1 - \frac{1}{\poly{n}}$, the following holds simultaneously for all $i,i' \in [d_1]$ and $j,j' \in [d_2]$:
	\begin{equation} \label{cntk:phi-dot-bound2}
		\left| \left< \dt{\phi}_{i,j}^{(h)}(y), \dt{\phi}_{i',j'}^{(h)}(z)\right> - \frac{1}{q^2} \cdot \tkappa'\left( \left<\mu_{i,j}^{(h)}(y), \mu_{i',j'}^{(h)}(z)\right> \right) \right| \le \bigo\left(\frac{\varepsilon}{L^2}\right) \cdot \widehat{B},
	\end{equation}
	where $\widehat{B}:= \frac{1}{q^2} \cdot \sqrt{\tkappa'\left(\|\mu_{i,j}^{(h)}(y)\|_2^2\right) \cdot \tkappa'\left(\|\mu_{i',j'}^{(h)}(z)\|_2^2\right)}$.
	By conditioning on the inductive hypothesis $P_1(h-1)$ and using \cref{properties-gamma} we have, 
	$\left| \left\| \mu_{i,j}^{(h)}(y) \right\|_2^2 - 1 \right| \le h \cdot \frac{\varepsilon}{60L^2}$ and $\left| \left\| \mu_{i',j'}^{(h)}(z) \right\|_2^2 - 1 \right| \le h \cdot \frac{\varepsilon}{60L^2}$.
	Therefore, the precondition of the theorem implies that $\left| \tkappa'\left(\|\mu^{(h)}_{i,j}(y)\|_2^2\right) - \tkappa'(1) \right| \le h \cdot \frac{\varepsilon}{20L^2}$ and $\left| \tkappa'\left(\|\mu^{(h)}_{i',j'}(z)\|_2^2\right) - \tkappa'(1) \right| \le h \cdot \frac{\varepsilon}{20L^2}$. Consequently, because $\tkappa'(1) \le 1.01 \kappa'(1) = 1.01$, we find that
	\[ \widehat{B} \le \frac{11}{10}\frac{1}{q^2}.\]
	By plugging this into \cref{cntk:phi-dot-bound2} we get the following, with probability at least $1 - \bigo\left( \frac{\delta}{L} \right)$,
	\begin{equation} \label{cntk:phi-dot-bound3}
		\left| \left< \dt{\phi}_{i,j}^{(h)}(y), \dt{\phi}_{i',j'}^{(h)}(z)\right> - \frac{1}{q^2} \cdot \tkappa'\left( \left<\mu_{i,j}^{(h)}(y), \mu_{i',j'}^{(h)}(z)\right> \right) \right| \le \bigo\left(\frac{\varepsilon}{q^2\cdot L^2}\right).
	\end{equation}
	Furthermore, 
    we can use the precondition of the theorem to find that \cref{eq:mu-innerprod-bound} implies the following,
	\[ \left|\tkappa'\left( \left<\mu_{i,j}^{(h)}(y), \mu_{i',j'}^{(h)}(z)\right> \right) - \tkappa'\left(A\right) \right| \le \frac{h \cdot \varepsilon}{20L^2}. \]
	By incorporating the above inequality into \cref{cntk:phi-dot-bound3} using triangle inequality, we find that, with probability at least $1 - \frac{1}{\poly{n}}$, the following holds simultaneously for all $i,i' \in [d_1]$ and all $j,j'\in [d_2]$:
	\begin{equation} \label{cntk:phi-dot-bound4}
		\left| \left< \dt{\phi}_{i,j}^{(h)}(y), \dt{\phi}_{i',j'}^{(h)}(z)\right> - \frac{1}{q^2} \cdot \tkappa'\left(A\right) \right| \le \bigo\left(\frac{\varepsilon}{q^2 L^2}\right) + \frac{h}{q^2} \cdot \frac{\varepsilon}{20L^2}.
	\end{equation}
	Since $-1 \le A \le 1$, we can use the precondition of the theorem to conclude $\left| \tkappa'\left(A\right) - \kappa'\left(A\right) \right| \le \frac{\varepsilon}{15 L^2}$.
	By combining this inequality with \cref{cntk:phi-dot-bound4} via triangle inequality and using the fact that $\dt{\Gamma}_{i,j,i',j'}^{(h)}(y,z) = \frac{1}{q^2} \cdot \kappa'(A)$, we get the following bound simultaneously for all $i,i' \in [d_1]$ and all $j,j'\in [d_2]$, with probability at least $1 - \frac{1}{\poly{n}}$:
	\begin{equation}\label{cntk:phi-dot-bound-final}
		\left| \left< \dt{\phi}_{i,j}^{(h)}(y), \dt{\phi}_{i',j'}^{(h)}(z)\right> - \dt{\Gamma}_{i,j,i',j'}^{(h)}(y,z) \right| \le \frac{1}{q^2} \cdot \frac{\varepsilon}{8L}.
	\end{equation}
	Similarly we can prove that with probability at least $1 - \frac{1}{\poly{n}}$, the following hold simultaneously for all $i,i' \in [d_1]$ and all $j,j'\in [d_2]$,
	\begin{equation}\label{cntk:phi-norm-bound-final}
		\left| \left\| \dt{\phi}_{i,j}^{(h)}(y)\right\|_2^2 - \dt{\Gamma}_{i,j,i,j}^{(h)}(y,y) \right| \le \frac{1}{q^2} \cdot \frac{\varepsilon}{8L}, \text{ and } \left| \left\| \dt{\phi}_{i',j'}^{(h)}(z)\right\|_2^2 - \dt{\Gamma}_{i',j',i',j'}^{(h)}(z,z) \right| \le \frac{1}{q^2} \cdot \frac{\varepsilon}{8L}.
	\end{equation}
	We will use \cref{cntk:phi-dot-bound-final} and \cref{cntk:phi-norm-bound-final} to prove the inductive step for $P_2(h)$.

Next, we consider two cases for the value of $h$. When $h<L$, the vectors $\psi_{i,j}^{(h)}(y) , \psi_{i',j'}^{(h)}(z)$ are defined in line~\ref{psi-cntk} and when $h=L$, these vectors are defined differently in line~\ref{psi-cntk-last}. First we consider the case of $h<L$.
If we let $f_{i,j} := \psi^{(h-1)}_{i,j}(y) \otimes \dt{\phi}_{i,j}^{(h)}(y)$ and $g_{i',j'} := \psi^{(h-1)}_{i',j'}(z) \otimes \dt{\phi}_{i',j'}^{(h)}(z)$ and $\eta_{i,j}^{(h)}(y) := \left(Q^2\cdot f_{i,j}\right) \oplus \phi_{i,j}^{(h)}(y)$ and $\eta_{i',j'}^{(h)}(z) := \left(Q^2\cdot g_{i',j'}\right) \oplus \phi_{i',j'}^{(h)}(z)$, then by \cref{lem-polysketch} and union bound, with probability at least $1 - \frac{1}{\poly{n}}$, we have the following inequalities simultaneously for all $i,i' \in [d_1]$ and $j,j' \in [d_2]$:
	\begin{align}
		& \left|\left<\eta^{(h)}_{i,j}(y), \eta^{(h)}_{i',j'}(z)\right> - \langle f_{i,j},g_{i',j'} \rangle - \left<\phi_{i,j}^{(h)}(y), \phi_{i',j'}^{(h)}(z)\right> \right| \le \bigo\left( \frac{\varepsilon}{L} \right) \cdot \left\| f_{i,j} \right\|_2 \left\| g_{i',j'} \right\|_2  \nonumber\\
		& \left\| \eta^{(h)}_{i,j}(y)\right\|_2^2 \le \frac{11}{10} \cdot  \|f_{i,j}\|_2^2 + \left\| \phi_{i,j}^{(h)}(y) \right\|_2^2 \label{eta-innerprod-bound}\\
		& \left\| \eta^{(h)}_{i',j'}(z)\right\|_2^2 \le \frac{11}{10} \cdot \|g_{i',j'}\|_2^2 + \left\| \phi_{i',j'}^{(h)}(z) \right\|_2^2 \nonumber
	\end{align}
		Now we bound the term $\left|\left<\eta^{(h)}_{i,j}(y), \eta^{(h)}_{i',j'}(z)\right> - \langle f_{i,j},g_{i',j'} \rangle - \left<\phi_{i,j}^{(h)}(y), \phi_{i',j'}^{(h)}(z)\right> \right|$ using \cref{eta-innerprod-bound}, \cref{cntk:phi-norm-bound-final}, and \cref{properties-gamma} along with inductive hypotheses $P_2(h-1)$. With probability at least $1 - \frac{1}{\poly{n}}$ the following holds simultaneously for all $i,i' \in [d_1]$ and all $j,j' \in [d_2]$:
		
	\[
	\begin{split}
		&\left|\left<\eta^{(h)}_{i,j}(y), \eta^{(h)}_{i',j'}(z)\right> - \langle f_{i,j},g_{i',j'} \rangle - \left<\phi_{i,j}^{(h)}(y), \phi_{i',j'}^{(h)}(z)\right> \right| \\
		&\qquad \le \bigo\left( \frac{\varepsilon}{L} \right) \cdot \sqrt{\Pi_{i,j,i,j}^{(h-1)}(y,y) \cdot \dot{\Gamma}_{i,j,i,j}^{(h)}(y,y)\cdot \Pi_{i',j',i',j'}^{(h-1)}(z,z) \cdot \dot{\Gamma}_{i',j',i',j'}^{(h)}(z,z)}\\
		&\qquad = \bigo\left( \frac{\varepsilon \cdot h}{L} \right) \cdot \frac{ \sqrt{N_{i,j}^{(h)}(y) \cdot N_{i',j'}^{(h)}(z)}}{q^2},
	\end{split}
	\]
	where the last line above follows from \cref{properties-gamma} together with the fact that $\dot{\Gamma}_{i,j,i,j}^{(h)}(y,y) = \dot{\Gamma}_{i',j',i',j'}^{(h)}(z,z) = \frac{1}{q^2}$.
By combining the above with inductive hypotheses $P_1(h), P_2(h-1)$ and \cref{cntk:phi-dot-bound-final} via triangle inequality and invoking \cref{properties-gamma} we get that the following holds simultaneously for all $i,i' \in [d_1]$ and all $j,j' \in [d_2]$, with probability at least $1 - \frac{1}{\poly{n}}$,
	\small
	\begin{align}\label{eta-innerprod-bound2}
		&\left|\left<\eta^{(h)}_{i,j}(y), \eta^{(h)}_{i',j'}(z)\right> - \Pi_{i,j,i',j'}^{(h-1)}(y,z)\cdot \dot{\Gamma}_{i,j,i',j'}^{(h)}(y,z) - \Gamma_{i,j,i',j'}^{(h)}(y,z) \right| \nonumber\\
		& \le \frac{\varepsilon}{10} \cdot \frac{(h-1)^2}{L+1} \cdot \sqrt{N_{i,j}^{(h)}(y) \cdot N_{i',j'}^{(h)}(z)} \cdot \left(\left| \dot{\Gamma}_{i,j,i',j'}^{(h)}(y,z) \right|+ \frac{1}{q^2} \cdot \frac{\varepsilon}{8L} \right) + \frac{1}{q^2} \cdot \frac{\varepsilon}{8L} \cdot \left| \Pi_{i,j,i',j'}^{(h-1)}(y,z)\right| \nonumber\\ 
		& + \frac{(h+1) \cdot \varepsilon}{60L^2} \cdot \frac{\sqrt{N_{i,j}^{(h)}(y) \cdot N_{i',j'}^{(h)}(z)}}{q^2} + \bigo\left( \frac{\varepsilon \cdot h}{L} \right) \cdot \frac{ \sqrt{N_{i,j}^{(h)}(y) \cdot N_{i',j'}^{(h)}(z)}}{q^2}\nonumber\\
		& \le \frac{\varepsilon}{10} \cdot \frac{(h-1)^2}{L+1} \cdot \frac{\sqrt{N_{i,j}^{(h)}(y) \cdot N_{i',j'}^{(h)}(z)}}{q^2} \cdot \left(1 + \frac{\varepsilon}{8L} \right) + \frac{h-1}{q^2} \cdot \frac{\varepsilon}{8L} \cdot \sqrt{N_{i,j}^{(h)}(y) \cdot N_{i',j'}^{(h)}(z)}\nonumber\\
		&+ \left( \frac{(h+1) \cdot \varepsilon}{60L^2} + \bigo\left( \frac{\varepsilon \cdot h}{L} \right) \right) \cdot \frac{\sqrt{N_{i,j}^{(h)}(y) \cdot N_{i',j'}^{(h)}(z)}}{q^2} \nonumber\\
		&\le \frac{\varepsilon}{10} \cdot \frac{h^2-h/2}{L+1} \cdot \frac{\sqrt{N_{i,j}^{(h)}(y) \cdot N_{i',j'}^{(h)}(z)}}{q^2}.\nonumber
	\end{align}
	\normalsize
By plugging the above bound into the definition of  in line~\ref{psi-cntk} of the algorithm using triangle inequality and using \cref{eq:dp-cntk} we get the following with probability at least $1 - \frac{1}{\poly{n}}$:
	\begin{equation}\label{eq:psi-bound3}
		\begin{split}
			&\left| \left< \psi_{i,j}^{(h)}(y) , \psi_{i',j'}^{(h)}(z) \right> - \Pi_{i,j,i',j'}^{(h)}(y,z) \right|\\ 
			&\le \frac{\varepsilon}{10} \cdot \frac{h^2-h/2}{L+1} \cdot \sum_{a=-\frac{q-1}{2}}^{\frac{q-1}{2}} \sum_{b=-\frac{q-1}{2}}^{\frac{q-1}{2}}  \frac{\sqrt{N_{i+a,j+b}^{(h)}(y) \cdot N_{i'+a,j'+b}^{(h)}(z)}}{q^2}\\
			&\le  \frac{\varepsilon}{10} \cdot \frac{h^2-h/2}{L+1} \cdot \sqrt{\sum_{a=-\frac{q-1}{2}}^{\frac{q-1}{2}} \sum_{b=-\frac{q-1}{2}}^{\frac{q-1}{2}} \frac{N_{i+a,j+b}^{(h)}(y)}{q^2}} \cdot \sqrt{ \sum_{a=-\frac{q-1}{2}}^{\frac{q-1}{2}} \sum_{b=-\frac{q-1}{2}}^{\frac{q-1}{2}} \frac{N_{i'+a,j'+b}^{(h)}(z)}{q^2}}\\
			&\le \frac{\varepsilon}{10} \cdot \frac{h^2}{L+1} \cdot\sqrt{ N^{(h+1)}_{i,j}(y) \cdot N^{(h+1)}_{i',j'}(z)}.
		\end{split}
	\end{equation}
    
    Similarly, we can prove that with probability at least $1 - \frac{1}{\poly{n}}$ the following hold simultaneously for all $i,i' \in [d_1]$ and all $j,j' \in [d_2]$,
	\[ \begin{split}
		&\left| \left\| \psi_{i,j}^{(h)}(y)\right\|_2^2 - \Pi_{i,j,i,j}^{(h)}(y,y) \right| \le \frac{\varepsilon}{10} \cdot \frac{h^2}{L+1} \cdot N^{(h+1)}_{i,j}(y),\\ 
		&\left| \left\| \psi_{i',j'}^{(h)}(z)\right\|_2^2 - \Pi_{i',j',i',j'}^{(h)}(z,z) \right| \le \frac{\varepsilon}{10} \cdot \frac{h^2}{L+1} \cdot N^{(h+1)}_{i',j'}(z).
	\end{split} \]
	This is sufficient to prove the inductive step for statement $P_2(h)$, in the case of $h<L$, i.e., $\Pr[P_2(h)|P_2(h-1), P_1(h), P_1(h-1)] \ge 1 - \frac{1}{\poly{n}}$.

Now we prove the inductive step for $P_2(h)$ in the case of $h=L$. Similar to before, if we let $f_{i,j} := \psi^{(L-1)}_{i,j}(y) \otimes \dot{\phi}_{i,j}^{(L)}(y)$ and $g_{i',j'} := \psi^{(L-1)}_{i',j'}(z) \otimes \dot{\phi}_{i',j'}^{(L)}(z)$, then by \eqref{psi-cntk-last}, we have $\psi_{i,j}^{(L)}(y) = \left(Q^2\cdot f_{i,j}\right)$ and $\psi_{i',j'}^{(L)}(z) = \left(Q^2\cdot g_{i',j'}\right)$. Thus by \cref{lem-polysketch} and union bound, we find that, with probability at least $1 - \frac{1}{\poly{n}}$, the following inequality holds simultaneously for all $i,i' \in [d_1]$ and $j,j' \in [d_2]$:
	\[ \left|\left<\psi^{(L)}_{i,j}(y), \psi^{(L)}_{i',j'}(z)\right> - \langle f_{i,j},g_{i',j'} \rangle \right| \le \bigo\left( \frac{\varepsilon}{L} \right) \cdot \left\| f_{i,j} \right\|_2 \left\| g_{i',j'} \right\|_2.\]
	Therefore, using \eqref{cntk:phi-norm-bound-final} and \cref{properties-gamma} along with inductive hypotheses $P_2(L-1)$, with probability at least $1 - \frac{1}{\poly{n}}$, the following holds simultaneously for all $i,i' \in [d_1]$ and $j,j' \in [d_2]$,
	\[
	\begin{split}
		&\left|\left<\psi^{(L)}_{i,j}(y), \psi^{(L)}_{i',j'}(z)\right> - \langle f_{i,j},g_{i',j'} \rangle \right| \\
		&\qquad \le \bigo\left( \frac{\varepsilon}{L} \right) \cdot \sqrt{\Pi_{i,j,i,j}^{(L-1)}(y,y) \cdot \dot{\Gamma}_{i,j,i,j}^{(L)}(y,y)\cdot \Pi_{i',j',i',j'}^{(L-1)}(z,z) \cdot \dot{\Gamma}_{i',j',i',j'}^{(L)}(z,z)}\\
		&\qquad = \bigo\left( {\varepsilon} \right) \cdot \frac{\sqrt{N_{i,j}^{(L)}(y) \cdot N_{i',j'}^{(L)}(z)}}{q^2}.
	\end{split}
	\]
	By combining the above with inductive hypotheses $P_1(L), P_2(L-1)$ and \cref{cntk:phi-dot-bound-final} via triangle inequality and invoking \cref{properties-gamma} and also using the definition of $\Pi^{(L)}(y,z)$ given in \cref{eq:dp-cntk-last-layer}, we get that the following holds, simultaneously for all $i,i' \in [d_1]$ and $j,j' \in [d_2]$, with probability at least $1 - \frac{1}{\poly{n}}$,
		\begin{align}
		&\left|\left<\psi^{(L)}_{i,j}(y), \psi^{(L)}_{i',j'}(z)\right> - \Pi_{i,j,i',j'}^{(L)}(y,z) \right| \nonumber\\
		& \le \frac{\varepsilon}{10} \cdot \frac{(L-1)^2}{L+1} \cdot \sqrt{N_{i,j}^{(L)}(y) \cdot N_{i',j'}^{(L)}(z)} \cdot \left(\left| \dot{\Gamma}_{i,j,i',j'}^{(L)}(y,z) \right|+ \frac{1}{q^2} \cdot \frac{\varepsilon}{8L} \right) + \frac{1}{q^2} \cdot \frac{\varepsilon}{8L} \cdot \left| \Pi_{i,j,i',j'}^{(L-1)}(y,z)\right| \nonumber\\ 
		& + \frac{(L+1) \cdot \varepsilon}{60L^2} \cdot \frac{\sqrt{N_{i,j}^{(L)}(y) \cdot N_{i',j'}^{(L)}(z)}}{q^2} + \bigo\left( {\varepsilon} \right) \cdot \frac{ \sqrt{N_{i,j}^{(L)}(y) \cdot N_{i',j'}^{(L)}(z)}}{q^2}\nonumber\\
		& \le \frac{\varepsilon}{10} \cdot \frac{(L-1)^2}{L+1} \cdot \frac{\sqrt{N_{i,j}^{(L)}(y) \cdot N_{i',j'}^{(L)}(z)}}{q^2} \cdot \left(1 + \frac{\varepsilon}{8L} \right) +  \frac{\varepsilon}{8q^2} \cdot \sqrt{N_{i,j}^{(L)}(y) \cdot N_{i',j'}^{(L)}(z)}\nonumber\\
		&+ \left( \frac{(L+1) \cdot \varepsilon}{60L^2} + \bigo\left( {\varepsilon} \right) \right) \cdot \frac{\sqrt{N_{i,j}^{(L)}(y) \cdot N_{i',j'}^{(L)}(z)}}{q^2} \nonumber\\
		&\le \frac{\varepsilon \cdot (L-1)}{10} \cdot \frac{\sqrt{N_{i,j}^{(L)}(y) \cdot N_{i',j'}^{(L)}(z)}}{q^2}.\nonumber
	\end{align}
	This proves the inductive step for statement $P_2(h)$, in the case of $h=L$, i.e., $\Pr[P_2(L)|P_2(L-1), P_1(L), P_1(L-1)] \ge 1 - \frac{1}{\poly{n}}$.
	The induction is complete and hence the correctness of \cref{alg-cntk-sketch} is proved by union bounding over all $h = 0,1,2, \ldots L$.

The runtime of the algorithm immediately follows by invoking \cref{lem-polysketch} because computing vector $Z^{(h)}_{i,j,\ell}(x)$ for every $i,j,\ell$ and $h=1,2, \ldots L$ dominates the runtime of this algorithm.
\end{proof}

As an example, let us invoke \cref{alg-cntk-sketch} and \cref{thm-cntk-sketch} on the CNTK with GAP corresponding to the normalized Gaussian dual kernel $K_G$, defined per \cref{eq-normalzied-gaussian-kernel}. Note that the dot-product factor corresponding to this dual kernel is $\kappa(t) = \exp(t-1)$. The truncated Taylor series of this function is $\tkappa(t) = \sum_{j=0}^p \frac{t^j}{e \cdot j!}$ and the truncated Taylor series expansion of the derivative of this function is $\tkappa'(t) = \sum_{j=0}^p \frac{t^j}{e \cdot j!}$. If $p = \Omega(\log n)$ then it can be verified that polynomials $\tkappa(t), \tkappa'(t)$ satisfy the preconditions of \cref{thm-cntk-sketch}. Therefore, by \cref{thm-cntk-sketch}, we can sketch the CNTK kernel using $\bigo\left( \frac{L^4}{\varepsilon^2} \cdot d_1d_2  \cdot \poly{\log n} \right)$ running time. Also the target dimension of the sketch is $m' = \bigo\left( \frac{L^2}{\varepsilon^2} \log^3n \right)$. So the runtime of our \cref{alg-cntk-sketch} is only linear in the number of image pixels $d_1d_2$, which is in stark contrast to quadratic scaling of the exact CNTK computation \cite{arora2019exact}. In fact, using our CNTK sketching method, the kernel regression can be solved approximately in time $\bigo\left( \frac{L^4}{\varepsilon^2} \cdot d_1d_2  \cdot n ~\poly{\log n} + m'^2 \cdot n \right) = \bigo\left( \left( \frac{L^4}{\varepsilon^2} \cdot d_1d_2 + \frac{L^4}{\varepsilon^4} \right) \cdot n \cdot \poly{\log n} \right)$, which is significantly faster than the exact kernel regression which takes $\Omega\left( L (d_1d_2 \cdot n)^2 \right)$ when the number of pixels $d_1d_2$ or the training set size $n$ are large.

\section{Gauss-Hermite Quadrature Derivation}
\label{sec-appendix-guass-hermite-quadrature}

\begin{figure}[t]
\includegraphics[width=0.48\textwidth]{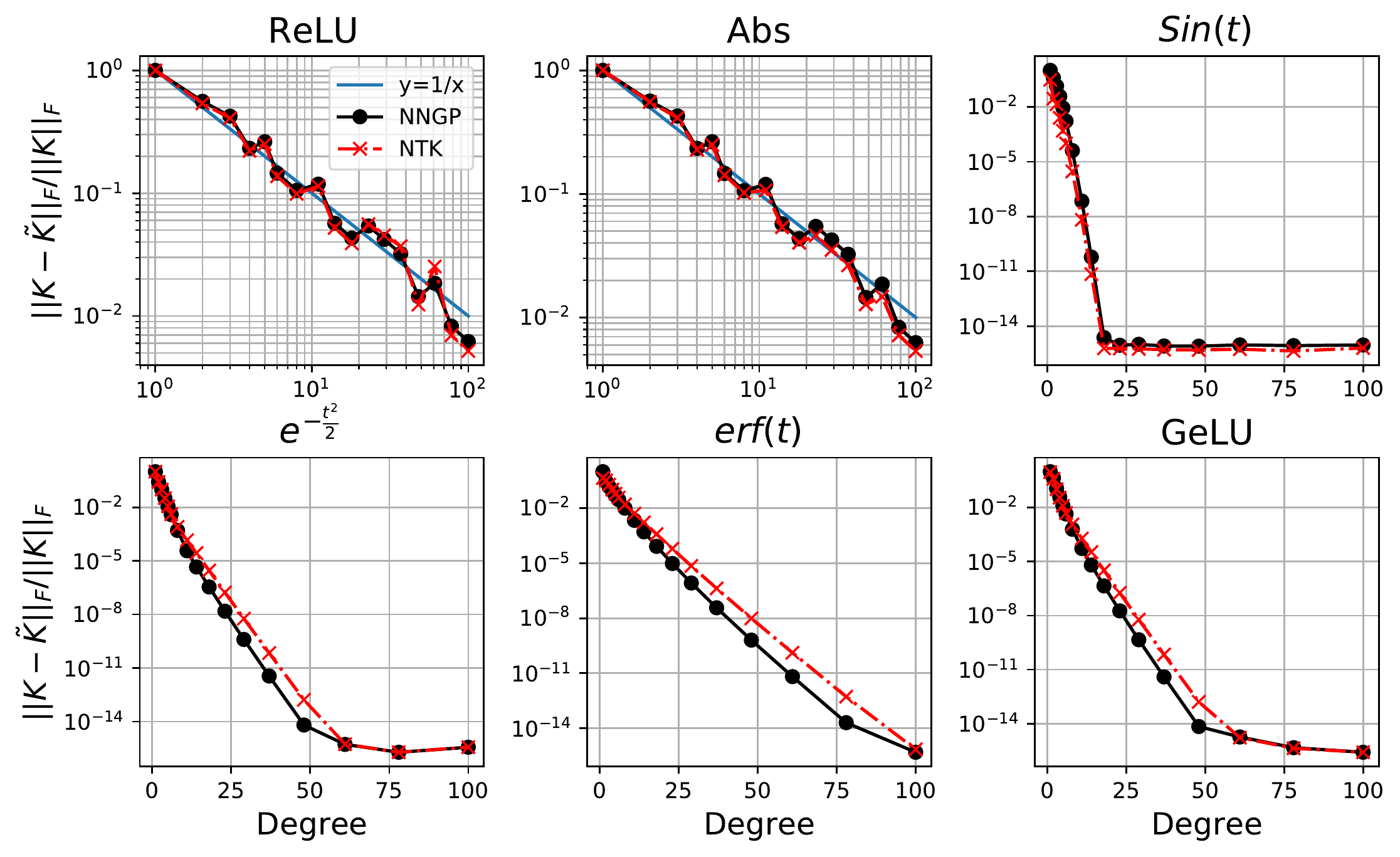}
\includegraphics[width=0.48\textwidth]{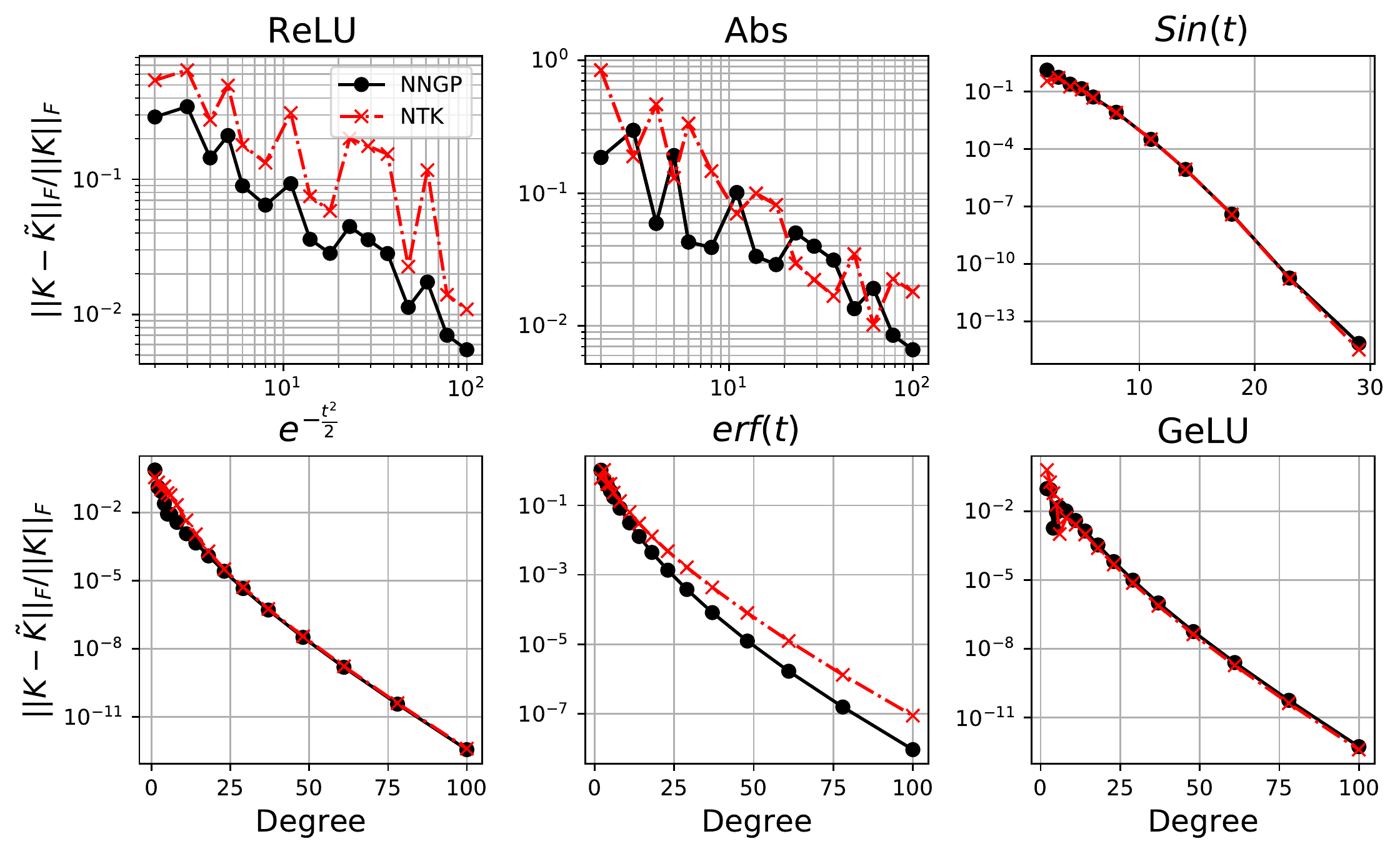}
\caption{Relative errors of dual kernel approximation via the Hermite polynomial approximation under synthetic dataset with \textbf{(left)} $n=1{,}000, d=256$ FC1, \textbf{(right)} $n=10, d=32\times32\times3$ Myrtle-5.}\label{fig-error}
\end{figure}

Here we provide more details on ~\cref{sec-guass-hermite-quadrature}. Utilizing the whitening transformation of covariance ${\bm\Lambda}$ used in the proof of \cref{lmm-derivative-dual-kernel} in \cref{sec-proof-lmm-derivative-dual-kernel} the dual activation function can be expressed as
\begin{align}
k_\sigma(a, b, c) &:= \mathop{\Ebb}_{(u,v)\sim \mathcal{N}(0,{\bm\Lambda})}\left[\sigma(u) \sigma(v)\right] \\
&=  \mathop{\Ebb}_{(\alpha,\beta)\sim \mathcal{N}(0,\I_2)}\left[\sigma(a\alpha) \cdot \sigma(bc \alpha +  b\sqrt{1-c^2} \beta)\right] \\
&= \frac{1}{2\pi}\int \int d\alpha d\beta e^{-\frac{\alpha^2}{2}} e^{-\frac{\beta^2}{2}}\left[\sigma(a\alpha) \cdot \sigma(bc \alpha +  b\sqrt{1-c^2} \beta)\right] \\
&= \frac{1}{\pi}\int \int d\alpha d\beta e^{-\alpha^2} e^{-\beta^2}\left[\sigma(\sqrt{2} a\alpha) \cdot \sigma(\sqrt{2} bc \alpha +  \sqrt{2}b\sqrt{1-c^2} \beta)\right] \\
& \approx \frac{1}{\pi} \sum^{q}_{i=1}\sum^{q}_{j=1} w_i w_j \left[\sigma(\sqrt{2} a x_i) \cdot \sigma(\sqrt{2}bc x_i + \sqrt{2} b\sqrt{1-c^2} x_j)\right]\,.
\end{align}
Here $(x_i, w_i)$,  correspond to roots of $q$-th degree (Physicist's) Hermite polynomial $H_q(x)$ and associated weights~\cite{abramowitz1988handbook}
\begin{align}
    w_i = \frac{2^{q-1} q! \sqrt{\pi}}{q^2 \left(H_{q-1}(x_i)\right)^2} = \frac{q! \sqrt{\pi}}{q^2 (h_{q-1}(\sqrt{2} x_i))^2}
\end{align}
where the conversion between physicist's to probabilist's convention $H_n(x) = 2^\frac{n}{2} h_n(\sqrt{2}x)$.
The roots are obtained by Golub-Welsch algorithm~\cite{golub1969calculation} and can be found in scientific computing package such as Scipy~\cite{2020scipy}'s \href{https://docs.scipy.org/doc/scipy/reference/generated/scipy.special.roots_hermite.html}{\texttt{scipy.special.roots\_hermite}} function. For alternative parameterization for multivariate Gauss-Hermite quadrature, refer to notes by \citet{jackel2005note}.

For activation function where exact dual activation is known, one can  measure the error from the quadrture. In \cref{fig-error}, we compute errors for $\mathrm{ReLU}$, Abs (i.e., $\sigma(t)=\abs{t}$), $\sin$, Gaussian, $\mathrm{erf}$ and GeLU activations. For non-smooth activation ($\mathrm{ReLU}$, Abs), approximation error decays as power-law like where as for smooth activation the error decays exponentially as one increases Hermite polynomial degree $q$. 

\begin{figure}[t]
	\begin{subfigure}{0.33\textwidth}
		\includegraphics[width=\textwidth]{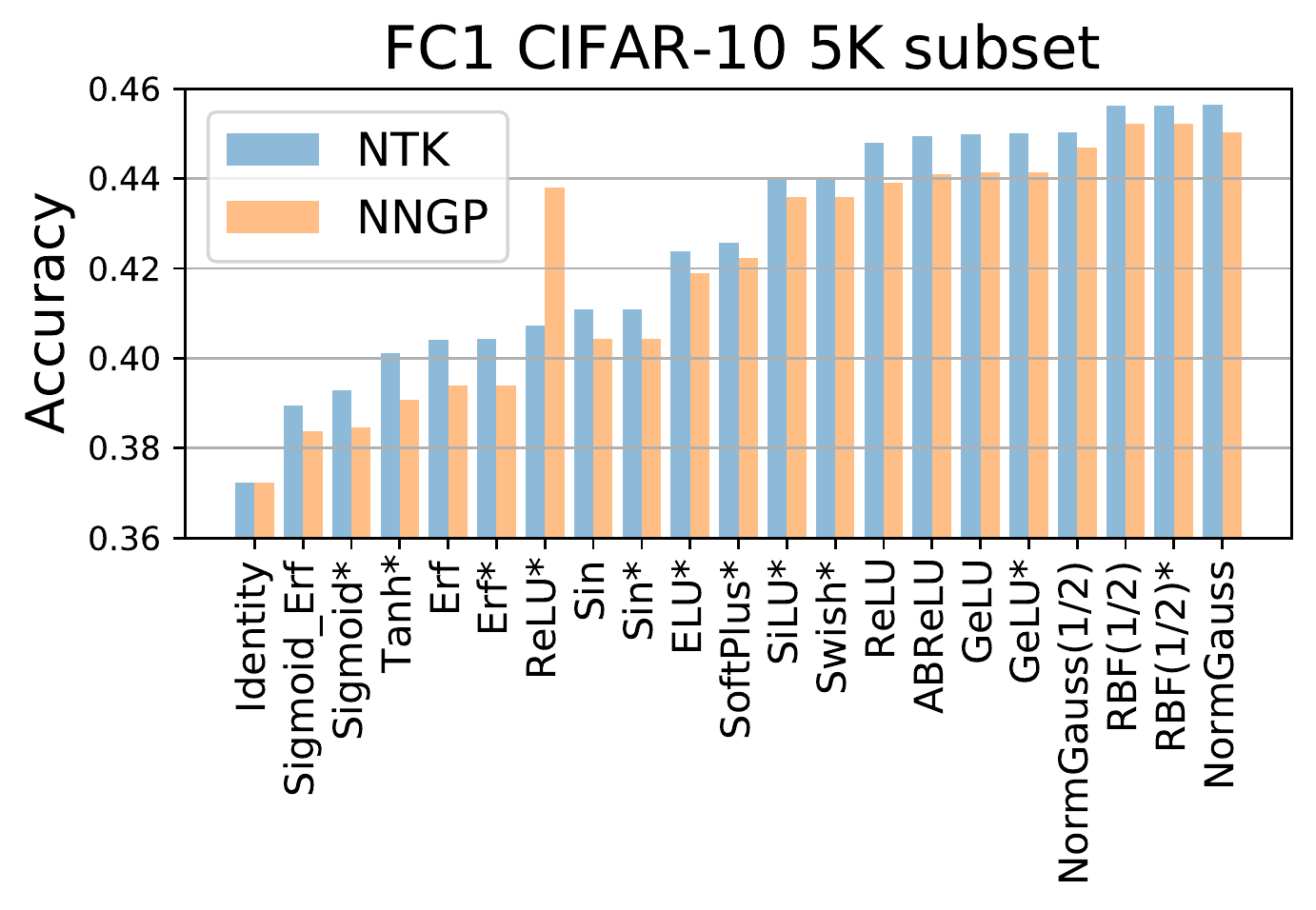}
	\end{subfigure}
	\begin{subfigure}{0.33\textwidth}
		\includegraphics[width=\textwidth]{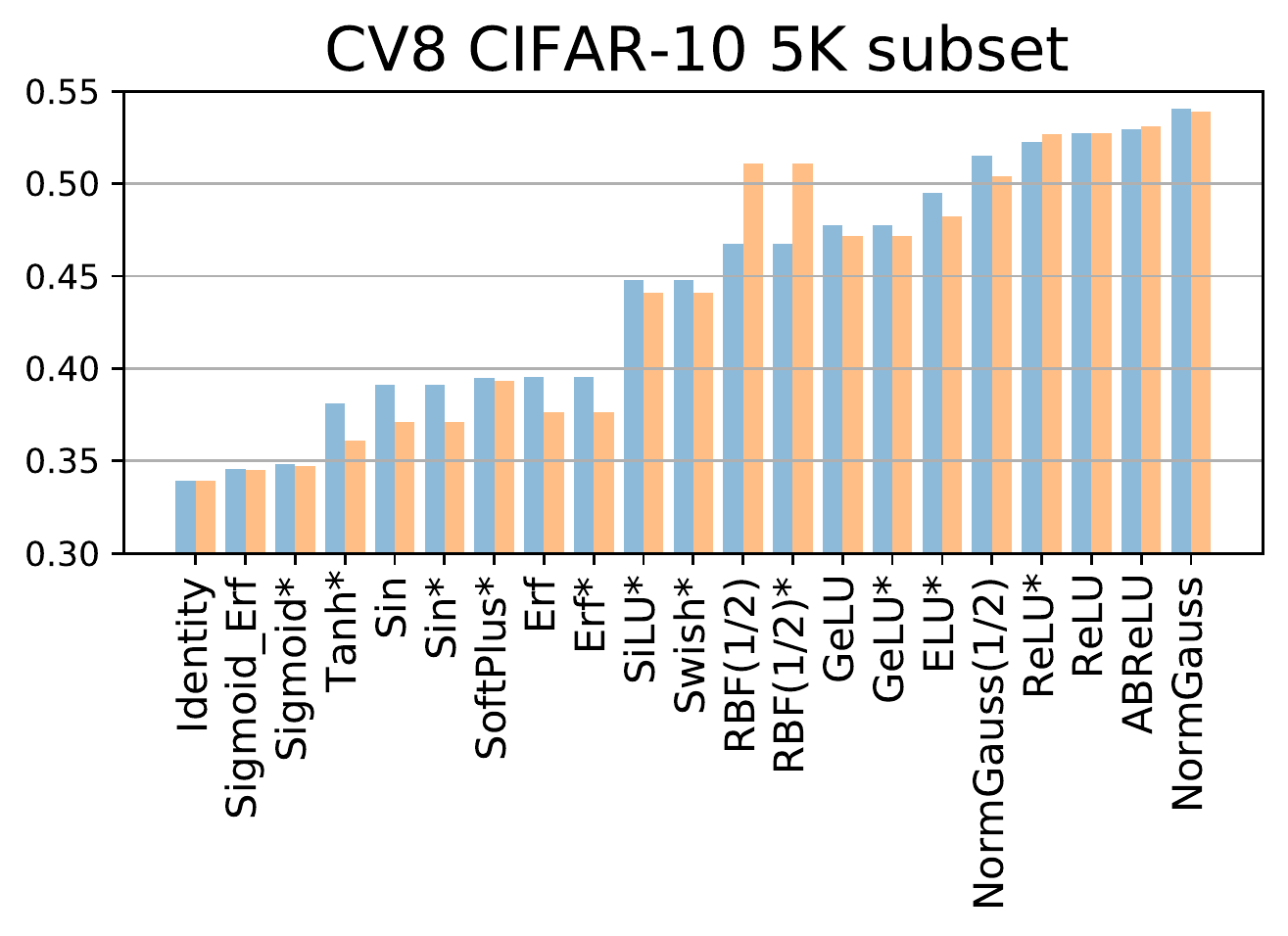}
	\end{subfigure}	
	\begin{subfigure}{0.33\textwidth}
		\includegraphics[width=\textwidth]{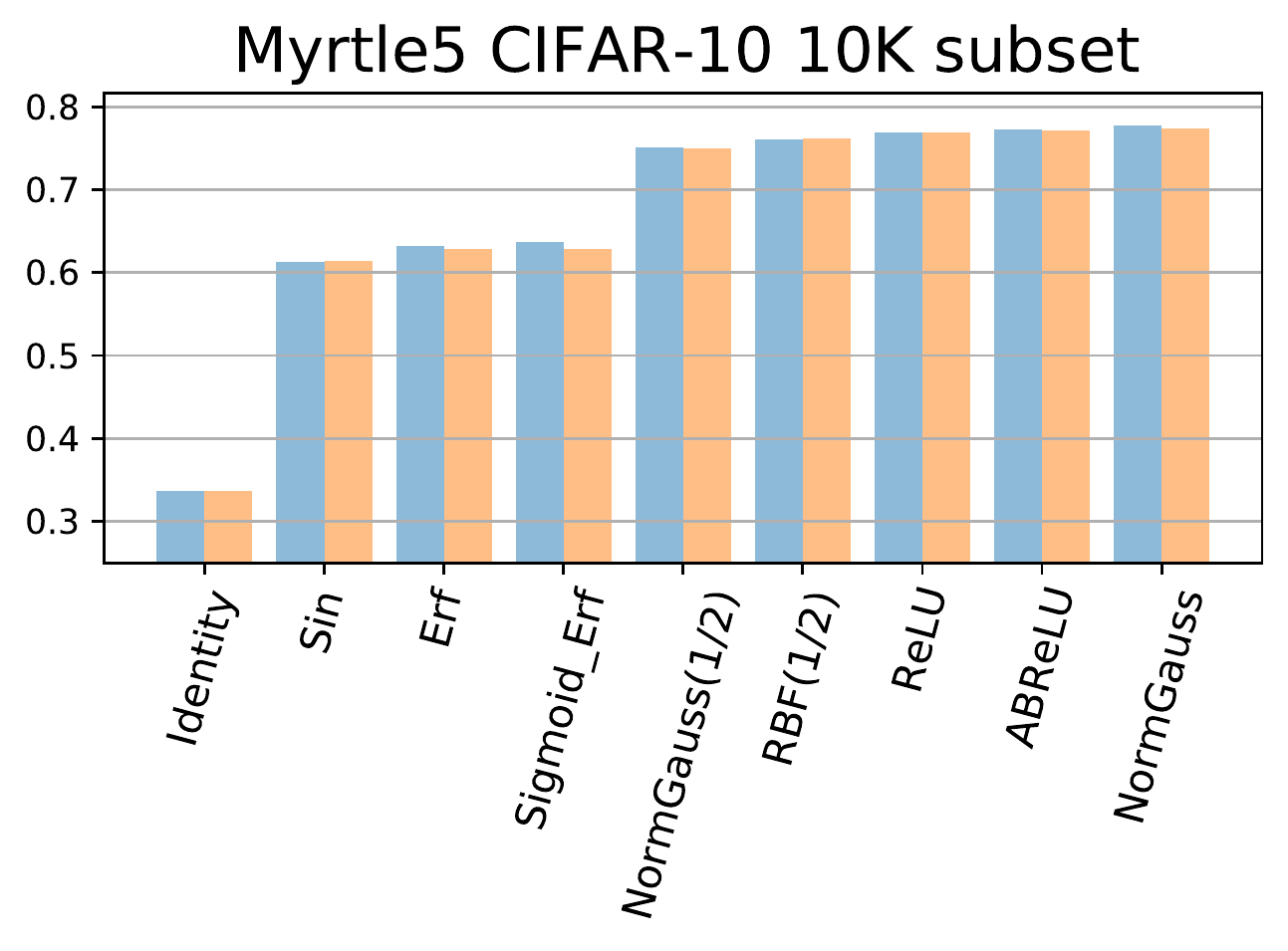}
	\end{subfigure}
	\caption{\textbf{Classification performance on a CIFAR-10 subset of various architectures.} We compare performance of various activation functions in neural kernels. $*$ denotes that Hermite-quadrature was used to numerically compute the dual activation functions. The slopes of ABReLU are chosen to match the Normalized Gaussian.} \label{fig-exact-activation-comparison}	
\end{figure}

We utilize this method as well as our expanded dual activation \cref{tab-dual-kernel} to compare performance of various activation functions on CIFAR-10 dataset. In \cref{fig-exact-activation-comparison}, we study three architectures; 1 hidden layer fully connected network (FC1, equivalent to pure dual activation kernel), 8 layer convolutional network with vectorization (CV8), and Myrtle5 network. We compared classification performance on subset of CIFAR-10. In each plot activation function is sorted by NTK's classification performance. One notable observation is that normalized Gaussian shows consistently best performance across architecture. Also note that smooth activations computed with Gauss-Hermite quadrature (denoted by $*$) shows almost identical performance when analytic form is available (e.g. GeLU, Erf, RBF(1/2)). Notable outlier is FC1 NTK with ReLU, however we expect that non-smooth activation may be approximated poorly.
It's also interesting to observer sigmoid-like activations (Sigmoid, Tanh, Erf) performs poorly across the board whereas ReLU-like activations (Normalized Gaussian, ABReLU, ReLU, GeLU, RBF) are among high performant group. 


\section{Table of dual activation functions} \label{sec-appendix-dual-activations}

We describe dual kernel functions of several activations and their derivatives in \cref{table-dual-kernel-details}.  One can generalize duel kernels of affine transformations of these activations. Specifically, if $\widetilde{\sigma}(t) = A \cdot \sigma(B t) + C$ for some $A,B,C\in \Rbb$ then
\begin{align} \label{eq-dual-kernel-affine}
    k_{\widetilde{\sigma}}(a,b,c) = A^2 \cdot k_{\sigma}(Ba, Bb, c) + C^2 + A\, C 
    \Ebb_{t\sim\mathcal{N}(0,1)}[\sigma(Bat)+\sigma(Bbt)]
\end{align}
which follows from that
\begin{align}
    k_{\widetilde{\sigma}}(a,b,c) 
    &= \Ebb_{(u,v)\sim \mathcal{N}(0,{\bm\Lambda}_{a,b,c})}\left[ A^2\sigma(Bu) \sigma(Bv)+C^2 + AC \left(\sigma(Bu) + \sigma(Bv)\right)\right] \nonumber\\
    &=A^2 \Ebb_{(u,v)\sim \mathcal{N}(0,{\bm\Lambda}_{a,b,c})}\left[\sigma(Bu)\sigma(Bv)\right] + C^2 + AC \Ebb_{(u,v)\sim \mathcal{N}(0,{\bm\Lambda}_{a,b,c})}\left[\sigma(Bu) + \sigma(Bv)\right] \nonumber \\
    &= A^2 \cdot k_{\sigma}(Ba, Bb, c) + C^2 + AC \left(  \Ebb_{u\sim \mathcal{N}(0,a^2)}\left[\sigma(Bu)\right] + \Ebb_{u\sim \mathcal{N}(0,b^2)}\left[\sigma(Bv)\right]\right)\nonumber \\
    &= A^2 \cdot k_{\sigma}(Ba, Bb, c) + C^2 + AC \left(  \Ebb_{t\sim \mathcal{N}(0,1)}\left[\sigma(Bat)\right] + \Ebb_{t\sim \mathcal{N}(0,1)}\left[\sigma(Bbt)\right]\right).
\end{align}
Below we provide detailed expressions of omitted dual kernel formulations in the table.

\begin{table}
\renewcommand{\arraystretch}{2}
\scalebox{0.8}{
\begin{tabular}{@{}lllll@{}}
\toprule
Activation & $\sigma(t)$ & Dual kernel $k_\sigma (a, b, c)$ &
$k_{\dot \sigma} (a, b, c)$
&
Implemented as
\\
\midrule
    \rowcolor{Gray}
    \Gape[0pt][2pt]{\makecell[l]{Rectified \\ monomials\\\citep{cho2009kernel}}} & $ t^n \cdot \mathbbm{1}_{\{t \geq 0\}}$ & $ \frac{(a b)^n}{2\pi} J_n\left( \cos^{-1}(c)\right)$ 
    &
    $ \frac{n^2 (a b)^{n-1}}{2\pi} J_{n-1}\left( \cos^{-1}(c)\right)$
    & \texttt{RectifiedMonomial}
    \\
    
    \makecell[l]{ReLU\\\citep{cho2009kernel}} & $ \max(t,0)$ & $ \frac{a b}{2\pi} \left(\sqrt{1-c^2} + (\pi-\cos^{-1}(c))c\right)$& 
    $\frac{1}{2\pi} (\pi-\cos^{-1}(c))$
    & \texttt{ReLU}
    \\
    
    \rowcolor{Gray}
    \Gape[0pt][2pt]{\makecell[l]{ABReLU\\\cite{tsuchida2018invariance,tsuchida2019richer,neuraltangents2020}}} & \Gape[0pt][2pt]{\makecell[l]{$A \max(t,0)$ \\ $+ B \max(-t,0)$}} & \Gape[0pt][2pt]{\makecell[l]{$\frac{ab(B-A)^2 \left( \sqrt{1-c^2} + \left(\pi-\cos^{-1}(c) \right)c\right)}{2\pi}$\\$+ABabc$}} & \cref{sec-abrelu} & \texttt{ABReLU}
    \\
    \multirow{3}{*}{\makecell[l]{Sinusoidal\\\cite{sitzmann2020implicit,tancik2020fourier}}}&$\sin(t)$ & $e^{ - \frac{a^2 + b^2}{2}} \sinh(abc)$ & $e^{ - \frac{a^2 + b^2}{2}} \cosh(abc)$ & \multirow{3}{*}{\texttt{Sin}}
    \\
    &$\cos(t)$ & $e^{ - \frac{a^2 + b^2}{2}} \cosh(abc)$ & $e^{ - \frac{a^2 + b^2}{2}} \sinh(abc)$ 
    \\
     &$A\sin(Bt + C)$ & \cref{eq-gen-sin-dual}& \cref{eq-gen-sin-dual-prime} 
    \\
    
    \rowcolor{Gray}
    \Gape[0pt][2pt]{\makecell[l]{Error function\\\cite{williams1996computing,lee2019wide}}} &$\mathrm{erf}(t)$ & $\frac{2}{\pi} \sin^{-1}\left( \frac{2abc}{\sqrt{(1 + 2a^2)(1 + 2 b^2)}} \right)$ & $\frac{4}{\pi} \frac{1}{\sqrt{(1+2a^2)(1+2b^2)-4(abc)^2}}$ & \texttt{Erf}
    \\
    
    \makecell[l]{Gaussian\\\cite{williams1996computing}} &$\exp(-A t^2)$ & $\frac{1}{\sqrt{(2Aa^2 + 1)(2A b^2 + 1)- (2A abc)^2}}$ & $\frac{4A^2 abc}{\left((2Aa^2 + 1)(2A b^2 + 1)- (2A abc)^2\right)^{3/2}}$ 
    & \texttt{Gaussian}
    \\
    
    \rowcolor{Gray}
    \Gape[0pt][2pt]{\makecell[l]{Exponential\\\cite{mairal2014convolutional,daniely2016toward}}} & $\exp(A t)$ & $\exp\left(\frac{A^2}{2}\left( a^2 + b^2 + 2abc\right) \right)$ & $A^2\exp\left(\frac{A^2}{2}\left( a^2 + b^2 + 2abc\right) \right)$ &\texttt{Exp}
    \\
    
    \makecell[l]{GeLU\\\cite{tsuchida2020avoiding}} &$\frac{t}{2} \left( 1 + \mathrm{erf}\left(\frac{t}{\sqrt{2}}\right)\right)$ & \cref{eq-dual-gelu} &  
    \cref{sec-gelu} &  \texttt{Gelu}
    \\
    
    \rowcolor{Gray}
    Gabor & $\exp(-t^2) \sin(t)$ & \cref{eq-dual-activation-gabor} & \cref{sec-gabor} & \texttt{Gabor}	\\

    Polynomial &$\sum_{j} c_j t^j$ & \cref{thm-dual-kernel-expansion} & \cref{thm-dual-kernel-expansion} & \texttt{Polynomial}
    \\
    
    \rowcolor{Gray}
    \Gape[0pt][2pt]{\makecell[l]{Normalized \\Gaussian~\cite{shankar2020neural}}} & Unknown & $ab\exp(c-1)$ & $\exp(c-1)$ & \texttt{ExpNormalized} \\
    
    \makecell[l]{RBF\\\cite{rahimi2007random}} & $\sqrt{2}\sin\left( \sqrt{2A}t + \frac{\pi}{4}\right)$ & $\exp\left(-A\left(a^2+b^2-2abc\right)\right)$ & $2A\exp\left(-A\left(a^2+b^2-2abc\right)\right)$ & \texttt{Rbf}\\
\bottomrule
\end{tabular}
}
\vspace{0.05in}
\caption{Dual kernels of activation and its derivative for various functions.} \label{table-dual-kernel-details}
\end{table}

\subsection{Rectified monomials} \label{sec-rectified-monomials} 
\citet{cho2009kernel} proposed closed-form expressions of dual kernel functions for rectified activations, i.e., $\sigma(t) = t^n \cdot \mathbbm{1}_{\{t \geq 0\}}$ for $n \ge 0$, as
\begin{align}
    k_\sigma(a,b,c) = \frac{(ab)^n}{2 \pi} \cdot J_n\left(\cos^{-1}(c)\right)
\end{align}
where for $\theta = \cos^{-1}(c) \in [0,\pi]$
\begin{align}
    J_n(\theta) := (-1)^n (\sin \theta)^{(2 n + 1)} \left( \frac{1}{\sin \theta} \frac{\partial}{\partial \theta}\right)^n \left(\frac{\pi - \theta}{\sin \theta} \right).
\end{align}
For $n=0$ and $1$ 
\begin{align}
    J_0(\theta) =\pi - \theta,\,~~\quad J_1(\theta) = \sin \theta + (\pi - \theta) \cos \theta.
\end{align}
Applying \cref{lmm-derivative-dual-kernel} provides that
\begin{align}
    k_{\sigma'}(a,b,c) = \frac{n^2 (ab)^{n-1}}{2 \pi} \cdot J_{n-1}\left( \cos^{-1}(c)\right).
\end{align}
These are implemented in our code as \texttt{RectifiedMonomial} (with a special case of \texttt{Sign} for convenience).

\subsection{ABReLU, Leaky ReLU, Abs} \label{sec-abrelu} 
ABReLU activation function is given by
\begin{align}
    \sigma(t) = A \min(t, 0) + B \max(t, 0), ~~~ \text{ for }~ A,B \in \Rbb
\end{align}
The dual kernel functions can be obtained by extension of \cite{cho2009kernel} which is worked out in \cite{tsuchida2018invariance, tsuchida2019richer}
\begin{align}
    k_{\sigma}(a,b,c) &= ab \left( \frac{(B-A)^2}{2\pi}J_1\left(\cos^{-1}(c) \right) + ABc\right) \\
    &= ab\left( \frac{(B-A)^2}{2 \pi} \left( \sqrt{1-c^2} + (\pi - \cos^{-1}(c))c)\right) + ABc\right)
\end{align}
and
\begin{align}
    k_{\sigma'}(a,b,c) &= \frac{(B-A)^2}{2 \pi} J_0(\cos^{-1}(c) + AB \\
    &= \frac{(B-A)^2}{2 \pi} \left( \pi - \cos^{-1}(c)\right) + AB.
\end{align}
A special case of ABReLU covers leaky ReLU~\cite{maas2013rectifier} ($B=1$), that is,
\begin{align}
    \sigma(t) = A \min(t, 0) + \max(t, 0)\,,
\end{align}
and the corresponding dual kernel functions are
\begin{align}
    k_{\sigma}(a,b,c) &= ab \left( \frac{(1-A)^2}{2\pi}J_1\left(\cos^{-1}(c) \right) + A c\right),
\end{align}
and
\begin{align}
    k_{\sigma'}(a,b,c) &= \frac{(1-A)^2}{2 \pi} J_0\left(\cos^{-1}(c)\right) + A.
\end{align}
Another special case is the absolute value function (Abs) ($A=-1,B=1$), that is,
\begin{align}
    \sigma(t) = |t|\,,
\end{align}
and the corresponding dual kernel functions are
\begin{align}
    k_{\sigma}(a,b,c) &= ab \left(\frac{2}{\pi}J_1\left(\cos^{-1}(c) \right)  - c \right)
\end{align}
and
\begin{align}
    k_{\sigma'}(a,b,c) &= 1 - \frac{2}{\pi} \cos^{-1}(c) \,.
\end{align}
These are respectively implemented as \texttt{ABRelu}, \texttt{LeakyRelu}, and \texttt{Abs} in~\cite{neuraltangents2020}.

\subsection{Sinusoidal and RBF} \label{sec-sinusoidal}
A generalized sinusoidal activation is given by
\begin{align}
    \sigma(t) =A \sin(B t +C). 
\end{align}
The corresponding dual kernels are
\begin{align}
    k_\sigma (a, b, c) &=  \frac{A^2}{2} \cdot e^{- \frac{B^2(a^2 + b^2)}{2}} \left(e^{a b c B^2} -  \cos(2C) e^{- a b c B ^2}\right) \label{eq-gen-sin-dual} \\
    k_{\sigma'} (a, b, c) &= \frac{A^2 B^2}{2} \cdot e^{- \frac{B ^2(a^2 + b^2)}{2}} \left(e^{a b c B^2} + \cos(2C) e^{- a b c B ^2}\right) \label{eq-gen-sin-dual-prime}.
\end{align}


Note that the generalized sinusoidal activation with $a = \sqrt{2}$, $b=\sqrt{2 A}$, and $c=\frac{\pi}{4}$ gives that
\begin{align}
    k_\sigma(a,b,c) &= \exp\left( -A \left( a^2 + b^2 - 2abc\right)\right), \\
    k_{\sigma'}(a,b,c) &= 2A\exp\left( -A \left( a^2 + b^2 - 2abc\right)\right)\,,
\end{align}
which corresponds to (translation invariant) the Gaussian RBF kernel:
\begin{equation}
    k_\text{RBF}(x,y) = \exp\left(- d A \norm{x - y}^2\right)\,.
\end{equation}
for some $x,y\in\Rbb^d$.

Moreover, one could consider mixture of activation functions as discussed in \citet{louart2018random, adlam2019random} of 50\% $\cos$ and 50\% $\sin$ which also leads to stationary kernel
\begin{align} 
k_{\cos + \sin}(a,b,c) &=  \frac{1}{2}\exp(- \frac{1}{2} (a^2 + b^2 -2 abc)). 
\end{align}
In order to obtain stationary kernel with respect to inputs, one only needs to insert these transformation at the first layer of the network as highlighted in implicit neural representation (e.g. NeRF)~\cite{sitzmann2020implicit,tancik2020fourier}.

These are implemented in our code as \texttt{Sin}, \texttt{Cos}, and \texttt{Rbf}.

\subsection{Error function} \label{sec-erf}
The error function is given by
\begin{align}
    \sigma(t) = \frac{2}{\sqrt{\pi}} \int_{0}^{t} e^{-x^2} dx.
\end{align}
Following \cite{williams1996computing} and applying \cref{lmm-derivative-dual-kernel}, we get
\begin{align}
    k_\sigma(a,b,c) &= \frac{2}{\pi} \sin^{-1} \left( \frac{2abc}{\sqrt{(1+a^2)(1+b^2)}}\right), \\
    k_{\sigma'}(a,b,c) &= \frac{4}{\pi} \frac{1}{\sqrt{(1+2a^2)(1+2b^2)-4(abc)^2}}.
\end{align}

An affine transformation of the error function could behave similar to sigmoid activation function with range $(0, 1)$, that is, 
\begin{align}
    \sigma_{\textrm{sigmoid-like}} (x) = \frac{1}{2} \left( \operatorname{erf}\left(\frac{x}{2.4020563531719796}\right) + 1\right).
\end{align}

The corresponding dual kernels can be obtained by applying affine transformation to that of error function as discussed in \cref{eq-dual-kernel-affine}. The error function is implemented in \cite{neuraltangents2020} as \texttt{Erf}, and we release \texttt{Sigmoid\_like} in our code.



\subsection{Gaussian function} \label{sec-gaussian}

Consider Gaussian function
\begin{align*}
    \sigma(t) = \exp(- At^2)\,.
\end{align*}

One can obtain $k_{\sigma}$~\cite{williams1996computing},
\begin{align}
    k_\sigma(a,b,c) &= \frac{1}{\sqrt{(2Aa^2 + 1)(2A b^2 + 1)- (2A abc)^2}}  
\end{align}
and using \cref{lmm-derivative-dual-kernel} obtain
\begin{align}
    k_{\sigma'}(a,b,c) = \frac{4A^2 abc}{\left((2Aa^2 + 1)(2A b^2 + 1)- (2A abc)^2\right)^{3/2}}.
\end{align}
Note that Gaussian function itself can be obtained as derivative of Affine Erf thus could use \cref{lmm-derivative-dual-kernel} with Affine Erf. This function is implemented as \texttt{Gaussian} in our code.

\subsection{GeLU} \label{sec-gelu}
The Gaussian Error Linear Unit (GeLU)~\cite{hendrycks2016gaussian} is defined as
\begin{align} \label{eq-gelu-activation}
    \sigma(t) = \frac{t}{2} \left( 1 + \mathrm{erf}\left(\frac{t}{\sqrt{2}}\right)\right) = \frac{t}{\sqrt{2\pi}} \int_{-\infty}^t e^{-\frac{s^2}{2}}ds\,,
\end{align}
where $\mathrm{erf}(\cdot)$ is the Gauss error function. For efficiency, sometimes approximate formulation
\begin{align} \label{eq-gelu-approx-activation}
    \tilde \sigma(t) = \frac{t}{2}\left(1 + \tanh \left(\sqrt{\frac{2}{\pi}}(t + 0.044715t^3)\right)\right)\,, 
\end{align}
is used. We note that GeLU activation function is becoming popular in recent language models such as BERT~\cite{devlin2018bert}, ALBERT~\cite{lan2019albert}, RoBERTa~\cite{liu2019roberta} and GPT~\cite{radford2019language, brown2020language}. 
The corresponding dual kernel is studied in \citet{tsuchida2020avoiding}:
\begin{align} \label{eq-dual-gelu}
    k_{\mathrm{GeLU}}(a,b,c) 
    = \frac{a b c}{4} + \frac{a^2 b^2}{2\pi} \Bigg( &\frac{c^2 + 1+a^2+b^2+a^2b^2(1-c^2)}{(1+a^2)(1+b^2)\sqrt{1+a^2 + b^2 + a^2 b^2(1-c^2)}} \nonumber \\ &+
    \frac{c}{a b} \tan^{-1} \left( \frac{a b c}{\sqrt{1+a^2 + b^2 + a^2 b^2 (1-c^2)}}\right)\Bigg).
\end{align}
Using \cref{lmm-derivative-dual-kernel}, we have
\begin{align}
    k_{\mathrm{GeLU}'}(a,b,c) 
    = \frac14 &+  
    \frac{\left(2-a^2 b^2\right) abc (1+a^2)(1+b^2)+\left(a^2 b^2-1\right)
   (abc)^3}{2 \pi  (1+a^2)(1+b^2) \left(1+a^2 + b^2 + a^2 b^2 (1-c^2)\right){}^{3/2}} \nonumber \\
   &+\frac{1}{2\pi} \tan^{-1}\left( \frac{a b c}{\sqrt{1+a^2 + b^2 + a^2 b^2 (1-c^2)}}\right) \nonumber \\
   &+ \frac{abc}{2\pi} \frac{1}{\sqrt{1+a^2 + b^2 + a^2 b^2 (1-c^2)}}.
    \label{eq-dual-gelu-prime}
\end{align}

This is implemented in our code as \texttt{Gelu}.

\subsection{Monomials}
Consider monomials
\begin{align}
    \sigma_n(t) = t^n, \quad n \in \mathbb{N} \,.
\end{align}

The dual activation function is given in terms of Hypergeometric function $_2F_1$.  For even power $n\in 2\mathbb{Z}$
\begin{align}
    k_{\sigma_n} (a, b, c)
&= \frac{(2ab)^n \left(1-c^2\right){}^{n/2} }{\pi } \Gamma
   \left(\frac{n+1}{2}\right)^2 \,
   _2F_1\left(-\frac{n}{2},\frac{n+1}{2};\frac{1}{2};\frac{c^2}{c^2-1}\right)
\end{align}

For odd power $n\in 2\mathbb{Z}+1$
\begin{align}
    k_{\sigma_n} (a, b, c)=& 
   \frac{2^n(ab)^{n+1} \left(1-c^2\right){}^{\frac{n-1}{2}} \Gamma
   \left(\frac{n}{2}+1\right) \Gamma \left(\frac{n}{2}\right) }{\pi  (n+1) c} \Bigg(2 c^2 \,
   _2F_1\left(\frac{1}{2}-\frac{n}{2},\frac{n}{2}+1;\frac{1}{2};\frac{c^2}{c^2-1}\right) \nonumber  \\
   &+a \left(\,
   _2F_1\left(\frac{1}{2}-\frac{n}{2},\frac{n}{2}+1;\frac{1}{2};\frac{c^2}{c^2-1}\right)-\,
   _2F_1\left(\frac{1}{2}-\frac{n}{2},\frac{n}{2}+1;-\frac{1}{2};\frac{c^2}{c^2-1}\right)\right)\Bigg)\,,
\end{align}


The first five $k_{\sigma_n}$s are 
\begin{align}
k_{\sigma_0}(a,b,c) &= 1, \\ 
k_{\sigma_1}(a,b,c) &= a b c, \\
k_{\sigma_2}(a,b,c) &= a^2 b^2 \,  (2c^2+1) \,,\\ 
k_{\sigma_3}(a,b,c) &= 3a^3 b^3 c \, (2c^2 +3) \,, \\
k_{\sigma_4}(a,b,c) &=3 a^4b^4(8c^4 + 24c^2 + 3)\,, \\
k_{\sigma_5}(a,b,c) &= 15 a^5 b^5 c (8 c^4 + 40 c^2 + 15).
\end{align}
Note that dual activation functions of monomials are also obtained from \cref{thm-dual-kernel-expansion} by choosing $c_n=1, c_{n-1}=\dots=c_0 = 0$.
Moreover, obtaining $k_{\sigma_n'}$ is simple either by $\sigma_n (t)' = n t^{n-1}$ or applying \cref{lmm-derivative-dual-kernel} to above expressions on $k_{\sigma_n}$.

These are implemented in our code as \texttt{Monomial}.


\def\T{\mathcal{T}}

\subsection{Gabor} \label{sec-gabor}
Let us consider a simple version of localized oscillatory activation function given by
\begin{equation}
    \sigma_{\textrm{Gabor}}(t) = \exp(-t^2) \sin(t)\,.
\end{equation}

The dual actiavtion of Gabor function can be expressed as
\begin{align} \label{eq-dual-activation-gabor}
    k_{\mathrm{Gabor}}(a,b,c) = 
   \frac{\exp \left(-\frac{-4 a^2 b^2 c^2+2 a b c+4 a b+a+b}{-8 a^2 b^2 c^2+8 a b+4 a+4 b+2}\right) \left(\exp \left(\frac{2 a b c}{-4 a^2 b^2 c^2+4 a b+2 a+2 b+1}\right)-1\right)}{\sqrt{-4 a^2 b^2 c^2+a (4 b+2)+2 b+1}}
\end{align}

and that of derivative of Gabor function can be obtained using \cref{lmm-derivative-dual-kernel} as
\begin{multline}
    k_{\mathrm{Gabor}^\prime}(a,b,c) = 
    \exp \left(-\frac{-4 a^2 b^2 c^2+2 a b c+4 a b+a+b}{-8 a^2 b^2 c^2+8 a b+4 a+4 b+2}\right) \times \\ \times\Biggl[\Bigl(4 a b c \left(-4 a^2 b^2 c^2+a b c+3 b+2\right)+2 a \left(8 a b^2 c+6 a b c+2 b+1\right)+2 b+1\Bigr)\times \\ \times\exp \left(\frac{2 a b c}{-4 a^2 b^2 c^2+4 a b+2 a+2 b+1}\right)+4 a b c \left(4 a^2 b^2 c^2+a b c-3 b-2\right)+\\+2 a \left(-8 a b^2 c-6 a b c+2 b+1\right)+2 b+1\Biggr] \Bigg/ \left(-4 a^2 b^2 c^2+a (4 b+2)+2 b+1\right)^{5/2}
    .
\end{multline}

This is implemented in our code as \texttt{Gabor}.


\subsection{ELU}
For Exponential Linear Unit (ELU)~\cite{clevert2015fast}
\begin{align*}
    \sigma(t) = \mathrm{step}(t) t + \mathrm{step}(-t)(e^t -1)\,.
\end{align*}
The $k_\sigma (a, b, c)$ is computed in \citet{tsuchida2020avoiding} and we refer to the original paper for the expression.

Note that $k_{\sigma'}$ for ELU has not been computed but \cref{lmm-derivative-dual-kernel} allows to simply obtain it using expression in ~\citet{tsuchida2020avoiding}.

\section{Additional Experiment: Kernel Informed Activation} \label{sec-appendix-abrelu}
We explore an activation informed by the normalized Gaussian kernel that achieves the best performance among neural kernels~\cite{shankar2020neural}. Although the exact activation is unknown, one can conduct a reverse engineering to find a proper activation whose dual kernel is known and close to the normalized Gaussian. In particular, we focus on the ABReLU activation and recall that its dual kernel is 
\[
k_{\mathrm{ABReLU}}(a,b,c) = ab \left( \frac{(B-A)^2 \left( \sqrt{1-c^2} + \left(\pi-\cos^{-1}(c) \right)c\right)}{2\pi}+ABc\right)
\]
for some $A,B \in \Rbb$. Observe that $k_{\mathrm{ABReLU}}$ is also homogeneous as like the normalized Gaussian, i.e., $k_\sigma(a,b,c)=ab \cdot \kappa_{\sigma}(c)$ for $c\in[-1,1]$. 
We find two slope variables $A,B$ by fitting $\kappa_{\sigma}$ at extreme points, i.e., $\kappa_{\mathrm{ABReLU}}(c)=\exp(c-1)$ for $c=\pm1$. This turns into a quadratic equation and gives us \[\mathrm{ABReLU}(t)=-0.096 \min(t,0)+1.411\max(t,0)\] which is illustrated in \cref{fig-cifar10} ({\bf left}). 
We train a $5$-layer ConvNet (known as Myrtle-5~\cite{shankar2020neural}) of $128$ width for CIFAR-10 classification. Similar to the CNTK experiment in \cref{sec-exp}, we convert image classes into $10$-dimensional one-hot vectors and pre-process CIFAR-10 images with regularized ZCA~\cite{shankar2020neural,lee2020finite}. We use the SGD optimizer with initial learning rate $0.1$, Nesterov momentum with factor $0.9$ and $\ell_2$ regularizer $0.0005$. The batch size is set to $64$. The network is trained by minimizing the mean-squared-error (MSE) loss and we report the best test accuracy for $200$ epochs.
Interestingly, the ABReLU can achieve the highest test accuracy compared to ReLU, GeLU, Erf and parameterized ReLU (PReLU) activations. This supports a connection between infinite width neural kernels and finite width networks in aspect of activation. 

\begin{figure}[h]
    \centering
    \begin{subfigure}{0.45\textwidth}
        \includegraphics[width=0.9\textwidth]{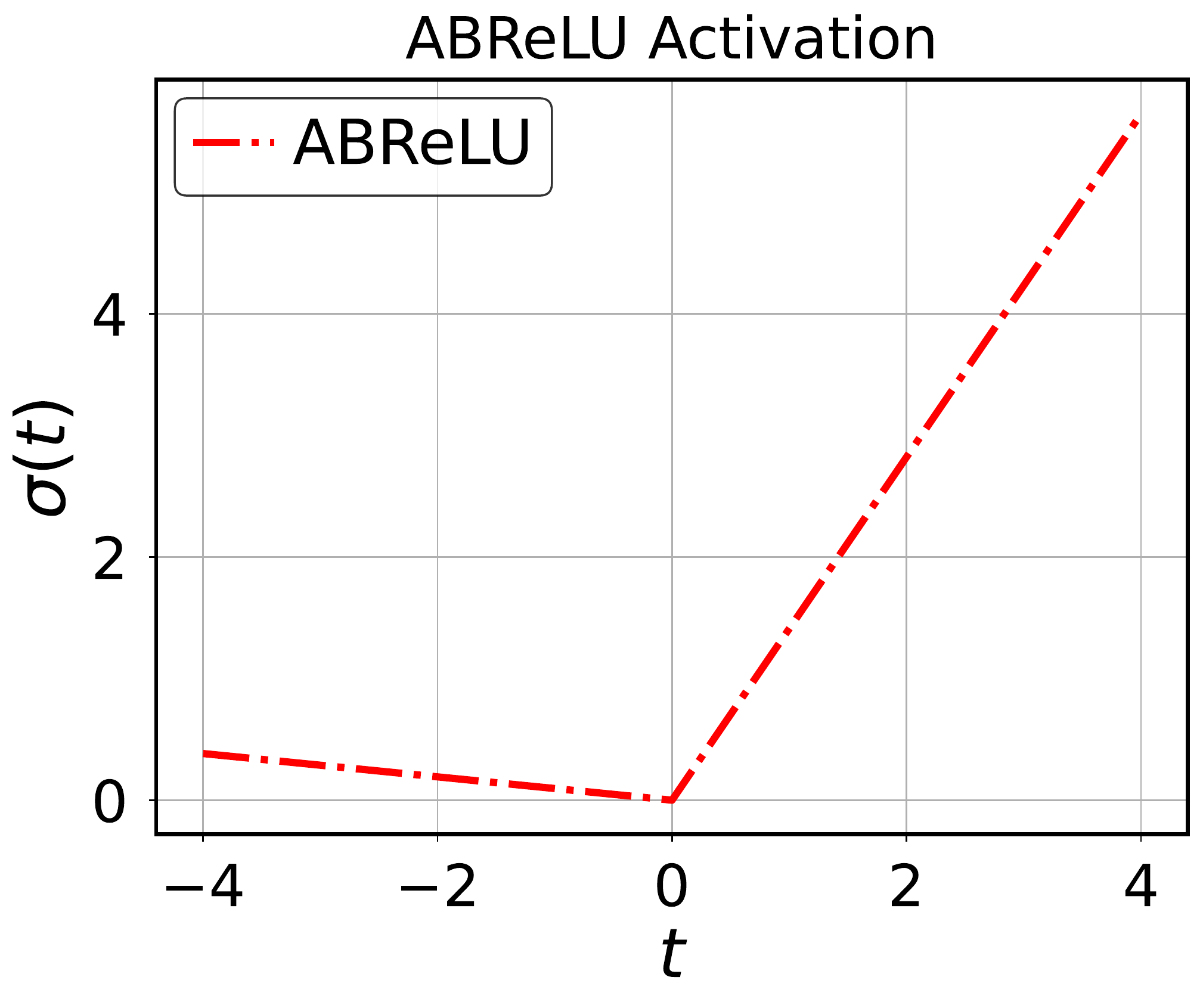}
    \end{subfigure}
    \hspace{0.01in}
    \begin{subfigure}{0.45\textwidth}
        \includegraphics[width=0.9\textwidth]{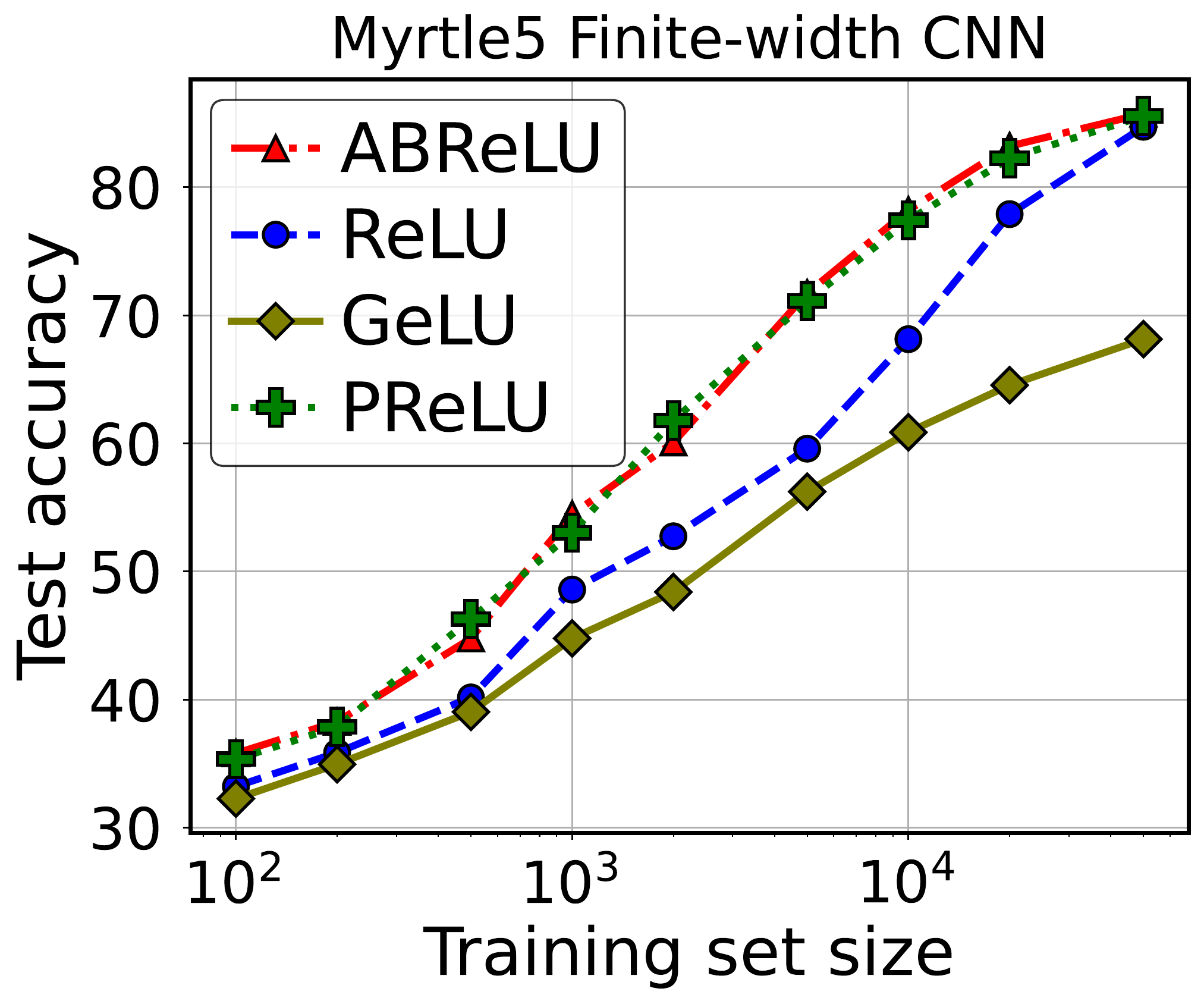}
    \end{subfigure}
    \caption{Kernel informed ABReLU ({\bf left}) and test accuracy of finite-width Myrtle5 networks with various activations ({\bf right}).}
    \label{fig:my_label}
\end{figure}


\end{document}